\documentclass[lettersize,journal]{IEEEtran}
\usepackage{amsmath,amsfonts}
\usepackage{algorithmic}
\usepackage{array}
\usepackage[caption=false,font=normalsize,labelfont=sf,textfont=sf]{subfig}
\usepackage{textcomp}
\usepackage{stfloats}
\usepackage{url}
\usepackage{verbatim}
\usepackage{graphicx}
\usepackage{amsmath,amsfonts,amssymb,amsthm,mathrsfs,mathtools}
\usepackage{algorithmic}
\usepackage[linesnumbered,ruled,vlined]{algorithm2e}
\usepackage{xcolor}

\usepackage[normalem]{ulem}
\usepackage{multirow}
\usepackage{booktabs}
\usepackage{makecell}
\usepackage[table]{xcolor} 
\usepackage{cite}
\usepackage{mdframed}
\usepackage{enumerate}
\usepackage{enumitem}
   
\theoremstyle{definition}
\newtheorem{assumption}{Assumption}
\newtheorem{theorem}{Theorem}
\newtheorem{definition}{Definition}
\newtheorem{corollary}{Corollary}
\newtheorem{lemma}{Lemma}

\newtheorem{proposition}{Proposition}

\def\BibTeX{{\rm B\kern-.05em{\sc i\kern-.025em b}\kern-.08em
    T\kern-.1667em\lower.7ex\hbox{E}\kern-.125emX}}
\usepackage{balance}
\begin{document}
\bstctlcite{IEEEexample:BSTcontrol}
\title{Efficient Exact Unlearning for Federated Zeroth-Order Optimization}

\author{Weijun~Yin, 
Yanqing~Yao,
Falko~Dressler,~\IEEEmembership{Fellow,~IEEE},
Di~Wang,~\IEEEmembership{Member,~IEEE},
Youming~Tao~\IEEEmembership{Member,~IEEE}

\thanks{Weijun Yin and Yanqing Yao are with the School of Cyber Science and Technology, Beihang University, Beijing 100191, China, and also with the Beijing Advanced Innovation Center for Future Blockchain and Privacy Computing, Beihang University, Beijing 100083, China (e-mail: yinweijun@buaa.edu.cn; yaoyq@buaa.edu.cn).}%
\thanks{Falko Dressler and Youming Tao are with the School of Electrical Engineering and Computer Science, TU Berlin, 10623 Berlin, Germany (e-mail: dressler@ccs-labs.org; tao@ccs-labs.org).}%
\thanks{Di Wang is with the Division of Computer, Electrical and Mathematical Sciences and Engineering, King Abdullah University of Science and Technology, Saudi Arabia (email: di.wang@kaust.edu.sa).}%
}



\maketitle

\begin{abstract}
Federated zeroth-order optimization~(FedZOO) reduces the memory and communication costs of collaborative training for edge intelligence by avoiding backpropagation while communicating only random seeds and scalar coefficients. However, existing methods lack efficient mechanisms for exact unlearning to fully remove the influence of training data when users request deletion. A key insight is that total variation stability of the learning algorithm creates overlap between model distributions before and after data deletion, allowing exact unlearning to reuse computation through coupling. Based on this observation, we jointly design two stable zeroth-order learning algorithms and their corresponding exact-unlearning procedures: \textsc{FedZOO-S} obtains stability through client and mini-batch subsampling, enabling \textsc{FedZOUL-S} to perform efficient unlearning by preserving earlier updates and restarting from the first affected step. \textsc{FedZOO-NS} further combines subsampling with Gaussian noise to increase overlap between update distributions; \textsc{FedZOUL-NS} exploits this overlap through rejection--reflection coupling to reuse affected updates when the coupling is accepted. We prove exact unlearning for both algorithms: the model and saved state have the same joint distribution as retraining with the corresponding learner on retained data. We further establish stability and convergence guarantees for the learning algorithms and derive upper bounds on expected unlearning cost. Experiments demonstrate effective unlearning with competitive model utility and up to $56.1\%$ less cumulative unlearning time than re-optimization. Compared with first-order baselines, our methods reduce communication overhead by up to approximately four orders of magnitude.

\end{abstract}

\begin{IEEEkeywords}
Federated learning, Exact Unlearning, Zeroth-Order Optimization, Total Variation Stability
\end{IEEEkeywords}

\section{Introduction}

Driven by advances in mobile and IoT hardware, edge intelligence has developed rapidly in recent years. Federated learning~(FL) enables collaborative optimization across edge devices by allowing them to exchange model parameters or gradients without moving data out of the local domain~\cite{mcmahan2017communication}. However, practical deployment remains constrained by limited memory and communication resources, with growing model sizes further amplifying these challenges. 
Techniques such as quantization~\cite{fu2020fractrain,reisizadeh2020fedpaq,ding2025lbifl}, sparsification~\cite{lin2018deep,tao2024private,guastella2025sparsyfed}, pruning and submodel training~\cite{huang2020deepadapter,horvath2021fjord,chen2025adaptive,tao2026autosub}, and knowledge distillation~\cite{zhang2019your,lin2020ensemble,fan2025fedcot} can reduce computation or communication costs. 
More recently, parameter-efficient and subspace-based federated fine-tuning methods have been developed to adapt large language models~(LLMs) with fewer trainable parameters~\cite{zhang2023fedpetuning,chen2026fedsubmuon}. 
Nevertheless, these training approaches still rely on backpropagation, whose intermediate activations can impose substantial memory overhead even when only a small subset of parameters is trainable.
Exchanging tensor-valued updates can also remain costly over bandwidth-constrained edge networks.

Federated zeroth-order optimization~(FedZOO) provides a promising alternative.
Unlike first-order methods, zeroth-order optimization~(ZOO) estimates descent
directions using only function evaluations and therefore avoids
backpropagation~\cite{wang2025scalable}. This substantially reduces the memory
required for optimization. Moreover, recent FedZOO methods exploit shared
randomness to represent high-dimensional perturbations by random seeds and
scalar coefficients~\cite{qin2024federated}, enabling communication costs that
do not scale with the model dimension. In particular, a fixed finite seed pool
allows accumulated zeroth-order updates to be represented in a compact
seed-coordinate space. These properties make FedZOO particularly attractive
for collaborative optimization of large models on resource-constrained edge
devices.

Existing FedZOO methods, however, do not address machine unlearning~(MU).
Federated systems may receive requests to remove training samples because of
data corrections, user withdrawal, or privacy requirements such as the
right to be forgotten under regulations under GDPR~\cite{voigt2024EU}. Simply deleting the requested samples from
local storage is insufficient because their effects may already have been
incorporated into the global model. Federated unlearning therefore aims to
remove these historical contributions from the learned model without repeating
the entire training process. This requirement is especially relevant to edge
intelligence, where participating devices may continuously contribute
user-generated and potentially sensitive data.

Achieving efficient exact unlearning for FedZOO is challenging.
Existing federated unlearning methods are largely designed around first- or
second-order optimization and can reintroduce backpropagation,
high-dimensional model states, or tensor-level communication, thereby losing
the resource advantages that motivate FedZOO.
Moreover, exact unlearning requires distributional equivalence to rerunning
the learning algorithm on the retained data~\cite{tao2024communication},
for which retraining from scratch is always valid but computationally expensive.

A key insight from stability-based exact unlearning is that this difficulty can be addressed from the learning side~\cite{ullah2021machine,tao2024communication,tao2026one-shot}.
If a randomized learning algorithm is stable in total variation under a single-sample change, then the output distributions on neighboring datasets overlap, creating opportunities to reuse computation through an appropriate coupling.
This leads to a fundamental question in our setting:

\begin{mdframed}[
    leftmargin=0pt,
    rightmargin=0pt,
    innerleftmargin=4pt,
    innerrightmargin=4pt,
    innertopmargin=4pt,
    innerbottommargin=4pt,
]
\emph{How can we design federated zeroth-order learning algorithms that are
TV-stable while retaining the communication and memory advantages of FedZOO?}
\end{mdframed}

To answer this question, we develop two TV-stable FedZOO learning algorithms
and their corresponding exact-unlearning procedures, built on a common
seed-encoded representation.
The first, \textsc{FedZOO-S}, obtains stability from client and mini-batch
subsampling: a single sample changes the realized trajectory only when its
client participates and the sample is selected for a local update.
Its unlearning counterpart, \textsc{FedZOUL-S}, therefore preserves the
unaffected prefix and replays only from the first affected step.
The second, \textsc{FedZOO-NS}, augments subsampling with Gaussian
randomization.
The added noise creates overlap between the conditional update distributions
on neighboring datasets, allowing some affected updates to be coupled rather
than immediately replayed.
Its unlearning counterpart, \textsc{FedZOUL-NS}, exploits this overlap through
rejection--reflection coupling so that some affected realized updates can
still be reused exactly.

Our contributions are summarized as follows.

\begin{enumerate}[leftmargin=*, label=(\arabic*)]
    \item We formulate efficient exact federated unlearning for zeroth-order optimization through the lens of algorithmic total variation stability. Rather than treating unlearning as a post hoc correction to an arbitrary FedZOO procedure, we design the learning algorithm itself to be TV-stable, so that neighboring training datasets induce overlapping output distributions that can be exploited through coupling for computation reuse. At the same time, we retain the fixed finite seed representation of FedZOO to avoid high-dimensional model-update communication.
    \item We propose \textsc{FedZOO-S}, a TV-stable federated zeroth-order learning algorithm whose stability arises from client and mini-batch subsampling, together with its exact-unlearning algorithm \textsc{FedZOUL-S}. A deletion leaves the realized trajectory unchanged when the deleted samples have not contributed to any realized local update. Otherwise, \textsc{FedZOUL-S} preserves the unaffected prefix and replays only the suffix starting from the first affected local step.
    \item We further propose \textsc{FedZOO-NS}, which augments subsampling with Gaussian randomization, together with its corresponding exact-unlearning algorithm \textsc{FedZOUL-NS}. The added noise creates overlap between the conditional update distributions on neighboring datasets, enabling rejection--reflection coupling. Consequently, some affected realized updates can still be reused exactly, and replay is required only after the first rejected coupling.
    \item We establish exact-unlearning guarantees for both learning--unlearning pairs and analyze their TV stability, optimization utility, and unlearning efficiency. We further evaluate their optimization and unlearning performance across federated learning settings with conventional models and large language models.
\end{enumerate}

\section{Related Work}
\subsection{Federated Zeroth-Order Optimization.}



In recent years, zeroth-order optimization (ZO), has shown strong potential in FL, especially in resource-constrained edge collaboration settings, because it avoids backpropagation. Fang et al.~\cite{fang2022communication} propose FedZO as one of the first derivative free federated ZO methods. For LLMs, Ling et al.~\cite{ling2024convergence} propose FedMeZO, which significantly reduces memory overhead and provides convergence analysis. Later, FedKSeed~\cite{qin2024federated} and DeComFL~\cite{li2025achieving} use scalar gradients and random seeds to synchronize perturbation randomness, achieving dimension-independent communication between edge nodes and the cloud. HiSo~\cite{li2026converge} further accelerate the convergence efficiency of federated ZO. At the same time, since transmitting estimated gradients or scalar gradients may expose privacy risks, methods such as DPZero~\cite{zhang2024dpzero}, DP-ZOSO~\cite{liu2026differentially} and DP-AggZO~\cite{bao2025unlocking} combine differential privacy with zeroth-order gradient estimation. These methods inject carefully calibrated noise during training to provide rigorous privacy guarantees. It is worth noting that this noise-based design not only strengthens privacy protection, but also motivates the noise perturbation strategy in our TV stable algorithm.

\subsection{Federated Unlearning.}




Federated unlearning (FU) aims to remove the influence of specified training data from a collaboratively learned model. Existing approaches can be broadly categorized into approximate and exact unlearning according to their guarantees relative to retraining on the retained data. Approximate methods seek to recover retraining-like behavior through model correction and gradient calibration~\cite{liu2022right,liu2021federaser,wang2024server,su2024f2ul,pan2025federated}, perturbation~\cite{che2023fast,zhang2023fedrecovery,fraboni2024sifu,wu2026certified,wu2026trajectoryaware}, or knowledge distillation~\cite{zhou2025federated,mora2026federated}. Related work on LLM unlearning also seeks to remove targeted knowledge while preserving useful model capabilities~\cite{zhang2026oblivionis,li2026not}. Although some approaches provide certified guarantees that bound distributional discrepancies through calibrated randomization~\cite{fraboni2024sifu,wu2026certified,wu2026trajectoryaware}, these guarantees do not establish exact equality between the unlearning and retraining distributions.
Exact FU requires the unlearning output to follow the same distribution as training on the retained data. To avoid full retraining, stability-based approaches exploit overlap between learning distributions to reuse prior computation~\cite{tao2024communication,tao2026one-shot}. However, these methods are developed around gradient-based training and model- or curvature-level states, which do not directly retain the memory and communication advantages of zeroth-order optimization. Extending them to FedZOO requires more than replacing the gradient oracle: seed-and-scalar updates change the state needed for replay and the transition distributions used for coupling. We therefore jointly design TV-stable zeroth-order learners and exact sample-level unlearning procedures around compact seed-encoded update histories. Our procedures reuse computation through subsampling and Gaussian coupling while preserving both the resource advantages of FedZOO and exact joint distributional equivalence of the model and saved state to retraining.

\section{Preliminaries}

\subsection{Federated Zeroth-Order Optimization}

Consider a federated optimization system with $M$ clients, where client
$c\in[M]$ holds a finite local dataset
$\mathcal D_c=\{x_{c,1},\ldots,x_{c,n_c}\}$ of size
$n_c=|\mathcal D_c|$.
The clients collaboratively optimize a global model
$\boldsymbol w\in\mathbb R^d$ under the coordination of a server.
We consider the uniformly weighted finite-sum objective
\begin{equation}
\label{eq:fed_objective}
\min_{\boldsymbol w\in\mathbb R^d}
f(\boldsymbol w)
\coloneqq
\frac{1}{M}
\sum_{c=1}^{M}
f_c(\boldsymbol w),
\qquad
f_c(\boldsymbol w)
\coloneqq
\frac{1}{n_c}
\sum_{x\in\mathcal D_c}
\ell(\boldsymbol w;x).
\end{equation}
Training proceeds for $R$ communication rounds.
In each round, $m$ clients participate, and each participating client
performs $E$ local optimization steps before the server uniformly aggregates
their updates.

Zeroth-order optimization estimates descent directions using function
evaluations without backpropagation.
Given a model $\boldsymbol w$, a sample $x$, and a perturbation
$\boldsymbol z\sim\mathcal N(\boldsymbol 0,I_d)$, the two-point scalar gradient is
\begin{equation}
\label{eq:zo_scalar}
g(\boldsymbol w;x,\boldsymbol z)
=
\frac{
\ell(\boldsymbol w+\mu\boldsymbol z;x)
-
\ell(\boldsymbol w-\mu\boldsymbol z;x)
}{
2\mu
},
\end{equation}
where $\mu>0$ is the perturbation scale.
The corresponding zeroth-order gradient estimator is
$g(\boldsymbol w;x,\boldsymbol z)\boldsymbol z$.

For a mini-batch $\mathcal B$ of size $b$, suppose that $q$ perturbations
$\{\boldsymbol z_{i,j}\}_{j=1}^{q}$ are used for each
$x_i\in\mathcal B$.
One local step is the mini-batch averaged zeroth-order update
\begin{equation}
\label{eq:zo_minibatch_update}
\boldsymbol w^{+}
=
\boldsymbol w
-
\frac{\eta}{bq}
\sum_{x_i\in\mathcal B}
\sum_{j=1}^{q}
g(\boldsymbol w;x_i,\boldsymbol z_{i,j})
\boldsymbol z_{i,j},
\end{equation}
where all $bq$ scalar gradients are evaluated at the same pre-update model
$\boldsymbol w$.

Following the seed-based representation of FedKSeed~\cite{qin2024federated},
a perturbation can be reproduced from its random seed when the same
pseudorandom number generator is used across devices.
A fixed shared pool of $K$ seeds can therefore represent zeroth-order
updates through their scalar coefficients in seed-coordinate space, avoiding
the transmission and storage of $d$-dimensional update vectors.

\subsection{Federated Unlearning}
\label{sec:fu_prelim}

Let
$\mathcal D
=
\{\mathcal D_1,\ldots,\mathcal D_M\}$
denote the federated dataset.
A randomized federated learning algorithm $\mathcal O$ produces a model
output together with auxiliary state maintained for subsequent unlearning:
\begin{equation}
\label{eq:learning_output_state}
(\boldsymbol W,\boldsymbol S)
\sim
\mathcal O(\mathcal D).
\end{equation}
Here, $\boldsymbol W$ denotes the model output and $\boldsymbol S$ collects
the internal state saved by the learning procedure.
The randomness of $\mathcal O$ may include client participation, data
sampling, zeroth-order perturbations, and other algorithmic randomness.
We use $\mathcal O_{\mathrm{out}}(\mathcal D)$ to denote the marginal
distribution of the model output $\boldsymbol W$.

An unlearning request specifies samples to be removed from one or more
clients.
Let $\mathcal C_{ul}$ be the set of affected clients and
$\mathcal D_{c,ul}\subseteq\mathcal D_c$ the samples requested for deletion
from client $c\in\mathcal C_{ul}$.
The retained federated dataset is
\[
\mathcal D^{-}
=
\{\mathcal D_1^{-},\ldots,\mathcal D_M^{-}\},
\qquad
\mathcal D_c^{-}
=
\begin{cases}
\mathcal D_c\setminus\mathcal D_{c,ul},
& c\in\mathcal C_{ul},\\
\mathcal D_c,
& c\notin\mathcal C_{ul}.
\end{cases}
\]
We consider valid deletion requests for which every retained local dataset
remains large enough for the prescribed mini-batch sampling, i.e.,
$|\mathcal D_c^{-}|\ge b$ for all affected clients.

The request may arrive during training.
We denote by $(R_{ul},t_{ul})$ the local training position immediately
preceding the request; that is, the request arrives immediately after local
step $t_{ul}$ of round $R_{ul}$.
Given the current model and saved state, an unlearning algorithm
$\mathcal U$ updates them according to the deletion request:
\[
(\boldsymbol W_{ul},\boldsymbol S_{ul})
=
\mathcal U
\left(
\boldsymbol W,
\boldsymbol S,
\{\mathcal D_{c,ul}\}_{c\in\mathcal C_{ul}}
\right).
\]

\begin{definition}[Exact Federated Unlearning]
\label{def:exact_fu}
A federated unlearning algorithm $\mathcal U$ is exact with respect to a
randomized learning algorithm $\mathcal O$ if, for every federated dataset
$\mathcal D$ and every valid deletion request,
\begin{equation}
\label{eq:exact_fu}
\mathcal U
\left(
\mathcal O(\mathcal D),
\{\mathcal D_{c,ul}\}_{c\in\mathcal C_{ul}}
\right)
\overset{d}{=}
\mathcal O(\mathcal D^{-}),
\end{equation}
where $\overset{d}{=}$ denotes equality in distribution over both the model
output and the maintained internal state.
\end{definition}

Thus, exact unlearning requires the post-unlearning procedure to have the
same distribution as rerunning the original learning algorithm on the
retained federated dataset, rather than merely producing a nearby model.

\subsection{Total Variation Stability}
\label{def-tv}

For two probability measures $P$ and $Q$ on the same measurable space, their
total variation (TV) distance is
\begin{equation}
\label{eq:tv_distance}
\operatorname{TV}(P,Q)
=
\sup_{A}
|P(A)-Q(A)|,
\end{equation}
where the supremum is over measurable events $A$.
When $P$ and $Q$ admit densities $p$ and $q$ with respect to a common
reference measure,
\[
\operatorname{TV}(P,Q)
=
\frac{1}{2}\|p-q\|_1.
\]

For two federated datasets
$\mathcal D=\{\mathcal D_c\}_{c=1}^{M}$ and
$\mathcal D'=\{\mathcal D'_c\}_{c=1}^{M}$, define
\[
\Delta(\mathcal D,\mathcal D')
=
\sum_{c=1}^{M}
|\mathcal D_c\triangle\mathcal D'_c|.
\]
We call them sample-level adjacent, denoted by
$\mathcal D\sim\mathcal D'$, if
$\Delta(\mathcal D,\mathcal D')=1$.

Following the stability-based exact-unlearning framework, TV stability is
defined on the marginal distribution of the learning algorithm's model
output, rather than on its auxiliary internal state.

\begin{definition}[$\rho$-Sample-Level TV Stability]
\label{def:tv_stability}
A randomized federated learning algorithm $\mathcal O$ is
$\rho$-sample-level TV stable if
\begin{equation}
\label{eq:tv_stability}
\sup_{\mathcal D\sim\mathcal D'}
\operatorname{TV}
\left(
\mathcal O_{\mathrm{out}}(\mathcal D),
\mathcal O_{\mathrm{out}}(\mathcal D')
\right)
\le \rho,
\end{equation}
where $\rho\in[0,1]$ is the stability parameter.
\end{definition}

TV stability therefore controls how much the distribution of the learned
model can change after modifying a single training sample.
For neighboring datasets, a smaller $\rho$ implies greater overlap between
the two output distributions, which can be exploited by coupling-based
unlearning to reuse computations from the original training trajectory.

\section{Our Method} \label{sec:method} We develop two pairs of federated learning and exact-unlearning algorithms. Both pairs build on a common seed-encoded representation of zeroth-order updates, which allows the realized training trajectory to be stored and replayed without retaining high-dimensional model updates. The first pair, \textsc{FedZOO-S} and \textsc{FedZOUL-S}, exploits client and sample subsampling: if a deleted sample has not entered the realized trajectory, the existing state remains valid; otherwise, exact unlearning replays the affected suffix. The second pair, \textsc{FedZOO-NS} and \textsc{FedZOUL-NS}, further randomizes local zeroth-order updates with Gaussian noise. The resulting overlap between update distributions on neighboring datasets allows some affected updates to be reused through coupling, reducing the need for replay.

\subsection{Seed-Encoded Zeroth-Order Updates} \label{sec:seed_encoded_zoo} All four algorithms use the same compact representation of zeroth-order updates. At initialization, the server generates a finite pool of $K$ distinct perturbation seeds \[ \mathbb{S}=\{s_1,\ldots,s_K\}, \] which is shared by the server and all clients and remains fixed throughout training. With a fixed pseudorandom number generator, each seed $s_k$ deterministically generates a Gaussian perturbation $\boldsymbol z_k=\mathsf{Perturb}(s_k)\in\mathbb{R}^d$, where $\boldsymbol z_k\sim\mathcal{N}(\boldsymbol 0,I_d)$. For model $\boldsymbol w$, sample $x$, and seed $s_k$, define the two-point zeroth-order scalar gradient as \[ g(\boldsymbol w;x,s_k) = \frac{ \ell(\boldsymbol w+\mu\boldsymbol z_k;x) - \ell(\boldsymbol w-\mu\boldsymbol z_k;x) }{2\mu}. \] The corresponding zeroth-order gradient estimator is $g(\boldsymbol w;x,s_k)\boldsymbol z_k$. Hence, once $\mathbb S$ is shared, an update along $\boldsymbol z_k$ is completely specified by its seed index and scalar coefficient, rather than by a $d$-dimensional vector. Because the seed pool is fixed and reused throughout training, coefficients associated with the same perturbation can be accumulated in seed-coordinate space. Let $\boldsymbol a^r=(a_1^r,\ldots,a_K^r)\in\mathbb{R}^K$ denote the global scalar-gradient accumulator after round $r$, where $a_k^r$ is the cumulative coefficient associated with $\boldsymbol z_k$. The corresponding global model is reconstructed as \begin{equation} \label{eq:reconstruction} \mathsf{Rec}(\boldsymbol a^r) = \boldsymbol w^0-\eta\sum_{k=1}^{K}a_k^r\boldsymbol z_k, \end{equation} and we write $\boldsymbol w^r=\mathsf{Rec}(\boldsymbol a^r)$. Thus, the persistent training state is represented by the $K$-dimensional accumulator $\boldsymbol a^r$, while the full model is reconstructed only when needed. At the beginning of round $r$, client $c$ reconstructs $\boldsymbol w^{r-1}$ from $\boldsymbol a^{r-1}$ and initializes $\boldsymbol a_c^{r,0}=\boldsymbol 0_K$. The local accumulator records only the updates generated by client $c$ during the current round. Each local step consists of one mini-batch averaged zeroth-order update. Specifically, let $\Delta\boldsymbol a_c^{r,t}\in\mathbb{R}^K$ denote the seed-coordinate increment computed at local step $t$ from $\boldsymbol w_c^{r,t-1}$. All scalar gradients in this step are evaluated at the same model $\boldsymbol w_c^{r,t-1}$, accumulated into $\Delta\boldsymbol a_c^{r,t}$, and applied jointly: \begin{align} \boldsymbol a_c^{r,t} &= \boldsymbol a_c^{r,t-1} + \Delta\boldsymbol a_c^{r,t}, \label{eq:local_accumulator_update} \\ \boldsymbol w_c^{r,t} &= \boldsymbol w_c^{r,t-1} - \eta \sum_{k=1}^{K} \Delta a_{c,k}^{r,t}\boldsymbol z_k. \label{eq:local_model_update} \end{align} Equivalently, \[ \boldsymbol w_c^{r,t} = \boldsymbol w^{r-1} - \eta\sum_{k=1}^{K} a_{c,k}^{r,t}\boldsymbol z_k. \] After local training, the server aggregates the final local accumulators of the participating clients and adds the resulting increment to $\boldsymbol a^{r-1}$, thereby obtaining $\boldsymbol a^{r}$. Exact unlearning additionally requires the realized randomness that determines the training trajectory to be recoverable. We therefore record the sampled clients, mini-batches, and perturbation seeds needed to reproduce each realized update. Distinct sources of randomness used by the algorithms are generated from independent random streams unless explicitly coupled by the unlearning procedure.

\subsection{Subsampling-Stable Learning and Unlearning}
\label{sec:subsampling_framework}

\textbf{Learning: FedZOO-S.}
A deletion does not affect the realized training trajectory if the deleted
sample has never been used to compute an update.
We exploit this property through two levels of subsampling: only a subset of
clients participates in each communication round, and each participating
client computes its local zeroth-order updates from randomly sampled
mini-batches.
This gives our first learning algorithm, \textsc{FedZOO-S}.

At round $r$, the server uniformly samples
$\mathcal C^r\subseteq[M]$ with $|\mathcal C^r|=m$ and sends the current
global accumulator $\boldsymbol a^{r-1}$ to the participating clients.
Each $c\in\mathcal C^r$ initializes
\[
\boldsymbol w_c^{r,0}
=
\mathsf{Rec}(\boldsymbol a^{r-1}),
\qquad
\boldsymbol a_c^{r,0}
=
\boldsymbol 0_K,
\]
and performs $E$ local steps.

At local step $t$, client $c$ samples a mini-batch
$\mathcal B_c^{r,t}\subseteq\mathcal D_c$ of size $b$ uniformly without
replacement.
For each $x_i\in\mathcal B_c^{r,t}$, it further samples $q$ distinct seeds
\[
\mathbb S_{c,i}^{r,t}
=
\{s_{i,1},\ldots,s_{i,q}\}
\subseteq\mathbb S
\]
uniformly without replacement.
Let
$\boldsymbol z_{i,j}=\mathsf{Perturb}(s_{i,j})$.
All scalar gradients within the local step are evaluated at the same
pre-update model:
\begin{equation}
\label{eq:fedzoos_scalar}
g_{i,j}^{r,t}
=
\frac{
\ell(\boldsymbol w_c^{r,t-1}+\mu\boldsymbol z_{i,j};x_i)
-
\ell(\boldsymbol w_c^{r,t-1}-\mu\boldsymbol z_{i,j};x_i)
}{
2\mu
}.
\end{equation}
The $bq$ scalar gradients jointly form one mini-batch averaged zeroth-order
update.
For each $s_k\in\mathbb S$, define
\begin{equation}
\label{eq:fedzoos_step_increment}
\Delta a_{c,k}^{r,t}
=
\frac{1}{bq}
\sum_{\substack{
x_i\in\mathcal B_c^{r,t},\,j\in[q]:\\
s_{i,j}=s_k
}}
g_{i,j}^{r,t},
\end{equation}
where an empty sum is zero.
The client then applies
\begin{align}
\boldsymbol a_c^{r,t}
&=
\boldsymbol a_c^{r,t-1}
+
\Delta\boldsymbol a_c^{r,t},
\label{eq:fedzoos_local_acc}
\\
\boldsymbol w_c^{r,t}
&=
\boldsymbol w_c^{r,t-1}
-
\eta
\sum_{k=1}^{K}
\Delta a_{c,k}^{r,t}\boldsymbol z_k.
\label{eq:fedzoos_local_model}
\end{align}
For subsequent unlearning, the client records
$\mathcal B_c^{r,t}$,
$\{\mathbb S_{c,i}^{r,t}\}_{x_i\in\mathcal B_c^{r,t}}$,
and $\boldsymbol a_c^{r,t}$.

After the local steps, the server updates
\begin{equation}
\label{eq:fedzoos_server}
\boldsymbol a^r
=
\boldsymbol a^{r-1}
+
\frac{1}{m}
\sum_{c\in\mathcal C^r}
\boldsymbol a_c^{r,E},
\end{equation}
and records $\mathcal C^r$ and $\boldsymbol a^r$.
Algorithm~\ref{algo:FedZOOS} summarizes \textsc{FedZOO-S}.

The two levels of subsampling directly enable efficient unlearning.
A sample affects the realized trajectory only when its client participates
in a round and the sample is subsequently selected into a local mini-batch.
If this never occurs before an unlearning request, deleting the sample does
not invalidate the realized trajectory.
Otherwise, the first local step containing the sample identifies the earliest
point from which replay is necessary.

\medskip
\noindent\textbf{Exact unlearning: FedZOUL-S.}
Let $\mathcal C_{ul}$ be the clients involved in an unlearning request, and
let $\mathcal D_{c,ul}\subseteq\mathcal D_c$ contain the samples to be deleted
from client $c$.
Suppose the request arrives immediately after local step
$(R_{ul},t_{ul})$.
For later use, define the next training position as
\[
\mathsf{Next}(r,t)
=
\begin{cases}
(r,t+1), & t<E,\\
(r+1,1), & t=E.
\end{cases}
\]
Using the client-participation sets and mini-batches recorded during
\textsc{FedZOO-S}, \textsc{FedZOUL-S} searches chronologically over the
realized prefix up to $(R_{ul},t_{ul})$ for the earliest step
$(r^\star,t^\star)$ satisfying
\begin{equation}
\label{eq:first_affected_step}
\mathcal B_c^{r^\star,t^\star}
\cap
\mathcal D_{c,ul}
\neq
\varnothing
\end{equation}
for some
$c\in\mathcal C_{ul}\cap\mathcal C^{r^\star}$.

If no such step exists, none of the deleted samples has contributed to the
realized trajectory.
After removing the requested samples, training therefore continues from
$\mathsf{Next}(R_{ul},t_{ul})$ using the current state, without replaying
any previous step.

Otherwise, the algorithm preserves the entire trajectory before
$(r^\star,t^\star)$ and restores the optimization state immediately preceding
this step.
Specifically, it loads $\boldsymbol a^{r^\star-1}$ and, for each
$c\in\mathcal C^{r^\star}$, restores
$\boldsymbol a_c^{r^\star,t^\star-1}$ and reconstructs
\begin{equation}
\label{eq:fedzouls_restore}
\boldsymbol w_c^{r^\star,t^\star-1}
=
\mathsf{Rec}(\boldsymbol a^{r^\star-1})
-
\eta
\sum_{k=1}^{K}
a_{c,k}^{r^\star,t^\star-1}\boldsymbol z_k.
\end{equation}
Training then resumes from $(r^\star,t^\star)$ on the retained datasets
using the same \textsc{FedZOO-S} procedure.
All newly generated states replace their original counterparts.

The preserved prefix contains no update involving a deleted sample. The coupling analysis in Section~\ref{sec:fedzoos_theory} shows that the two branches of \textsc{FedZOUL-S} jointly transport the original training law to the retained-data law: the realized prefix is preserved when no deleted sample appears, and the trajectory is regenerated from the first affected step otherwise. Therefore, the resulting model and saved state have the same distribution as a \textsc{FedZOO-S} run on the retained datasets.
Algorithm~\ref{algo:FedZOULS} summarizes the procedure.
We state the algorithm for set-valued deletion requests; the theoretical
analysis considers a single deletion, with multiple requests handled
sequentially.

The limitation of \textsc{FedZOUL-S} is that once a deleted sample has
contributed to a local update, the trajectory must be replayed from that
update onward.
The next learning--unlearning pair adds Gaussian randomization so that an
affected update can sometimes be reused as well.

\subsection{Noise-Stable Learning and Unlearning}
\label{sec:noise_framework}

\textbf{Learning: FedZOO-NS.}
\textsc{FedZOO-S} avoids replay when deleted samples have not contributed to
the realized trajectory, but an affected update still requires replay of the
remaining suffix.
To further reduce this cost, we randomize each local zeroth-order update with
Gaussian noise.
The resulting overlap between update distributions on neighboring datasets
allows some affected realized updates to be reused through coupling.
We call the resulting learning algorithm \textsc{FedZOO-NS}.

Client, mini-batch, and perturbation-seed sampling follow
\textsc{FedZOO-S}.
At local step $(r,t)$, let
$\mathcal B_c^{r,t}$ be the sampled mini-batch and
$\mathbb S_{c,i}^{r,t}$ the $q$ perturbation seeds sampled for
$x_i\in\mathcal B_c^{r,t}$.
All scalar gradients are evaluated at the same pre-update model
$\boldsymbol w_c^{r,t-1}$:
\begin{equation}
\label{eq:fedzoons_raw}
g_{i,j}^{r,t}
=
\frac{
\ell(\boldsymbol w_c^{r,t-1}+\mu\boldsymbol z_{i,j};x_i)
-
\ell(\boldsymbol w_c^{r,t-1}-\mu\boldsymbol z_{i,j};x_i)
}{
2\mu
},
\end{equation}
where
$\boldsymbol z_{i,j}=\mathsf{Perturb}(s_{i,j})$.
We clip each scalar gradient at level $C$ and normalize it according to the
mini-batch averaged update:
\begin{equation}
\label{eq:fedzoons_clean_coeff}
h_{i,j}^{r,t}
=
\frac{1}{bq}
\operatorname{clip}_C(g_{i,j}^{r,t}).
\end{equation}

Fix a deterministic ordering of the $bq$ sample--seed pairs and collect
their coefficients into
$\boldsymbol h_c^{r,t}\in\mathbb R^{bq}$.
Client $c$ independently draws
\[
\boldsymbol\xi_c^{r,t}
\sim
\mathcal N(\boldsymbol0,\sigma_\rho^2 I_{bq})
\]
and forms
\begin{equation}
\label{eq:fedzoons_noisy_vector}
\widetilde{\boldsymbol h}_c^{r,t}
=
\boldsymbol h_c^{r,t}
+
\boldsymbol\xi_c^{r,t}.
\end{equation}
The Gaussian noise is independent of the client, mini-batch, and
perturbation-seed sampling.

The noisy coefficient vector is mapped to the seed-coordinate increment
$\Delta\widetilde{\boldsymbol a}_c^{r,t}\in\mathbb R^K$.
For each seed $s_k\in\mathbb S$, the increment is
$\Delta\widetilde a_{c,k}^{r,t}
=
\sum_{\substack{
x_i\in\mathcal B_c^{r,t},\,j\in[q]:\\
s_{i,j}=s_k
}}
\widetilde h_{i,j}^{r,t}$,
where $\widetilde h_{i,j}^{r,t}$ denotes the component of
$\widetilde{\boldsymbol h}_c^{r,t}$ corresponding to the sampled pair
$(x_i,s_{i,j})$. The sum includes all sampled pairs sharing $s_k$
and is zero if no such pair is selected.
The client then updates
\begin{align}
\boldsymbol a_c^{r,t}
&=
\boldsymbol a_c^{r,t-1}
+
\Delta\widetilde{\boldsymbol a}_c^{r,t},
\label{eq:fedzoons_acc_update}
\\
\boldsymbol w_c^{r,t}
&=
\boldsymbol w_c^{r,t-1}
-
\eta
\sum_{k=1}^{K}
\Delta\widetilde a_{c,k}^{r,t}\boldsymbol z_k.
\label{eq:fedzoons_model_update}
\end{align}
After $E$ local steps, the server performs the same uniform aggregation as
in \textsc{FedZOO-S}:
\begin{equation}
\label{eq:fedzoons_server}
\boldsymbol a^r
=
\boldsymbol a^{r-1}
+
\frac{1}{m}
\sum_{c\in\mathcal C^r}
\boldsymbol a_c^{r,E}.
\end{equation}

For subsequent unlearning, each client records
$\mathcal B_c^{r,t}$,
$\{\mathbb S_{c,i}^{r,t}\}_{x_i\in\mathcal B_c^{r,t}}$,
$\widetilde{\boldsymbol h}_c^{r,t}$,
and $\boldsymbol a_c^{r,t}$.
The clean vector $\boldsymbol h_c^{r,t}$ need not be stored, since it can be
recomputed from the saved pre-update state, mini-batch, and perturbation
seeds.
Algorithm~\ref{algo:FedZOONS} summarizes \textsc{FedZOO-NS}.

\medskip
\noindent\textbf{Exact unlearning: FedZOUL-NS.}
Consider a deletion request
$\{\mathcal D_{c,ul}\}_{c\in\mathcal C_{ul}}$
arriving immediately after local step $(R_{ul},t_{ul})$.
\textsc{FedZOUL-NS} scans the realized prefix up to the request in
chronological order.
For an affected client
$c\in\mathcal C_{ul}\cap\mathcal C^r$, define
\[
\mathcal B_{c,ul}^{r,t}
=
\mathcal B_c^{r,t}\cap\mathcal D_{c,ul}.
\]
If
$\mathcal B_{c,ul}^{r,t}=\varnothing$,
the recorded local update is unaffected and is preserved.
Otherwise, the algorithm couples the recorded mini-batch to a mini-batch
distributed as a uniform sample from the retained local dataset.
It removes
$\mathcal B_{c,ul}^{r,t}$
and replaces these samples by the same number of retained samples, drawn
uniformly without replacement from the retained dataset while excluding the
samples already present in the mini-batch.
Denote the resulting mini-batch by
$(\mathcal B_c^{r,t})'$.
The perturbation-seed sets of retained samples are unchanged, while each
replacement sample inherits the seed set associated with the deleted sample
that it replaces.
Since data sampling and perturbation-seed sampling are independent, this
produces the required coupled realization of the sampling randomness.

Using the saved pre-update model, the coupled mini-batch, and the associated
seeds, the client recomputes the clean coefficient vector
$(\boldsymbol h_c^{r,t})'$.
Conditioned on the pre-update state and sampling realizations, the original
and retained-data randomized updates follow
\[
P
=
\mathcal N
\left(
\boldsymbol h_c^{r,t},
\sigma_\rho^2 I_{bq}
\right),
\qquad
Q
=
\mathcal N
\left(
(\boldsymbol h_c^{r,t})',
\sigma_\rho^2 I_{bq}
\right).
\]
Let $p$ and $q$ denote their densities.
For the recorded realization
$\widetilde{\boldsymbol h}_c^{r,t}$,
the rejection step accepts with probability
\begin{equation}
\label{eq:rr_acceptance}
\alpha_c^{r,t}
=
\min\left\{
1,\,
\frac{
q(\widetilde{\boldsymbol h}_c^{r,t})
}{
p(\widetilde{\boldsymbol h}_c^{r,t})
}
\right\}.
\end{equation}
Equivalently,
\begin{equation}
\label{eq:rr_acceptance_gaussian}
\alpha_c^{r,t}
=
\min\left\{
1,\,
\exp\!\left(
\frac{
\|\widetilde{\boldsymbol h}_c^{r,t}
-\boldsymbol h_c^{r,t}\|_2^2
-
\|\widetilde{\boldsymbol h}_c^{r,t}
-(\boldsymbol h_c^{r,t})'\|_2^2
}{
2\sigma_\rho^2
}
\right)
\right\}.
\end{equation}

If the test accepts, the recorded noisy coefficient vector is reused:
\[
(\widetilde{\boldsymbol h}_c^{r,t})'
=
\widetilde{\boldsymbol h}_c^{r,t}.
\]
The local model and accumulator therefore remain unchanged; only the saved
mini-batch and its sample--seed correspondence are replaced by the coupled
retained-data realization.

If the test rejects, the algorithm uses reflection to obtain a realization
from the retained-data transition.
Let
\[
\boldsymbol d_c^{r,t}
=
(\boldsymbol h_c^{r,t})'
-
\boldsymbol h_c^{r,t},
\qquad
\boldsymbol m_c^{r,t}
=
\frac{
\boldsymbol h_c^{r,t}
+
(\boldsymbol h_c^{r,t})'
}{2}.
\]
For
$\boldsymbol d_c^{r,t}\neq\boldsymbol0$, define
\[
\boldsymbol e_c^{r,t}
=
\frac{
\boldsymbol d_c^{r,t}
}{
\|\boldsymbol d_c^{r,t}\|_2
}.
\]
The reflected coefficient vector is
\begin{equation}
\label{eq:rr_reflection}
(\widetilde{\boldsymbol h}_c^{r,t})'
=
\widetilde{\boldsymbol h}_c^{r,t}
-
2
\left\langle
\widetilde{\boldsymbol h}_c^{r,t}
-
\boldsymbol m_c^{r,t},
\boldsymbol e_c^{r,t}
\right\rangle
\boldsymbol e_c^{r,t}.
\end{equation}
If the two clean coefficient vectors coincide, then $P=Q$ and the step is
accepted directly.

The reflected vector is mapped to the fixed seed-coordinate space to obtain
$\Delta\widetilde{\boldsymbol a}_c^{r,t}$ and is applied to the state
immediately preceding $(r,t)$:
\begin{align}
\boldsymbol a_c^{r,t}
&=
\boldsymbol a_c^{r,t-1}
+
\Delta\widetilde{\boldsymbol a}_c^{r,t},
\\
\boldsymbol w_c^{r,t}
&=
\boldsymbol w_c^{r,t-1}
-
\eta
\sum_{k=1}^{K}
\Delta\widetilde a_{c,k}^{r,t}\boldsymbol z_k.
\end{align}
The saved state of this local step is then replaced by
\[
\left(
(\mathcal B_c^{r,t})',
\{\mathbb S_{c,i}^{r,t}\}_{x_i\in(\mathcal B_c^{r,t})'},
(\widetilde{\boldsymbol h}_c^{r,t})',
\boldsymbol a_c^{r,t}
\right),
\]
namely, the coupled mini-batch and sample--seed correspondence, the
reflected noisy coefficient vector, and the corrected local accumulator.
All subsequent states depend on this corrected endpoint and are therefore
regenerated on the retained datasets.

If a rejection occurs at the last local step $t=E$, the corrected local
endpoint first replaces its original counterpart and the server recomputes
\[
\boldsymbol a^r
=
\boldsymbol a^{r-1}
+
\frac{1}{m}
\sum_{c\in\mathcal C^r}
\boldsymbol a_c^{r,E}
\]
before training proceeds from $\mathsf{Next}(r,E)=(r+1,1)$.

If every affected step up to the request is accepted, the realized
optimization state remains unchanged.
After removing the requested samples, training continues from
$\mathsf{Next}(R_{ul},t_{ul})$.
If a rejection occurs at $(r^\star,t^\star)$, training instead continues
from $\mathsf{Next}(r^\star,t^\star)$ using the corrected state.
Algorithm~\ref{algo:FedZOULNS} summarizes the procedure.

Thus, \textsc{FedZOUL-NS} first exploits subsampling to preserve unaffected
steps and then uses Gaussian coupling to reuse some affected randomized
updates.
Replay is required only after the first rejected coupling.

\begin{algorithm}[t]
\caption{\texttt{FedZOO-S}: Subsampling-Stable Federated ZOO}
\label{algo:FedZOOS}
\KwIn{$\boldsymbol w^0,R,E,K,q,b,m,\eta,\mu$.}
\KwOut{$\boldsymbol w^R$.}

Initialize fixed seed pool
$\mathbb S=\{s_1,\ldots,s_K\}$ and
$\boldsymbol a^0=\boldsymbol 0_K$\;

\For{$r=1,\ldots,R$}{
    Sample $\mathcal C^r\subseteq[M]$, $|\mathcal C^r|=m$,
    uniformly without replacement\;

    \ForEach{$c\in\mathcal C^r$ \textbf{in parallel}}{
        $\boldsymbol w_c^{r,0}\gets\mathsf{Rec}(\boldsymbol a^{r-1})$;
        $\boldsymbol a_c^{r,0}\gets\boldsymbol 0_K$\;

        \For{$t=1,\ldots,E$}{
            Sample $\mathcal B_c^{r,t}$ and
            $\{\mathbb S_{c,i}^{r,t}\}_{x_i\in\mathcal B_c^{r,t}}$\;

            Compute $\Delta\boldsymbol a_c^{r,t}$
            by Eqs.~\eqref{eq:fedzoos_scalar}--\eqref{eq:fedzoos_step_increment}\;

            Update
            $(\boldsymbol w_c^{r,t},\boldsymbol a_c^{r,t})$
            by Eqs.~\eqref{eq:fedzoos_local_acc}--\eqref{eq:fedzoos_local_model}\;

            $\mathsf{save}\!\left(
            \mathcal B_c^{r,t},
            \{\mathbb S_{c,i}^{r,t}\},
            \boldsymbol a_c^{r,t}
            \right)$\;
        }
    }

    $\boldsymbol a^r
    \gets
    \boldsymbol a^{r-1}
    +
    \frac1m
    \sum_{c\in\mathcal C^r}
    \boldsymbol a_c^{r,E}$\;

    $\mathsf{save}(\mathcal C^r,\boldsymbol a^r)$\;
}

\KwRet{$\mathsf{Rec}(\boldsymbol a^R)$}\;
\end{algorithm}

\begin{algorithm}[t]
\caption{\texttt{FedZOUL-S}: Exact Unlearning by Subsampling}
\label{algo:FedZOULS}
\KwIn{Deletion request
$\{\mathcal D_{c,ul}\}_{c\in\mathcal C_{ul}}$
after $(R_{ul},t_{ul})$; saved \texttt{FedZOO-S} states.}
\KwOut{Unlearned model $\boldsymbol w^R$.}

$\mathsf{affected}\gets\mathsf{false}$\;

\For{each $(r,t)\preceq(R_{ul},t_{ul})$ in chronological order}{
    \ForEach{$c\in\mathcal C_{ul}\cap\mathcal C^r$ \textbf{in parallel}}{
        $\mathsf{affected}_c
        \gets
        \mathbf 1\!\left\{
        \mathcal B_c^{r,t}\cap\mathcal D_{c,ul}
        \neq\varnothing
        \right\}$\;
    }

    \If{$\bigvee_{c\in\mathcal C_{ul}\cap\mathcal C^r}
        \mathsf{affected}_c$}{
        $(r^\star,t^\star)\gets(r,t)$\;
        $\mathsf{affected}\gets\mathsf{true}$\;
        \textbf{break}\;
    }
}

Remove $\mathcal D_{c,ul}$ from each affected local dataset\;

\eIf{$\mathsf{affected}$}{
    Restore the saved state immediately before
    $(r^\star,t^\star)$\;
    Continue \texttt{FedZOO-S} from
    $(r^\star,t^\star)$ on the retained datasets\;
}{
    Continue \texttt{FedZOO-S} from
    $\mathsf{Next}(R_{ul},t_{ul})$ using the current state\;
}

\KwRet{$\boldsymbol w^R$}\;
\end{algorithm}

\begin{algorithm}[t]
\caption{\texttt{FedZOO-NS}: Noise-Stable Federated ZOO}
\label{algo:FedZOONS}
\KwIn{$\boldsymbol w^0,R,E,K,q,b,m,\eta,\mu,C,\sigma_\rho$.}
\KwOut{$\boldsymbol w^R$.}

Initialize fixed seed pool
$\mathbb S=\{s_1,\ldots,s_K\}$
and $\boldsymbol a^0=\boldsymbol0_K$\;

\For{$r=1,\ldots,R$}{
    Sample $\mathcal C^r\subseteq[M]$, $|\mathcal C^r|=m$,
    uniformly without replacement\;

    \ForEach{$c\in\mathcal C^r$ \textbf{in parallel}}{
        $\boldsymbol w_c^{r,0}
        \gets\mathsf{Rec}(\boldsymbol a^{r-1})$;
        $\boldsymbol a_c^{r,0}\gets\boldsymbol0_K$\;

        \For{$t=1,\ldots,E$}{
            Sample $\mathcal B_c^{r,t}$ and
            $\{\mathbb S_{c,i}^{r,t}\}_{x_i\in\mathcal B_c^{r,t}}$\;

            Compute $\boldsymbol h_c^{r,t}$
            by Eqs.~\eqref{eq:fedzoons_raw}--\eqref{eq:fedzoons_clean_coeff}\;

            Draw
            $\boldsymbol\xi_c^{r,t}
            \sim\mathcal N(\boldsymbol0,\sigma_\rho^2I_{bq})$
            and set
            $\widetilde{\boldsymbol h}_c^{r,t}
            \gets
            \boldsymbol h_c^{r,t}+\boldsymbol\xi_c^{r,t}$\;

            Map $\widetilde{\boldsymbol h}_c^{r,t}$ to
            $\Delta\widetilde{\boldsymbol a}_c^{r,t}$\;

            Update
            $(\boldsymbol w_c^{r,t},\boldsymbol a_c^{r,t})$
            by Eqs.~\eqref{eq:fedzoons_acc_update}--\eqref{eq:fedzoons_model_update}\;

            $\mathsf{save}\!\left(
            \mathcal B_c^{r,t},
            \{\mathbb S_{c,i}^{r,t}\},
            \widetilde{\boldsymbol h}_c^{r,t},
            \boldsymbol a_c^{r,t}
            \right)$\;
        }
    }

    $\boldsymbol a^r
    \gets
    \boldsymbol a^{r-1}
    +
    \frac1m
    \sum_{c\in\mathcal C^r}
    \boldsymbol a_c^{r,E}$\;

    $\mathsf{save}(\mathcal C^r,\boldsymbol a^r)$\;
}

\KwRet{$\mathsf{Rec}(\boldsymbol a^R)$}\;
\end{algorithm}

\begin{algorithm}[t]
\caption{\texttt{FedZOUL-NS}: Exact Unlearning by Rejection--Reflection}
\label{algo:FedZOULNS}
\KwIn{Deletion request
$\{\mathcal D_{c,ul}\}_{c\in\mathcal C_{ul}}$
after $(R_{ul},t_{ul})$; saved \texttt{FedZOO-NS} states.}
\KwOut{Unlearned model $\boldsymbol w^R$.}

$\mathsf{reject}\gets\mathsf{false}$\;

\For{each $(r,t)\preceq(R_{ul},t_{ul})$ in chronological order}{
    \ForEach{$c\in\mathcal C_{ul}\cap\mathcal C^r$ \textbf{in parallel}}{
        $\mathsf{reject}_c\gets\mathsf{false}$\;

        $\mathcal B_{c,ul}^{r,t}
        \gets
        \mathcal B_c^{r,t}\cap\mathcal D_{c,ul}$\;

        \If{$\mathcal B_{c,ul}^{r,t}\neq\varnothing$}{
            Construct coupled mini-batch
            $(\mathcal B_c^{r,t})'$
            and its sample--seed correspondence\;

            Recompute
            $\boldsymbol h_c^{r,t}$ and
            $(\boldsymbol h_c^{r,t})'$\;

            Draw $u\sim\operatorname{Unif}(0,1)$\;

            \eIf{$u\le\alpha_c^{r,t}$ in Eq.~\eqref{eq:rr_acceptance}}{
                Replace the saved mini-batch and sample--seed
                correspondence by the coupled realization\;
            }{
                Compute
                $(\widetilde{\boldsymbol h}_c^{r,t})'$
                by Eq.~\eqref{eq:rr_reflection}\;

                Map
                $(\widetilde{\boldsymbol h}_c^{r,t})'$
                to
                $\Delta\widetilde{\boldsymbol a}_c^{r,t}$\;

                $\boldsymbol a_c^{r,t}
                \gets
                \boldsymbol a_c^{r,t-1}
                +
                \Delta\widetilde{\boldsymbol a}_c^{r,t}$\;
            
                $\boldsymbol w_c^{r,t}
                \gets
                \boldsymbol w_c^{r,t-1}
                -
                \eta
                \sum_{k=1}^{K}
                \Delta\widetilde a_{c,k}^{r,t}\boldsymbol z_k$\;
            
                $\mathsf{save}\!\left(
                (\mathcal B_c^{r,t})',
                \{\mathbb S_{c,i}^{r,t}\}_{x_i\in(\mathcal B_c^{r,t})'},
                (\widetilde{\boldsymbol h}_c^{r,t})',
                \boldsymbol a_c^{r,t}
                \right)$\;
            
                $\mathsf{reject}_c\gets\mathsf{true}$\;
            }
        }
    }

    \If{$\bigvee_{c\in\mathcal C_{ul}\cap\mathcal C^r}
        \mathsf{reject}_c$}{
        $(r^\star,t^\star)\gets(r,t)$;
        $\mathsf{reject}\gets\mathsf{true}$\;
        \textbf{break}\;
    }
}

Remove $\mathcal D_{c,ul}$ from each affected local dataset\;

\eIf{$\mathsf{reject}$}{
    \If{$t^\star=E$}{
        Recompute $\boldsymbol a^{r^\star}$ from the corrected
        local endpoints by Eq.~\eqref{eq:fedzoons_server}\;
    }
    Continue \texttt{FedZOO-NS} from
    $\mathsf{Next}(r^\star,t^\star)$ using the corrected state\;
}{
    Continue \texttt{FedZOO-NS} from
    $\mathsf{Next}(R_{ul},t_{ul})$ using the current state\;
}

\KwRet{$\boldsymbol w^R$}\;
\end{algorithm}

\section{Theoretical Analysis}
\label{sec:theory}

We now explain why the proposed learning and unlearning procedures work. Two questions are central. First, after a sample is deleted, can we recover exactly the state that would have been obtained by training on the retained data, without rerunning everything from the beginning? Second, because all zeroth-order updates reuse the same finite seed pool, what optimization guarantee can we still give for the original federated objective? We keep the discussion here at the level of the main ideas and results; the complete statements and proofs are given in Appendix~\ref{sec:theory_appendix}.

For reference, $T=RE$ denotes the total number of local positions. For a deletion of $x^-$ from client $c^-$, let $n_{c^-}$ be that client's data size immediately before deletion, and let $\underline n$ be a lower bound on such local dataset sizes. If the request arrives after round $R_{ul}$ and local step $t_{ul}$, then $T_{ul}=(R_{ul}-1)E+t_{ul}$ positions have already been executed.

\subsection{Stability, Exact Unlearning, and Cost}

The analysis of \textsc{FedZOO-S} starts from a simple observation: deleting a sample can only change the realized trajectory if its client is selected and the sample is actually drawn into a local mini-batch. We therefore couple the executions before and after deletion by sharing all randomness until that first event occurs. Up to that point, the two executions are identical. Since one local position hits the deleted sample with probability $(m/M)(b/n_{c^-})$, the probability of a hit within the first $\tau$ positions is at most $\min\{1,\tau mb/(Mn_{c^-})\}$. This gives the TV-stability result in Theorem~\ref{thm:fedzoos_tv} of Appendix~\ref{sec:fedzoos_theory}; in particular, $Tmb/(M\underline n)\le\rho$ is sufficient for $\rho$-TV stability.

The same coupling already contains the unlearning algorithm. If the deleted sample was never used, the current state is already distributed exactly as training on the retained data. If it was used, the two executions agree until its first occurrence, so only the computation after that point has to be regenerated. This is the idea behind \textsc{FedZOUL-S}. Theorem~\ref{thm:fedzouls_exact} in Appendix~\ref{sec:fedzoos_theory} proves exact unlearning for one deletion, and Corollary~\ref{cor:fedzouls_sequential} in the same subsection extends the argument to a non-adaptive stream of deletions.

\textsc{FedZOO-NS} begins with the same subsampling effect, but the added Gaussian noise gives us another way to keep previously computed updates. When deletion changes a sampled mini-batch, the mean of the noisy coefficient vector changes, but its covariance does not. The two Gaussian transition laws therefore overlap. Appendix~\ref{sec:fedzoons_stability} quantifies this overlap by coupling the mini-batches and sample--seed pairs first and then bounding the resulting Gaussian discrepancy. Theorem~\ref{thm:fedzoons_output_tv} gives the final TV bound. A convenient sufficient noise level is
\begin{equation}
\label{eq:main_ns_sigma_overview}
\sigma_\rho
=
\max\left\{
\frac{2C}{b\sqrt q},
\frac{2C}{\underline n\sqrt q\,\rho}
\sqrt{\frac{Tm}{M}}
\right\},
\end{equation}
as stated in Corollary~\ref{cor:fedzoons_noise_calibration} of Appendix~\ref{sec:fedzoons_stability}. The factor $\sqrt{m/M}$ appears because client participation is sampled once per round and then shared across the $E$ local steps.

For unlearning, this Gaussian overlap is used constructively. At an affected local step, \textsc{FedZOUL-NS} first asks whether the previously generated noisy update is also valid under the retained-data transition. If so, it is kept. Otherwise, a reflection coupling produces an update with the correct retained-data law, and training continues from the next position. Theorem~\ref{thm:fedzoulns_exact} in Appendix~\ref{sec:fedzoulns_exactness} proves that this procedure gives exact unlearning of the complete maintained state, and Corollary~\ref{cor:fedzoulns_sequential} in the same subsection covers a non-adaptive stream of deletions. Proposition~\ref{prop:fedzoulns_replay} further bounds the probability that replay is triggered before a request at $T_{ul}$; under Eq.~\eqref{eq:main_ns_sigma_overview}, it is at most $\min\{1,\rho T_{ul}\sqrt{m/M}/\sqrt{2\pi T}\}$. This replay probability is different from the final-model TV parameter $\rho$: one belongs to the particular transitionwise coupling used for unlearning, while the other measures the distance between the final output laws.

The cost analysis follows the same picture. For \textsc{FedZOUL-S}, work is repeated only if the deleted sample was used, and replay begins at its first occurrence. For \textsc{FedZOUL-NS}, affected updates are checked first and replay begins only after the first rejection, so some computation can still be retained even after the deleted sample has appeared. Theorem~\ref{thm:expected_unlearning_queries} and Corollary~\ref{cor:stream_unlearning_cost} in Appendix~\ref{sec:expected_unlearning_cost} give the corresponding loss-evaluation, communication, and request-stream bounds. The important point is that the saving comes from reusing computation that can be coupled exactly to retained-data training; no approximation to retraining is introduced.

\subsection{Optimization with a Finite Seed Pool}

We next ask whether the same finite seed pool that makes communication and unlearning efficient still permits meaningful optimization. Seed reuse does not change the federated objective; it only restricts the directions in which the model can move. Let $\boldsymbol Z=[\boldsymbol z_1,\ldots,\boldsymbol z_K]$ and $\boldsymbol A=\boldsymbol Z/\sqrt K$. Every iterate can be written as $\boldsymbol w=\boldsymbol w^0+\boldsymbol A\boldsymbol u$ for some $\boldsymbol u\in\mathbb R^K$. Hence the algorithm is optimizing the original objective on the affine space generated by the seed pool, with
\begin{equation}
\label{eq:main_projected_gradient_overview}
\|\nabla F_{\boldsymbol Z}(\boldsymbol u)\|^2
=
\frac1K\|\boldsymbol Z^\top\nabla f(\boldsymbol w)\|^2.
\end{equation}
The quantity on the right is therefore the gradient energy seen along the shared seed directions. Because the columns of $\boldsymbol Z$ need not be orthogonal, it should not be interpreted as an orthogonal projection. Appendix~\ref{sec:fixed_seed_analysis} gives the corresponding parameterization in detail.

The convergence proof then follows the actual seed-coordinate update produced at each local step. What matters is its conditional bias, its variance, and its second moment. For \textsc{FedZOO-S}, the bias comes from the two-point finite difference. For \textsc{FedZOO-NS}, clipping adds bias and the Gaussian perturbation adds variance. Appendix~\ref{sec:finite_seed_utility} derives these quantities and states the required regularity conditions. Writing $\beta_{\mathsf A}$, $\nu_{\mathsf A}$, and $G_{\mathsf A}$ for the corresponding uniform bounds, $L_{\boldsymbol Z}$ for the smoothness constant in seed coordinates, $\zeta_{\boldsymbol Z}$ for client heterogeneity, and $\chi_m=(M-m)/(m(M-1))$ for the finite-population client-sampling factor when $M>1$ (with $\chi_m=0$ for $M=1$), Theorem~\ref{thm:finite_seed_stationarity} in Appendix~\ref{sec:finite_seed_utility} shows that, whenever $4\eta L_{\boldsymbol Z}E\le1$,
\begin{align}
&\frac1R\sum_{r=1}^R
\mathbb E\left[
\frac1K\|\boldsymbol Z^\top\nabla f(\boldsymbol w^{r-1})\|^2
\,\middle|\,\boldsymbol Z
\right]
\notag\\
&\quad\le
\frac{2(f(\boldsymbol w^0)-f^\star)}{\eta ER}
+6\beta_{\mathsf A}^2
+8L_{\boldsymbol Z}\eta\left(
E\chi_m\zeta_{\boldsymbol Z}^2+\frac{\nu_{\mathsf A}^2}{m}
\right)
\notag\\
&\qquad
+2L_{\boldsymbol Z}^2\eta^2E^2G_{\mathsf A}^2,
\label{eq:main_convergence_overview}
\end{align}
where $f^\star$ is any finite lower bound of the global objective. The four terms have a simple interpretation: finite-horizon optimization error, finite-difference or clipping bias, stochastic variation together with client heterogeneity, and the drift accumulated across local steps. In \textsc{FedZOO-NS}, the noise required for stability enters the variance and second-moment terms, so the cost of making unlearning easier appears directly in the optimization bound.

For a fixed training horizon $ER$, the stepsize $\eta$ balances the first term, which scales as $1/(\eta ER)$, against the leading stochastic term, which grows linearly with $\eta$. Corollary~\ref{cor:balanced_seed_stationarity} in Appendix~\ref{sec:finite_seed_utility} gives the resulting balanced rate when the balancing choice also respects the stepsize condition above. Since $\chi_m\le1/m$, the leading stochastic term scales as $\mathcal O(1/\sqrt{mER})$. Thus, for fixed $E$ and while this term dominates, increasing the number of participating clients gives the usual linear speedup in communication rounds. The same bound also shows where this stops: finite-difference or clipping bias does not vanish with $m$, local drift eventually matters, and \textsc{FedZOO-NS} incurs an additional variance floor from the Gaussian noise used for stability.

\subsection{From Seed-Subspace to Full-Space Guarantees}

The previous result controls exactly what the algorithm can observe through the shared seed directions. By itself, that need not imply full-space stationarity. If the seed span misses a component of the gradient, then $\boldsymbol Z^\top\nabla f(\boldsymbol w)$ can be small even when $\nabla f(\boldsymbol w)$ is not. To bridge the two, Appendix~\ref{sec:full_space_utility} assumes that the seed directions capture the gradients encountered during training: for some $\kappa_{\boldsymbol Z}>0$ and $\varepsilon_{\mathrm{cov}}\ge0$,
\begin{equation}
\label{eq:main_coverage_overview}
\frac1K\|\boldsymbol Z^\top\nabla f(\boldsymbol w)\|^2
\ge
\kappa_{\boldsymbol Z}\|\nabla f(\boldsymbol w)\|^2
-
\varepsilon_{\mathrm{cov}}^2.
\end{equation}
Corollary~\ref{cor:full_space_stationarity} in Appendix~\ref{sec:full_space_utility} then converts the seed-subspace bound into a full-space guarantee, with the expected $1/\kappa_{\boldsymbol Z}$ scaling and additive coverage error $\varepsilon_{\mathrm{cov}}^2/\kappa_{\boldsymbol Z}$.

For the Gaussian seed pool used in our algorithms, this condition has a concrete interpretation. If all global-checkpoint gradients that can arise in the analysis lie in a fixed $r$-dimensional subspace chosen independently of the seed pool, Proposition~\ref{prop:gaussian_seed_coverage} in Appendix~\ref{sec:full_space_utility} shows that sufficiently many Gaussian seed directions form a uniform near-isometry on that subspace. Together with the stepsize condition stated there, Corollary~\ref{cor:gaussian_full_space_stationarity} gives the corresponding high-probability full-space result. The counterexample in the same subsection explains why some such condition is unavoidable when $K<d$: a rank-deficient seed span can make $\boldsymbol Z^\top\nabla f(\boldsymbol w)=0$ while the full gradient remains nonzero.

Finally, exact unlearning closes the loop. After a deletion is processed, the returned state has exactly the same law as a fresh run on the retained data. There is therefore no separate optimization error caused by unlearning. If the retained objective satisfies the same regularity conditions and the inherited stepsize remains valid, the convergence analysis applies again with the retained-data constants. Corollaries~\ref{cor:retained_finite_seed_utility} and~\ref{cor:retained_full_space_utility} in Appendix~\ref{sec:post_unlearning_utility} state the resulting seed-subspace and full-space guarantees. In this sense, the two halves of the theory fit together naturally: stability tells us which past computation can be kept exactly, and the fixed-seed convergence analysis tells us the quality of the resulting model.

\section{Experiments}
In this section, we evaluate the proposed framework across different tasks and settings, focusing on optimization performance, unlearning effectiveness, and resource efficiency. We further examine how key parameters affect the trade-off between model utility and system overhead.

\begin{table*}[t]
\centering
\caption{Comparison of accuracy ($\%$), runtime (s), and communication cost (MB) against zeroth-order baselines under varying levels of data heterogeneity on vision tasks. Our methods are shaded in gray, with the best accuracies in bold.}
\label{tab:table1}
\begin{tabular}{ll ccc ccc ccc}
\toprule
\multirow{2}{*}{Dataset} & \multirow{2}{*}{Algorithm} & \multicolumn{3}{c}{$\alpha=0.1$} & \multicolumn{3}{c}{$\alpha=0.5$} & \multicolumn{3}{c}{$\alpha=1.0$} \\
\cmidrule(lr){3-5} \cmidrule(lr){6-8} \cmidrule(lr){9-11}
& & Acc$\uparrow$ & Time$\downarrow$ & Comm$\downarrow$ & Acc$\uparrow$ & Time$\downarrow$ & Comm$\downarrow$ & Acc$\uparrow$ & Time$\downarrow$ & Comm$\downarrow$ \\
\midrule

 & FedKSeed Pro           & 78.18 & 7.05 & 0.17 & 79.42 & 7.28 & 0.17 & 79.44 & 7.10 & 0.17 \\
 & DeComFL                & 66.59 & 7.93 & 0.02 & 80.03 & 7.91 & 0.02 & 79.63 & 7.98 & 0.02 \\
\cmidrule{2-11} 
\rowcolor{gray!15} \cellcolor{white} & FedZOO-S        & \textbf{83.22} & 8.19 & 0.17 & \textbf{83.81} & 8.21 & 0.17 & \textbf{84.09} & 8.17 & 0.17 \\
\rowcolor{gray!15} \cellcolor{white} & FedZOO-NS ($\sigma=0.1$) & 79.14 & 8.25 & 0.17 & 82.08 & 8.23 & 0.17 & 80.73 & 8.24 & 0.17 \\
\rowcolor{gray!15} \cellcolor{white} \multirow{-5}{*}{MNIST} & FedZOO-NS ($\sigma=0.2$) & 75.04 & 8.28 & 0.17 & 79.00 & 8.26 & 0.17 & 79.02 & 8.24 & 0.17 \\
\midrule

 & FedKSeed Pro           & 51.46 & 7.82 & 0.17 & 57.75 & 7.96 & 0.17 & 60.92 & 7.93 & 0.17 \\
 & DeComFL                & 52.83 & 8.20 & 0.02 & 55.58 & 8.33 & 0.02 & 43.75 & 8.34 & 0.02 \\
\cmidrule{2-11} 
\rowcolor{gray!15} \cellcolor{white} & FedZOO-S        & 55.69 & 9.60 & 0.17 & 63.90 & 9.57 & 0.17 & \textbf{66.93} & 9.66 & 0.17 \\
\rowcolor{gray!15} \cellcolor{white} & FedZOO-NS ($\sigma=0.1$) & 58.43 & 9.95 & 0.17 & \textbf{66.88} & 9.91 & 0.17 & 64.76 & 9.93 & 0.17 \\
\rowcolor{gray!15} \cellcolor{white} \multirow{-5}{*}{FashionMNIST} & FedZOO-NS ($\sigma=0.2$) & \textbf{60.44} & 9.95 & 0.17 & 64.11 & 9.91 & 0.17 & 62.82 & 9.90 & 0.17 \\
\midrule

 & FedKSeed Pro           & 14.33 & 18.60 & 0.17 & \textbf{19.59} & 18.38 & 0.17 & 19.64 & 18.56 & 0.17 \\
 & DeComFL                & 15.30 & 19.03 & 0.02 & \textbf{19.59} & 19.08 & 0.02 & 19.59 & 19.12 & 0.02 \\
\cmidrule{2-11} 
\rowcolor{gray!15} \cellcolor{white} & FedZOO-S        & 15.95 & 20.62 & 0.17 & \textbf{19.59} & 21.31 & 0.17 & 20.59 & 21.77 & 0.17 \\
\rowcolor{gray!15} \cellcolor{white} & FedZOO-NS ($\sigma=0.1$) & \textbf{16.07} & 22.57 & 0.17 & 19.56 & 22.08 & 0.17 & \textbf{22.22} & 22.49 & 0.17 \\
\rowcolor{gray!15} \cellcolor{white} \multirow{-5}{*}{SVHN} & FedZOO-NS ($\sigma=0.2$) & 13.37 & 21.93 & 0.17 & 16.30 & 22.45 & 0.17 & 16.97 & 22.18 & 0.17 \\
\midrule

 & FedKSeed Pro           & 18.99 & 47.88 & 0.17 & 19.13 & 49.40 & 0.17 & 18.42 & 48.64 & 0.17 \\
 & DeComFL                & 16.43 & 49.49 & 0.02 & 18.81 & 50.35 & 0.02 & 20.94 & 49.92 & 0.02 \\
\cmidrule{2-11} 
\rowcolor{gray!15} \cellcolor{white} & FedZOO-S        & 19.47 & 55.44 & 0.17 & \textbf{21.13} & 56.73 & 0.17 & \textbf{21.08} & 55.45 & 0.17 \\
\rowcolor{gray!15} \cellcolor{white} & FedZOO-NS ($\sigma=0.1$) & \textbf{20.16} & 57.18 & 0.17 & 21.05 & 58.95 & 0.17 & 20.06 & 56.26 & 0.17 \\
\rowcolor{gray!15} \cellcolor{white} \multirow{-5}{*}{CIFAR-10} & FedZOO-NS ($\sigma=0.2$) & 19.69 & 56.73 & 0.17 & 19.82 & 58.33 & 0.17 & 20.26 & 57.17 & 0.17 \\
\bottomrule

\end{tabular}%
\end{table*}

\begin{table*}[t]
\centering
\caption{Comparison of ROUGE-L scores ($\%$), runtime (s), and communication cost (MB) with zeroth-order baselines under different split settings on TOFU. Our methods are shaded in gray, with the best ROUGE-L scores in bold.}
\label{tab:table2}
\resizebox{\linewidth}{!}{%
\begin{tabular}{ll ccc ccc ccc}
\toprule
\multirow{2}{*}{Model} & \multirow{2}{*}{Algorithm} & \multicolumn{3}{c}{Split99} & \multicolumn{3}{c}{Split95} & \multicolumn{3}{c}{Split90} \\
\cmidrule(lr){3-5} \cmidrule(lr){6-8} \cmidrule(lr){9-11}
 & & ROUGE-L$\uparrow$ & Time$\downarrow$ & Comm$\downarrow$ & ROUGE-L$\uparrow$ & Time$\downarrow$ & Comm$\downarrow$ & ROUGE-L$\uparrow$ & Time$\downarrow$ & Comm$\downarrow$ \\
\midrule

 & FedKSeed Pro            & 40.66 & 63.41 & 0.17 & 38.87 & 65.77 & 0.17 & 38.42 & 68.38 & 0.17 \\
 & DeComFL                 & 42.09 & 72.55 & 0.01 & 43.72 & 74.53 & 0.01 & 44.22 & 77.14 & 0.01 \\
\cmidrule{2-11} 
\rowcolor{gray!15} \cellcolor{white} & FedZOO-S                  & \textbf{44.07} & 74.07 & 0.17 & \textbf{44.47} & 76.47 & 0.17 & \textbf{45.01} & 78.61 & 0.17 \\
\rowcolor{gray!15} \cellcolor{white} & FedZOO-NS ($\sigma=10$) & 42.58 & 74.17 & 0.17 & 43.57 & 76.57 & 0.17 & 42.65 & 79.23 & 0.17 \\
\rowcolor{gray!15} \cellcolor{white} \multirow{-5}{*}{\makecell{Llama-3.2-\\1B-Instruct}} & FedZOO-NS ($\sigma=20$) & 38.61 & 74.08 & 0.17 & 41.03 & 76.59 & 0.17 & 40.68 & 79.47 & 0.17 \\
\midrule

 & FedKSeed Pro            & 41.83 & 178.74 & 0.17 & 42.04 & 185.11 & 0.17 & 45.38 & 191.99 & 0.17 \\
 & DeComFL                 & 43.48 & 194.49 & 0.01 & 46.38 & 196.60 & 0.01 & 46.63 & 204.29 & 0.01 \\
\cmidrule{2-11} 
\rowcolor{gray!15} \cellcolor{white} & FedZOO-S                  & \textbf{46.84} & 201.88 & 0.17 & \textbf{48.79} & 205.42 & 0.17 & \textbf{48.78} & 212.13 & 0.17 \\
\rowcolor{gray!15} \cellcolor{white} & FedZOO-NS ($\sigma=10$) & 46.16 & 207.60 & 0.17 & 48.70 & 212.93 & 0.17 & 46.21 & 220.67 & 0.17 \\
\rowcolor{gray!15} \cellcolor{white} \multirow{-5}{*}{\makecell{Llama-3.2-\\3B-Instruct}} & FedZOO-NS ($\sigma=20$) & 43.02 & 207.21 & 0.17 & 46.56 & 213.70 & 0.17 & 45.90 & 220.34 & 0.17 \\
\bottomrule

\end{tabular}%
}
\end{table*}

\subsection{Experimental Setup}
\subsubsection{Datasets and Models}
We evaluate our methods on MNIST~\cite{lecun1998gradient}, FashionMNIST~\cite{xiao2017fashion}, SVHN~\cite{netzer2011reading}, CIFAR-10~\cite{krizhevsky2009learning}, and the TOFU benchmark for LLM unlearning~\cite{maini2024tofu}. Vision data are partitioned across clients using a Dirichlet distribution with $\alpha\in\{0.1,0.5,1.0\}$. To construct low-quality forget sets, we randomly select $20\%$ of clients and flip the labels of $20\%$ of their local samples. For TOFU, we use the standard Split99, Split95, and Split90, with retain/forget ratios of $99\%/1\%$, $95\%/5\%$, and $90\%/10\%$. Retain data are uniformly distributed across all clients, whereas forget data are uniformly assigned to $20\%$ of clients. We use a two-hidden-layer MLP with 512 and 256 neurons and dropout for MNIST, a two-layer CNN for FashionMNIST and SVHN, and ResNet-18 for CIFAR-10. TOFU experiments use Llama-3.2-1B-Instruct and Llama-3.2-3B-Instruct.

\subsubsection{Baselines}
We compare our methods with state-of-the-art baselines in federated optimization and unlearning. For optimization, we consider two zeroth-order baselines: FedKSeed Pro~\cite{qin2024federated} uses a fixed seed pool and scalar-gradient accumulators; DeComFL~\cite{li2025achieving} reconstructs client models from historical seeds and aggregated gradient scalars, making per-round uplink and downlink costs independent of model dimension. In vision unlearning: Re-optimization retrains from scratch; FATS-SU~\cite{tao2024communication} replays training from the first step involving a target sample; FedRR~\cite{liu2022right} uses Fisher information for fast retraining; and FUDD~\cite{mora2026federated} generates private synthetic data through knowledge distillation for server-side unlearning. For TOFU, we consider two baselines: re-optimization and Oblivionis~\cite{zhang2026oblivionis}, which is a LoRA-based method that alternates between retain-set fine-tuning and forget-set unlearning.

\subsubsection{Hyperparameters}
The following settings are used unless otherwise specified: total number of clients $M=100$, number of participating clients per round $m=10$, accumulator size $K=4096$, number of local steps $E=10$, mini-batch size $b=32$, and number of perturbation directions per sample $q=1$. For vision tasks, we run $R=300$ communication rounds with perturbation scale $\mu=10^{-3}$ and clipping threshold $C=50.0$, and tune the learning rate $\eta$ over $\{10^{-5},3\times10^{-5},10^{-4},3\times10^{-4},10^{-3},3\times10^{-3}\}$. For TOFU, the corresponding settings are $R=30$, $\mu=5\times10^{-5}$, and $C=100.0$; we tune $\eta$ over $\{5\times10^{-7},10^{-6},3\times10^{-6},5\times10^{-6},10^{-5},3\times10^{-5}\}$.


\subsubsection{Evaluation Metrics}
(1) \textit{test accuracy} evaluates model utility on vision tasks and tracks its changes under sequential unlearning requests;
(2) \textit{ROUGE-L} measures generation quality on TOFU, with retain- and forget-set scores reported before and after unlearning, and Real Authors and World Facts scores used to assess general knowledge retention;
(3) \textit{time cost} denotes wall-clock execution time;
(4) \textit{communication cost} accounts for the total bytes exchanged between the server and clients; and
(5) \textit{peak GPU memory} reports the maximum GPU memory usage during unlearning.
Time and communication costs are reported per round for optimization and cumulatively over sequential requests for unlearning.

\begin{figure}[t]
    \centering
    \includegraphics[width=0.49\linewidth]{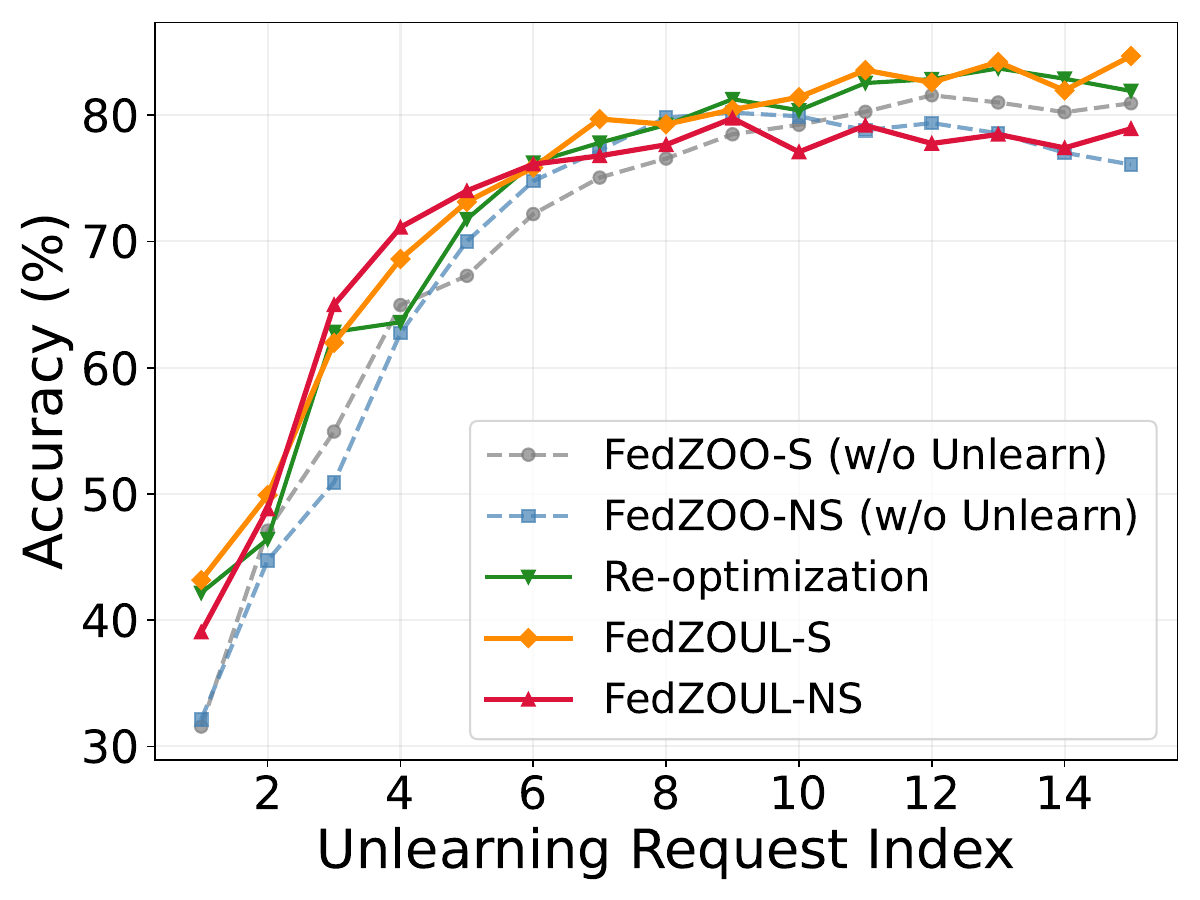}
    \hfill    
    \includegraphics[width=0.49\linewidth]{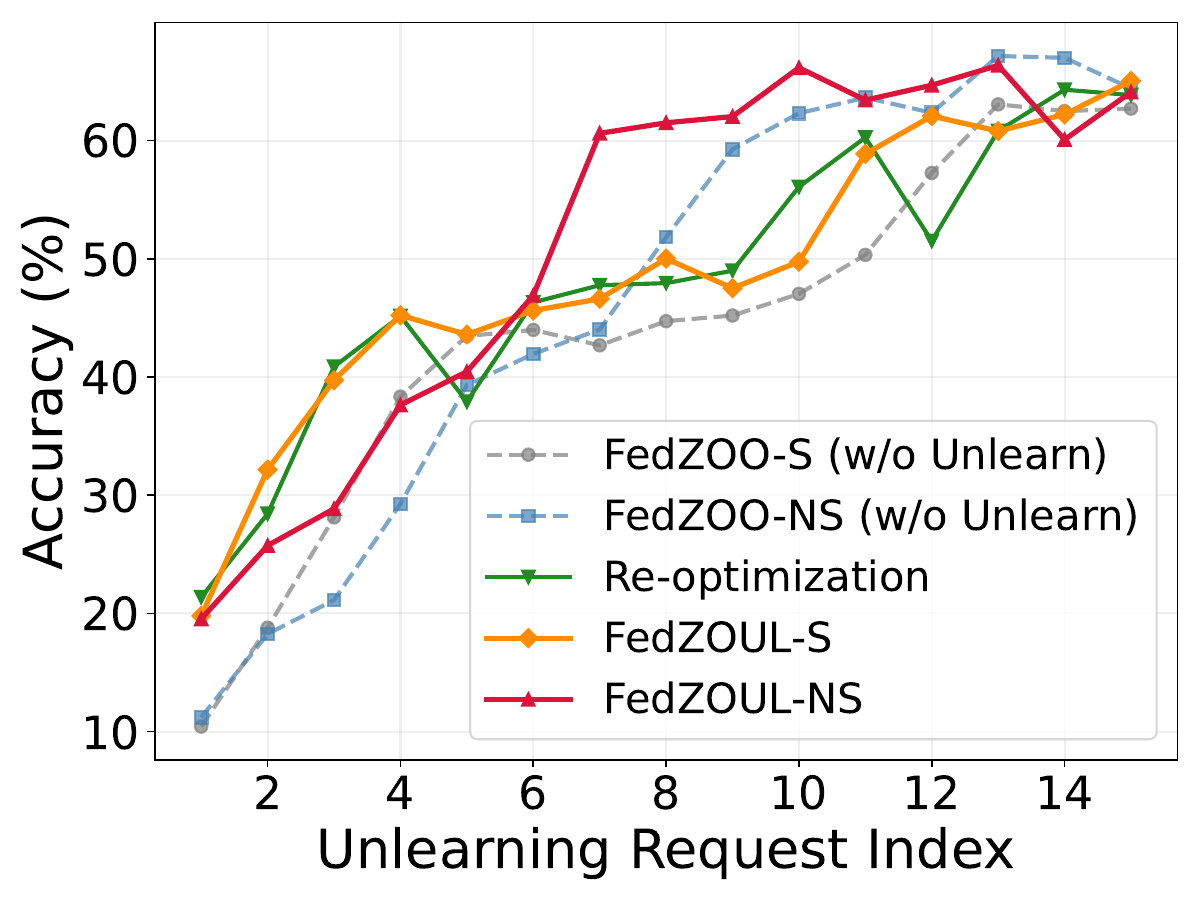}

    \includegraphics[width=0.49\linewidth]{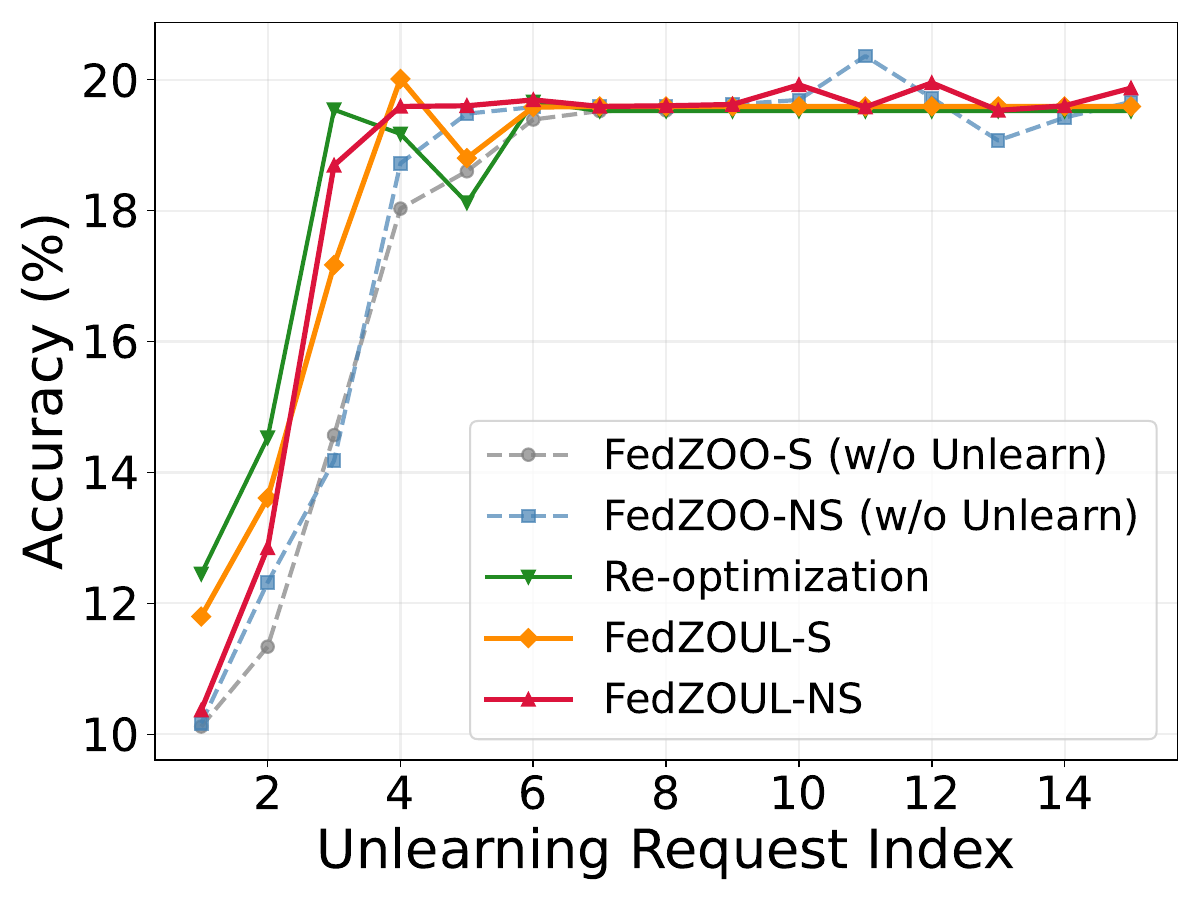}
    \hfill
    \includegraphics[width=0.49\linewidth]{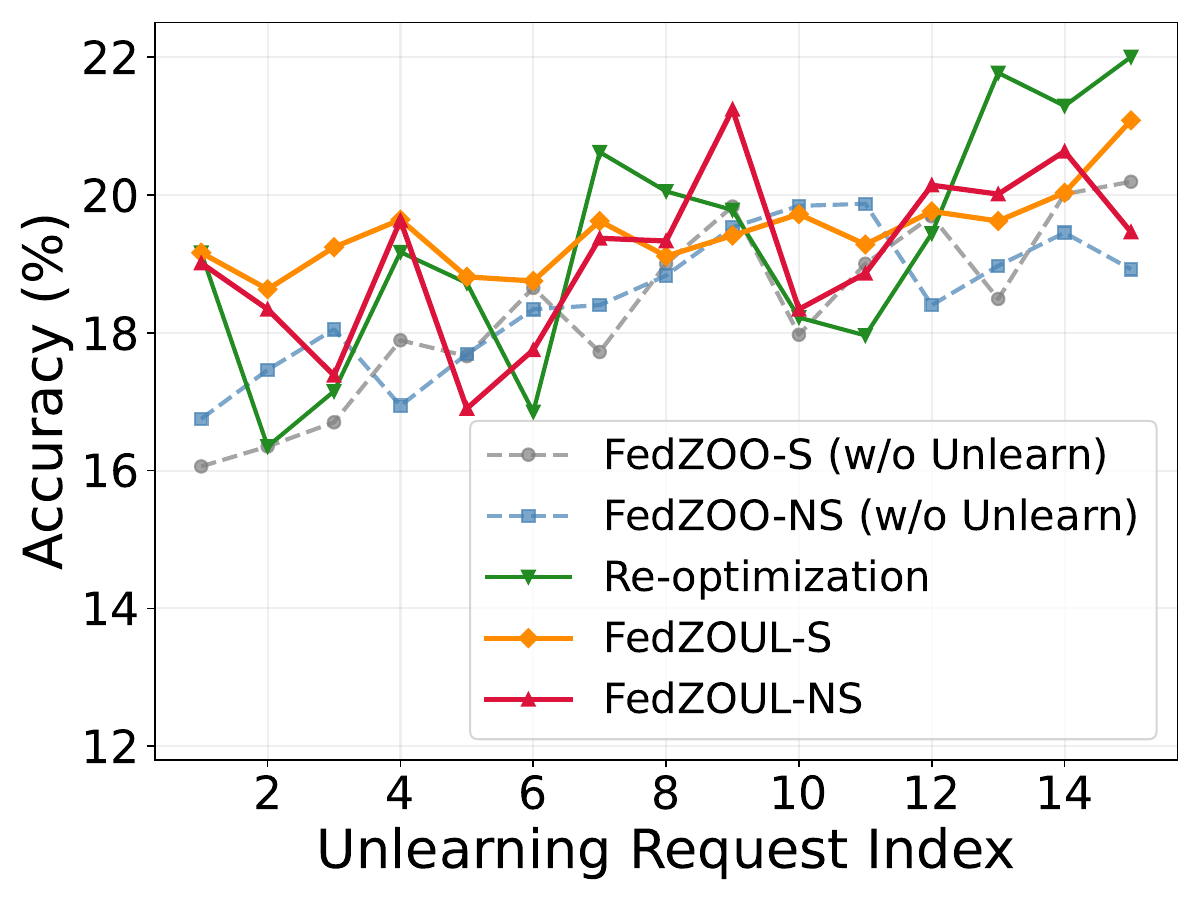}

    \caption{Accuracy comparison of FedZOUL-S/NS against FedZOO-S/NS (before unlearning) and re-optimization under stream unlearning requests on four vision datasets. Top-left:MNIST, top-right: FashionMNIST, bottom-left: SVHN, bottom-right: CIFAR-10.}
    \label{fig:vision-acc-fedzoul-s/ns}
\end{figure}

\begin{figure}[t]
    \centering
    \includegraphics[width=1\linewidth]{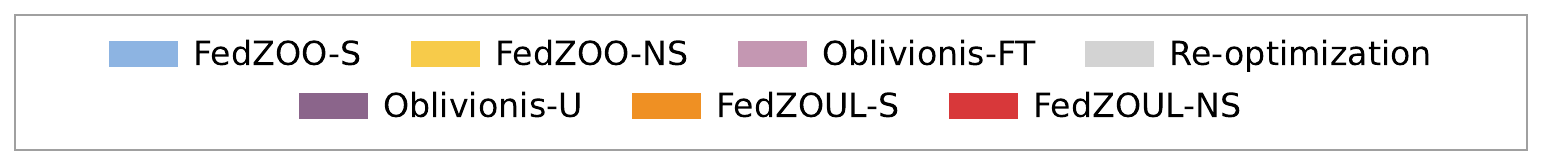} \\
    \vspace{0.2cm} 

    \includegraphics[width=0.49\linewidth]{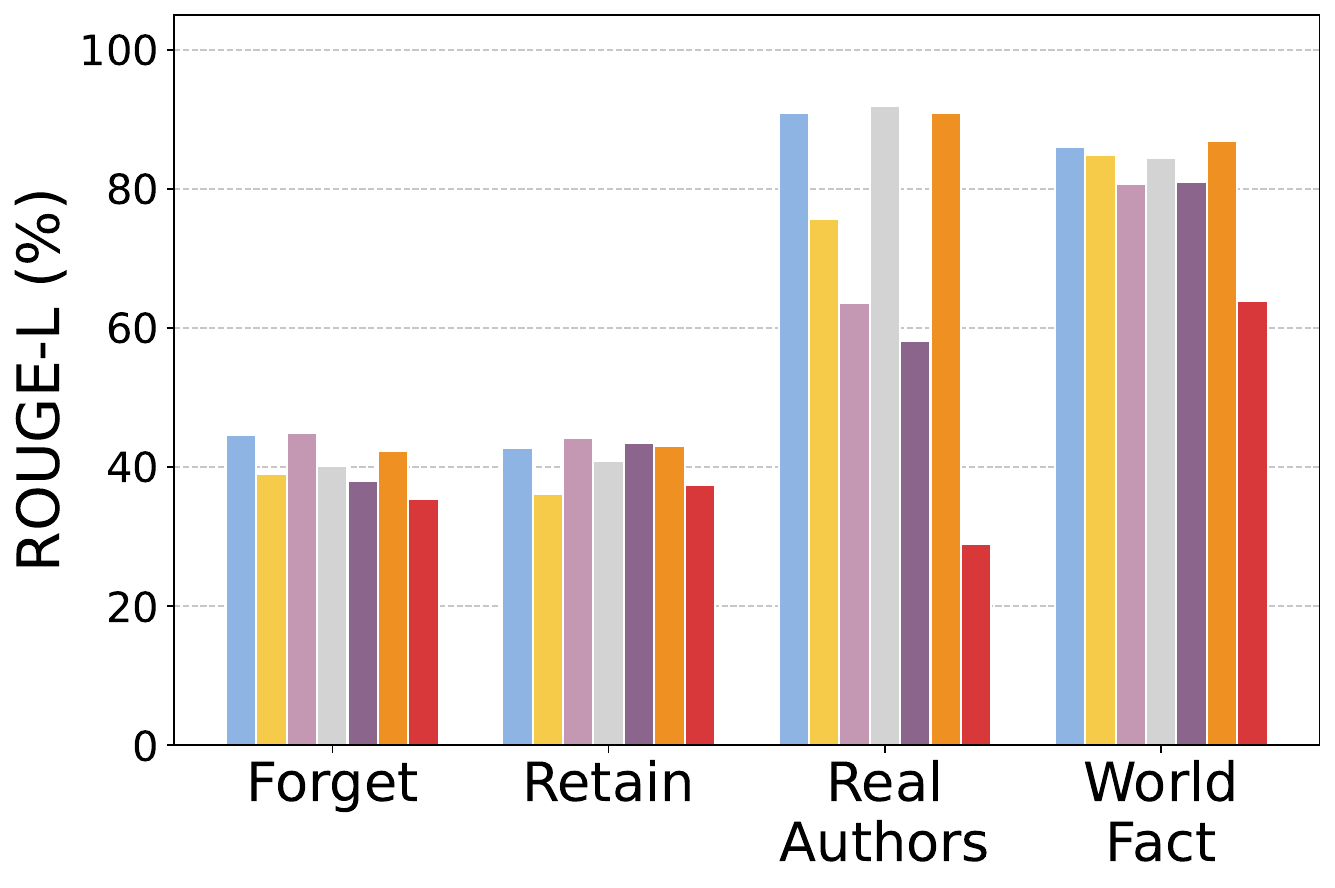}
    \hfill  
    \includegraphics[width=0.49\linewidth]{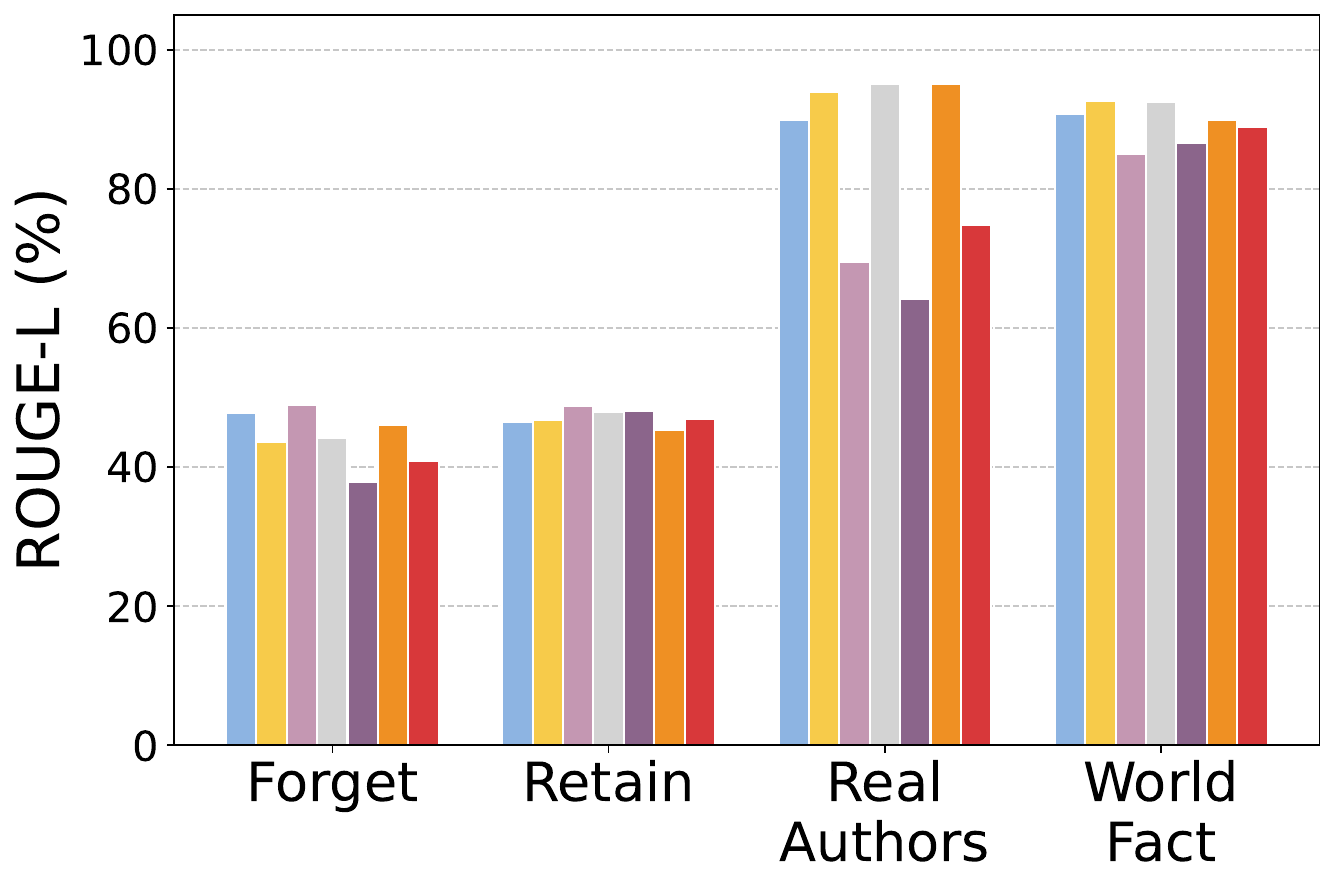}

    \caption{ROUGE-L scores of FedZOUL-S and FedZOUL-NS compared with re-optimization and Oblivionis before and after stream unlearning requests on TOFU (Split95). Oblivionis-FT and Oblivionis-U denote fine-tuning and unlearning under the Oblivionis framework, respectively. Left: Llama-3.2-1B-Instruct. Right: Llama-3.2-3B-Instruct.}
    \label{fig:tofu-rouge-l-fedzoul-s/ns-baselines}
\end{figure}

\subsection{Experimental Results}

\subsubsection{Performance of FedZOO-S and FedZOO-NS}
We evaluate \textsc{FedZOO-S} and \textsc{FedZOO-NS} on four vision benchmarks, as well as TOFU. For \textsc{FedZOO-NS}, we consider multiple noise scales. We vary the degree of data heterogeneity on the vision benchmarks using $\alpha\in\{0.1,0.5,1.0\}$ and evaluate TOFU under retain/forget splits of $99\%/1\%$, $95\%/5\%$, and $90\%/10\%$, using only the retain set for optimization. We compare our algorithms with two state-of-the-art baselines.

As shown in Table~\ref{tab:table1}, our methods match or outperform the baselines across all 12 vision settings, with accuracy gains of up to $9.13$ percentage points. \textsc{FedZOO-S} achieves the best or tied-best accuracy in seven settings. These gains may partly arise from using $bqE$ sample--direction pairs per participating client in each round, compared with $E$ direction evaluations in FedKSeed Pro and DeComFL. For \textsc{FedZOO-NS}, the broader trend is that increasing the noise scale reduces accuracy. Nevertheless, it outperforms \textsc{FedZOO-S} in several settings, showing that moderate noise does not necessarily reduce model utility. Nevertheless, \textsc{FedZOO-NS} outperforms \textsc{FedZOO-S} in several settings, showing that moderate noise does not necessarily reduce model utility. With respect to data heterogeneity, \textsc{FedZOO-S} generally improves as $\alpha$ increases and data heterogeneity decreases, whereas \textsc{FedZOO-NS} exhibits a non-monotonic trend, possibly due to noise-induced variation in its optimization trajectory. Table~\ref{tab:table2} reports the TOFU results. \textsc{FedZOO-S} achieves the highest ROUGE-L score for both model sizes across all splits. \textsc{FedZOO-NS} performs worse than \textsc{FedZOO-S}, with $\sigma=20$ consistently underperforming $\sigma=10$. This contrasts with the mixed effect of noise on the vision tasks and suggests that LLM fine-tuning may be more sensitive to scalar noise. The 3B model consistently outperforms the 1B model across all methods and splits, which may reflect its greater representational capacity.

Taken together, Tables~\ref{tab:table1} and~\ref{tab:table2} show that our methods trade additional local computation for improved utility while retaining dimension-independent communication. Their higher runtime stems from evaluating $bqE$ sample--direction pairs per participating client in each round, compared with $E$ direction evaluations in FedKSeed Pro and DeComFL. \textsc{FedZOO-S} and \textsc{FedZOO-NS} match FedKSeed Pro in per-round communication by exchanging fixed-size scalar-gradient accumulators. Under our settings, this cost exceeds that of DeComFL, whose clients upload only per-update scalars and reconstruct models from server-provided seeds and historical aggregates. Overall, both methods incur modest additional overhead, while their improved utility provides a strong foundation for the unlearning procedures.

\begin{figure}[t]
    \centering

    \includegraphics[width=0.49\linewidth]{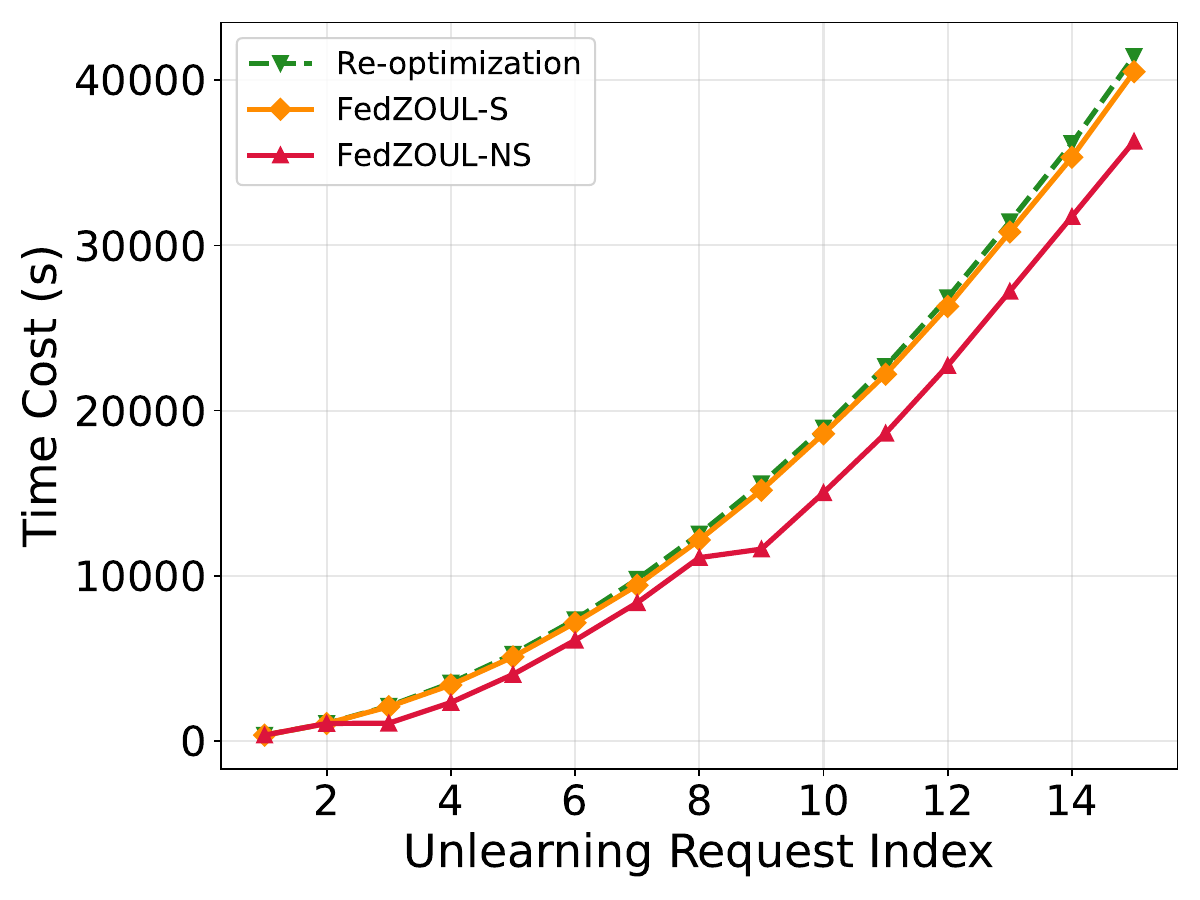}
    \hfill  
    \includegraphics[width=0.49\linewidth]{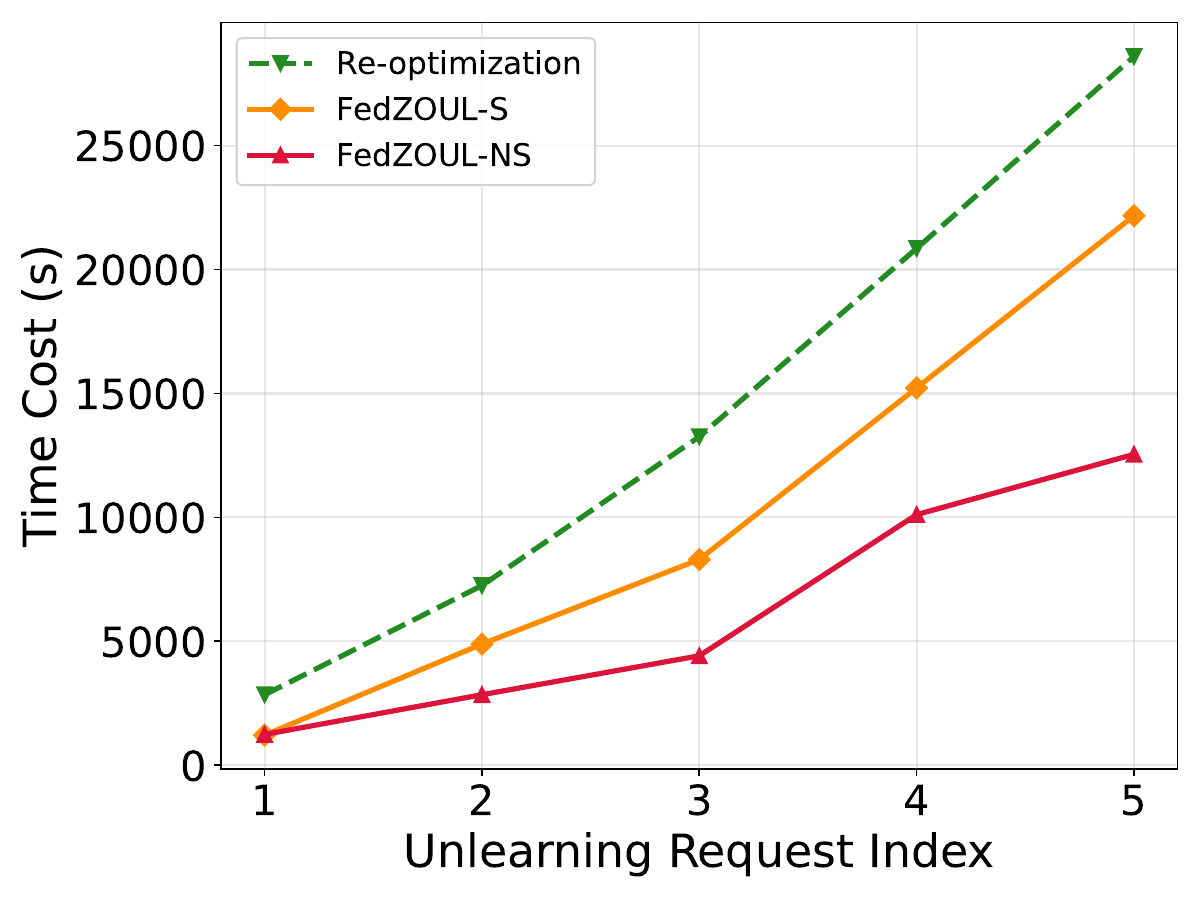}

    \caption{Cumulative time cost of FedZOUL-S/NS against re-optimization under stream unlearning requests on four vision datasets. Left: CIFAR-10, right: TOFU (Llama-3.2-3B-Instruct).}
    \label{fig:cumulative-time-cost-fedzoul-s/ns-re-optimization}
\end{figure}

\subsubsection{Performance of FedZOUL-S and FedZOUL-NS}

To evaluate the unlearning algorithms, we introduce a stream of unlearning requests at fixed intervals, with each request made by one or more target clients to remove samples from their local datasets.
Figures~\ref{fig:vision-acc-fedzoul-s/ns} and~\ref{fig:tofu-rouge-l-fedzoul-s/ns-baselines} summarize the performance on the vision tasks and TOFU under Split95, respectively. On the vision tasks, \textsc{FedZOUL-S} improves over \textsc{FedZOO-S} by up to $13.37$ percentage points after removing low-quality target data, while its final accuracy remains within $2.77$ percentage points of re-optimization across all four tasks. On TOFU, both variants reduce forget-set ROUGE-L after unlearning; although \textsc{FedZOUL-S} has higher forget-set scores than Oblivionis-U, it better preserves general knowledge, achieving $90.93\%$ and $95.13\%$ on Real Authors for the 1B and 3B models, compared with $58.08\%$ and $64.13\%$ for Oblivionis-U. \textsc{FedZOUL-NS} achieves competitive final accuracy on the vision tasks and forget-set performance similar to Oblivionis-U and re-optimization on TOFU. However, with the additional randomness introduced by noise injection, \textsc{FedZOUL-NS} shows larger deviations at some intermediate vision requests and greater degradation in general knowledge, particularly for the 1B model, consistent with the behavior observed for \textsc{FedZOO-NS} in the federated optimization experiments.

\begin{figure}[t]
    \centering
    \includegraphics[width=0.8\linewidth]{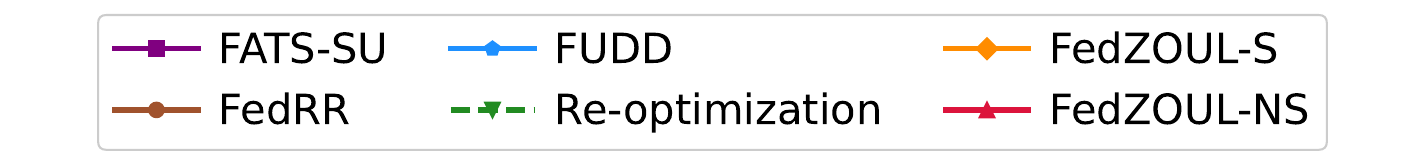} \\
    \vspace{0.2cm} 

    \includegraphics[width=0.49\linewidth]{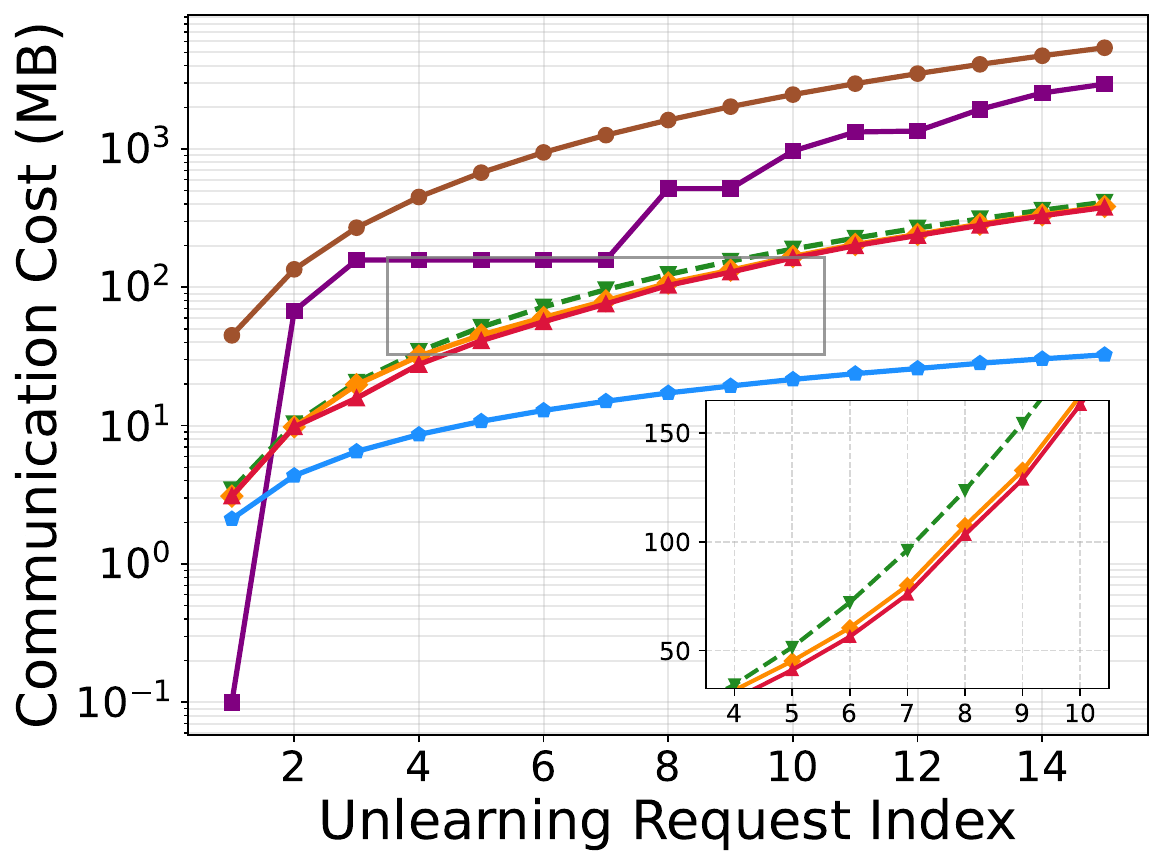}
    \hfill  
    \includegraphics[width=0.49\linewidth]{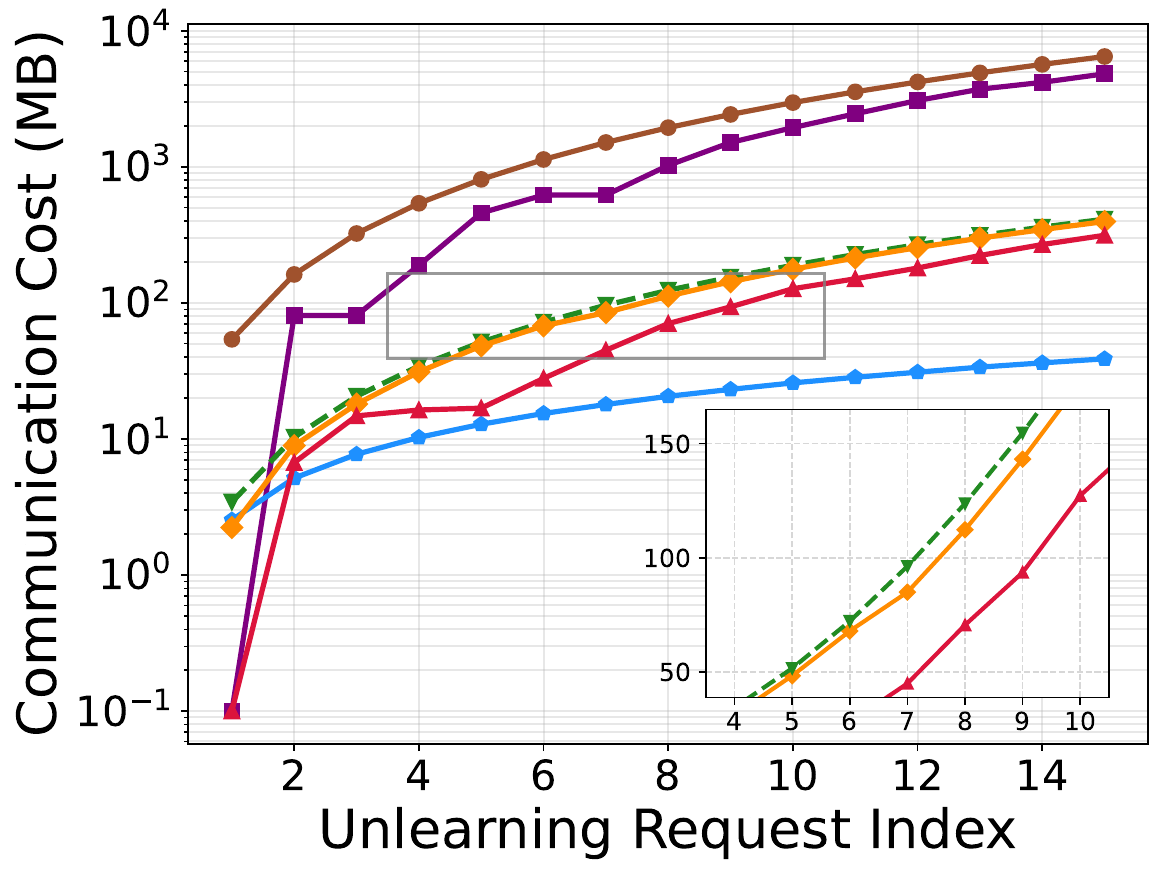}

    \vspace{0.1cm}

    \includegraphics[width=0.49\linewidth]{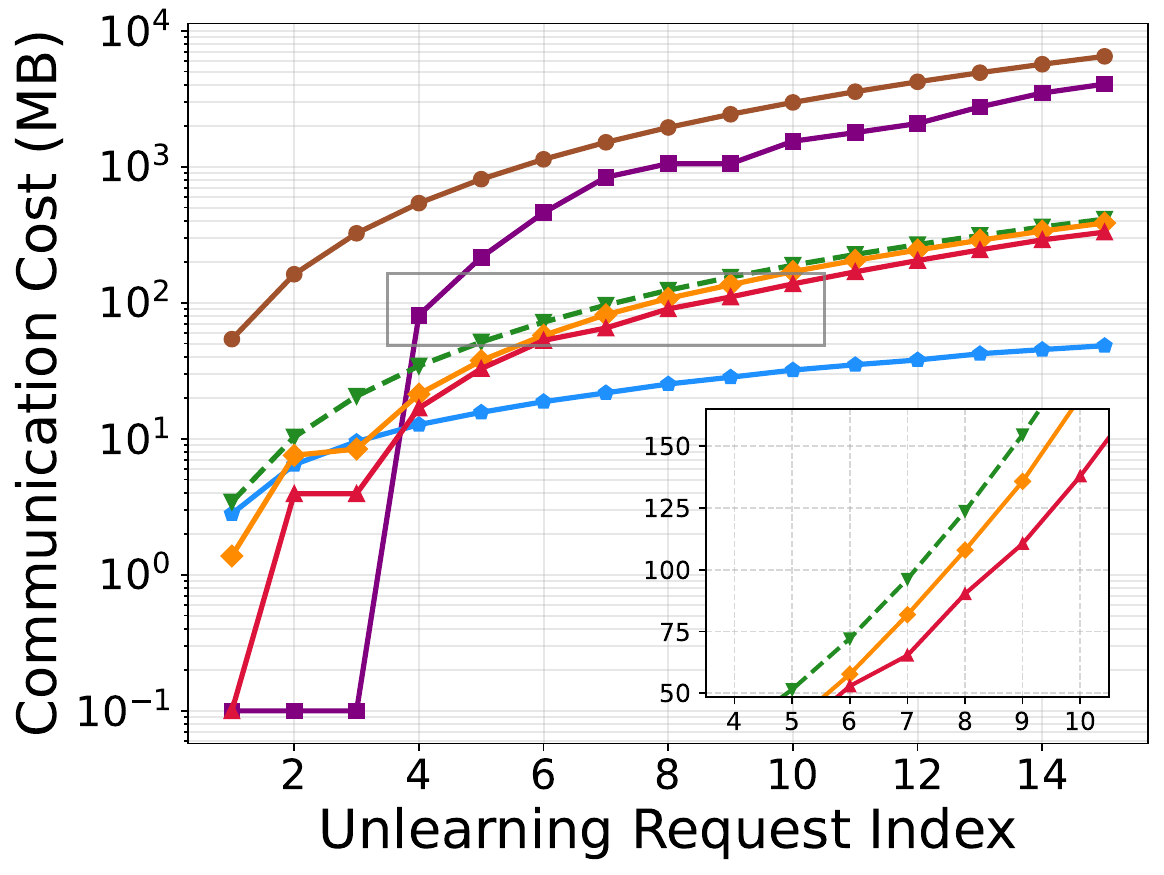}
    \hfill
    \includegraphics[width=0.49\linewidth]{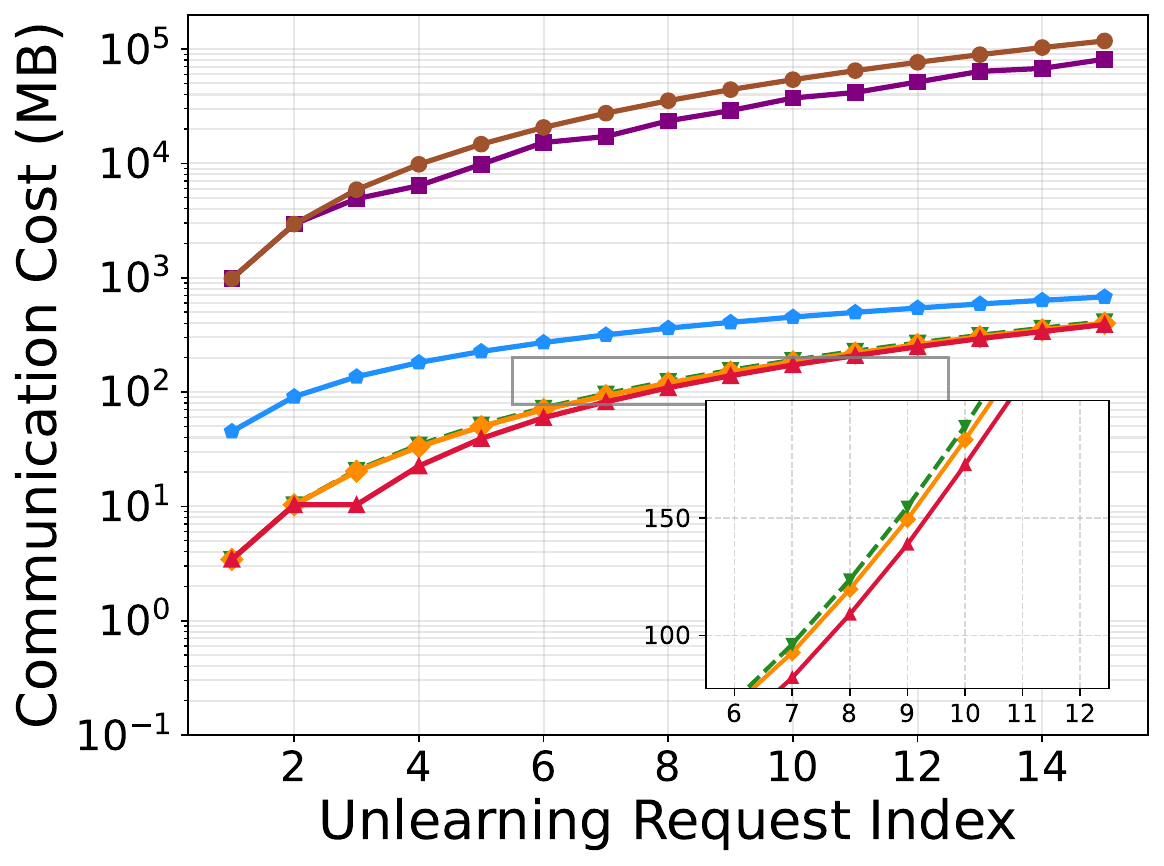}

    \caption{Cumulative communication cost of FedZOUL-S/NS against re-optimization and first-order unlearning baselines under stream unlearning requests on four vision datasets. Top-left: MNIST, top-right: FashionMNIST, bottom-left: SVHN, bottom-right: CIFAR-10.}
    \label{fig:vision-cumulative-comm-cost-fedzoul-s/ns-baselines}
\end{figure}

\begin{figure}[t]
    \centering
    
    \includegraphics[width=0.49\linewidth]{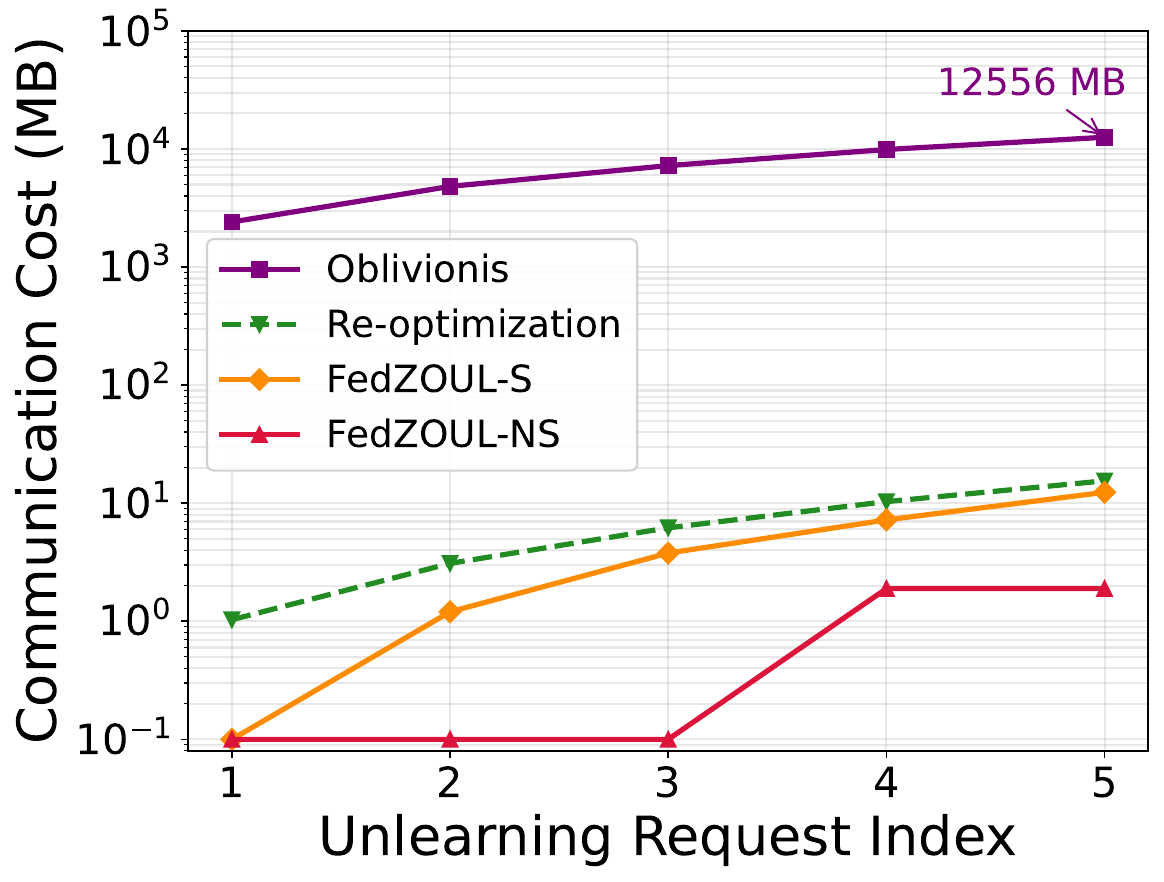}
    \hfill  
    \includegraphics[width=0.49\linewidth]{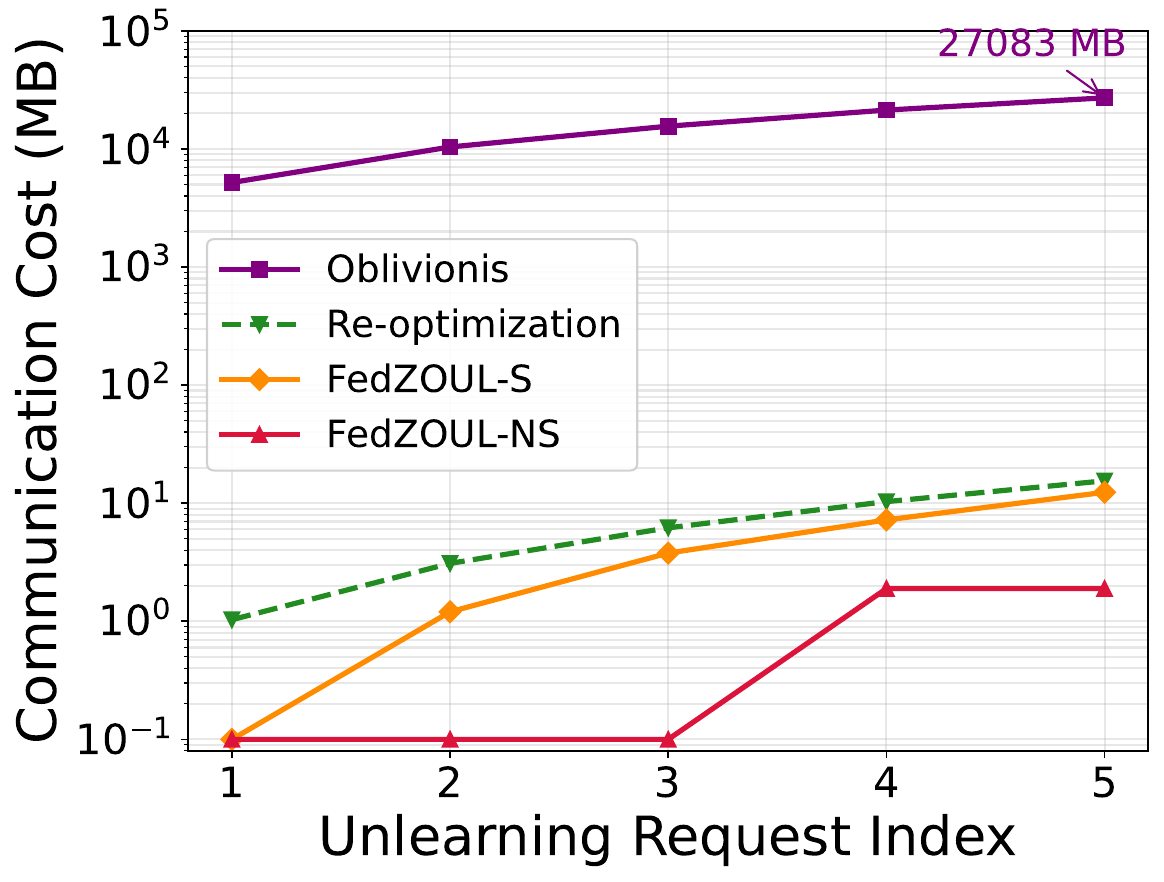}

    \caption{Cumulative communication cost of FedZOUL-S/NS against re-optimization and Oblivionis under stream unlearning requests on TOFU (Split95). Left: Llama-3.2-1B-Instruct, right: Llama-3.2-3B-Instruct.}
    \label{fig:tofu-cumulative-comm-cost-fedzoul-s/ns-baselines}
\end{figure}

To evaluate efficiency, we compare the cumulative runtime of \textsc{FedZOUL-S} and \textsc{FedZOUL-NS} with re-optimization in Figure~\ref{fig:cumulative-time-cost-fedzoul-s/ns-re-optimization}. At the final request, \textsc{FedZOUL-S} and \textsc{FedZOUL-NS} reduce runtime by $2.2\%$ and $12.4\%$ on CIFAR-10 and by $22.5\%$ and $56.1\%$ on TOFU with the 3B model, respectively. These savings arise because only a fraction of requests trigger suffix replay, while noise injection and rejection sampling further reduce replay in \textsc{FedZOUL-NS}.
Figures~\ref{fig:vision-cumulative-comm-cost-fedzoul-s/ns-baselines} and~\ref{fig:tofu-cumulative-comm-cost-fedzoul-s/ns-baselines} show similar gains in communication. Unlike FATS-SU and FedRR, which transmit parameter-sized information, our methods exchange fixed-size scalar-gradient accumulators and reduce suffix replay. After 15 requests on CIFAR-10, \textsc{FedZOUL-NS} incurs 389.1 MB, which is $99.5\%$, $99.7\%$, $42.7\%$, and $5.7\%$ lower than FATS-SU, FedRR, FUDD, and re-optimization, respectively. After five requests on TOFU with the 3B model, \textsc{FedZOUL-S} and \textsc{FedZOUL-NS} incur only 12.38 MB and 1.89 MB, compared with 27.08 GB for Oblivionis. Thus, \textsc{FedZOUL-NS} reduces communication by approximately four orders of magnitude, highlighting its benefit for bandwidth-constrained federated unlearning.

In terms of resource efficiency, we further evaluate peak GPU memory on vision tasks and TOFU, with detailed results reported in the Appendix~\ref{app:memory}. Both \textsc{FedZOUL-S} and \textsc{FedZOUL-NS} require less memory than the first-order unlearning baselines. On TOFU, their memory advantage over Oblivionis becomes more pronounced as model size increases. These results complement the time and communication savings and support deployment on devices with limited memory.

\subsubsection{Impact of the Client Sampling Ratio $m/M$}

\begin{figure}[t]
    \centering
    \includegraphics[width=0.49\linewidth]{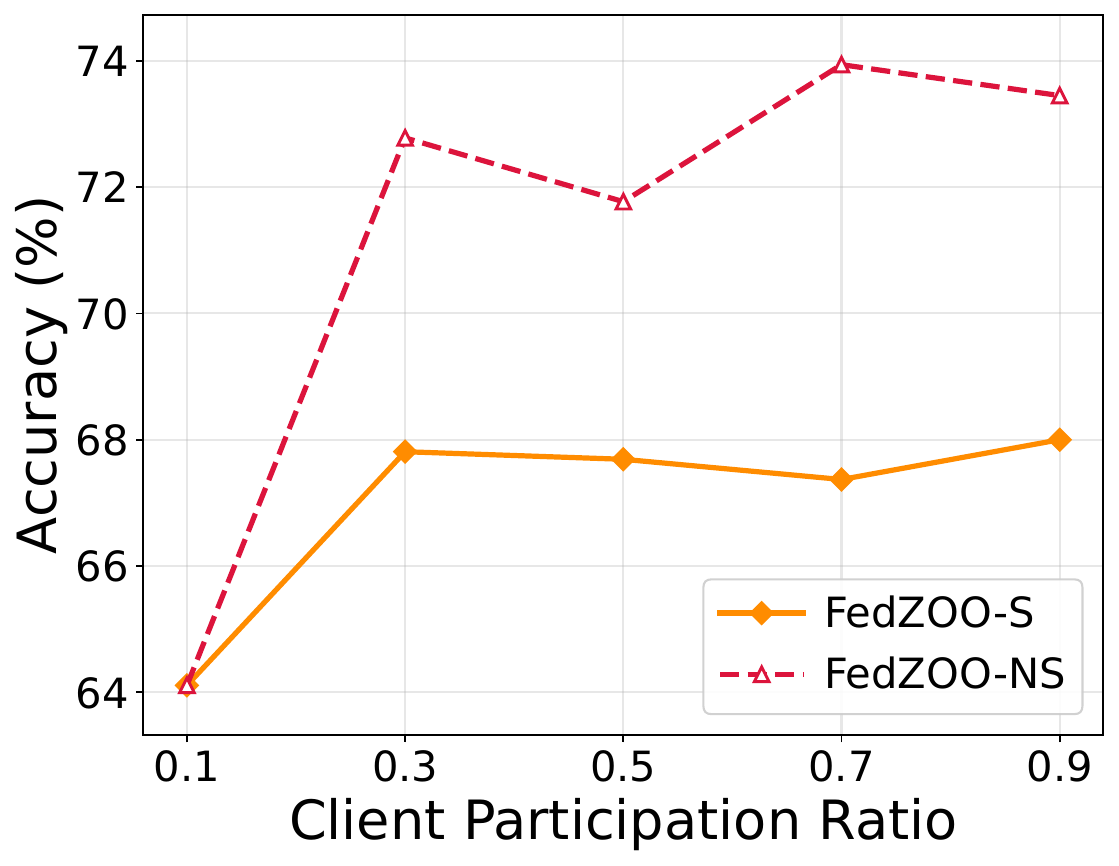}
    \hfill
    \includegraphics[width=0.49\linewidth]{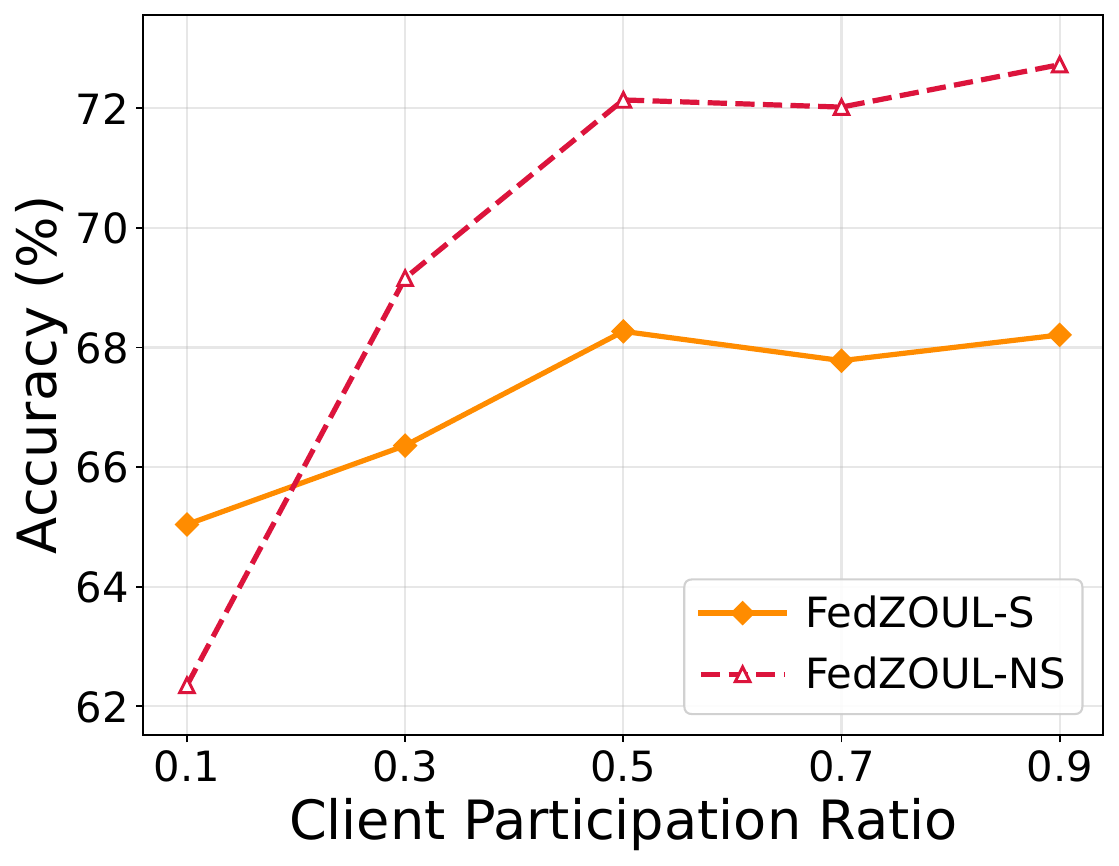}

    \vspace{0.1cm}

    \includegraphics[width=0.8\linewidth]{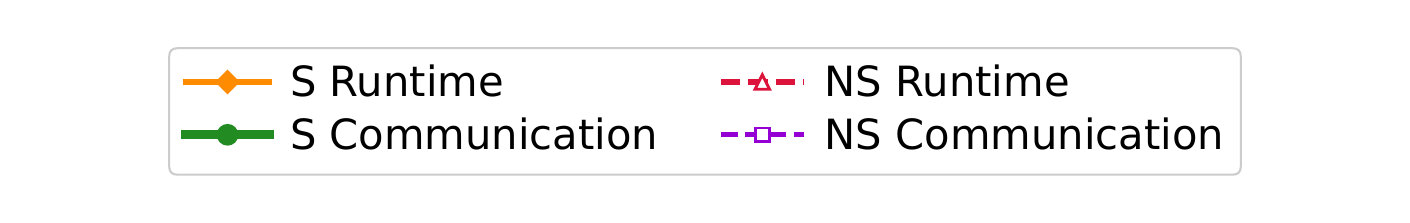} \\\vspace{-0.2cm}
    \includegraphics[width=0.493\linewidth]{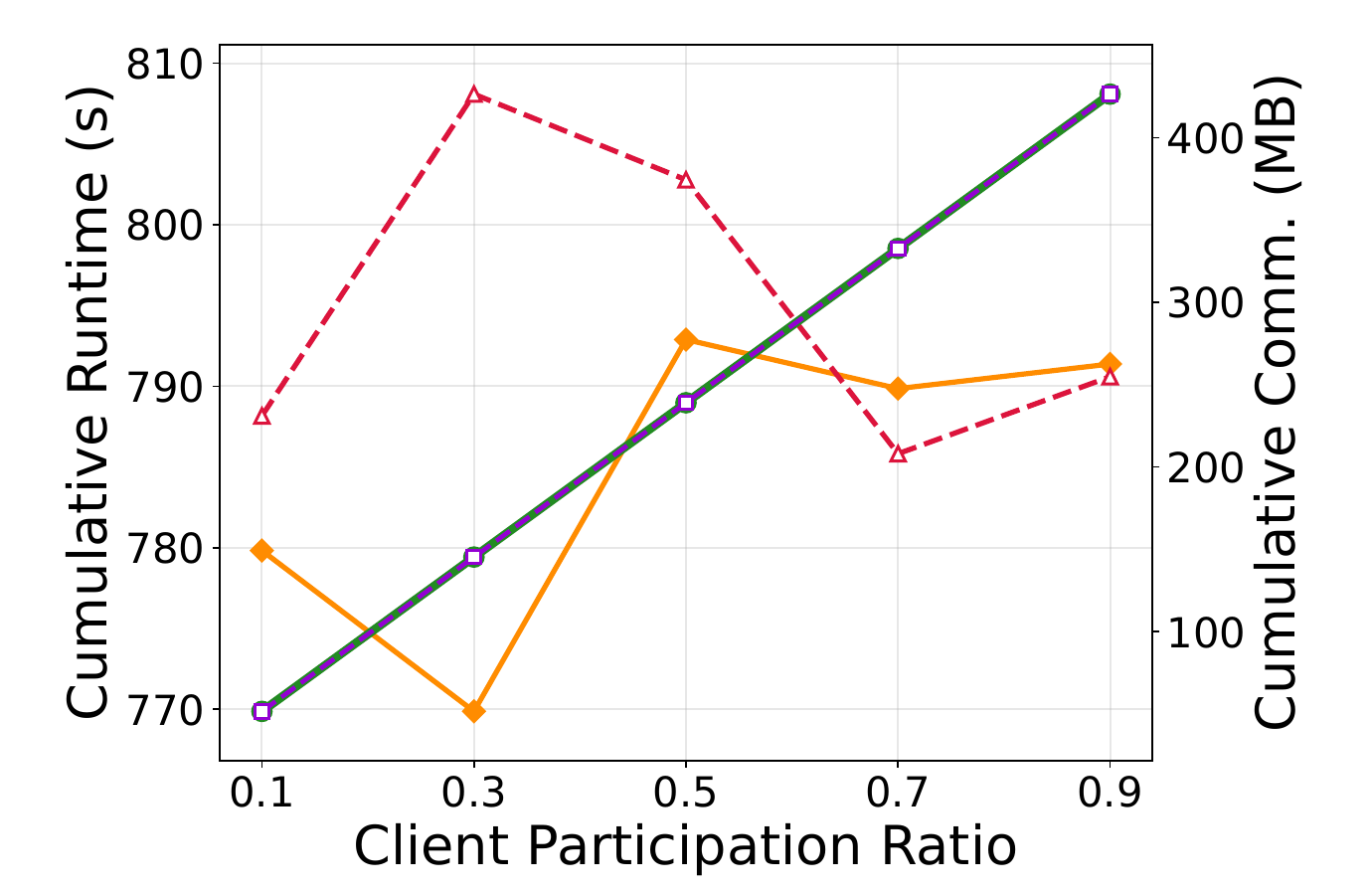}
    \hfill
    \includegraphics[width=0.493\linewidth]{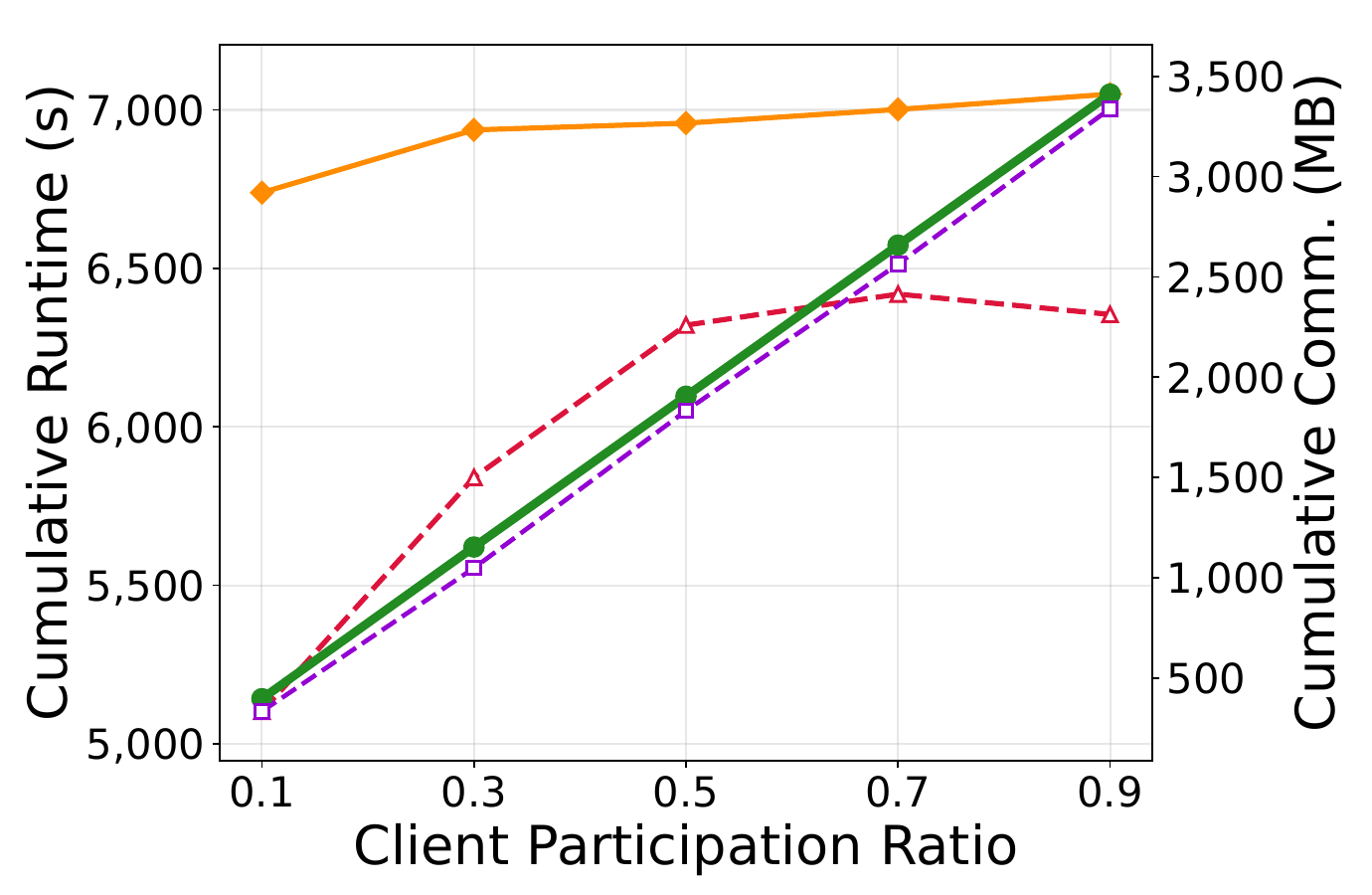}

    \caption{Impact of the client sampling ratio ($m/M$) on accuracy, cumulative runtime, and cumulative communication cost on FashionMNIST. Top row: accuracy; bottom row: cumulative runtime and communication cost. Left column: FedZOO-S/NS; right column: FedZOUL-S/NS.}
    \label{fig:fashionmnist-fedzoo-fedzoul-impact-client-ratio}
\end{figure}

We study how the client sampling ratio $m/M$ affects the utility and efficiency of \textsc{FedZOO-S/NS} and \textsc{FedZOUL-S}/NS on FashionMNIST. This ratio determines client participation and affects the TV-stability bounds of both learning algorithms. Figure~\ref{fig:fashionmnist-fedzoo-fedzoul-impact-client-ratio} shows that increasing $m/M$ from $0.1$ to $0.3$ improves optimization accuracy, while further increases provide limited gains. For unlearning, \textsc{FedZOUL-S} achieves its highest accuracy of $68.27\%$ at $m/M=0.5$, whereas \textsc{FedZOUL-NS} reaches $72.73\%$ at $0.9$. These results show that greater participation generally improves model utility, although accuracy does not increase monotonically. Except at $m/M=0.1$, \textsc{FedZOO-NS} and \textsc{FedZOUL-NS} achieve higher accuracy than \textsc{FedZOO-S} and \textsc{FedZOUL-S}, respectively, indicating that introducing randomness does not reduce model utility.

However, greater participation increases communication overhead, as each selected client uploads a fixed-size accumulator per round. The cumulative communication cost of \textsc{FedZOO-S/NS} increases by a factor of $8.27$ from $m/M=0.1$ to $0.9$, while runtime changes little under parallel client execution. For unlearning, communication grows approximately linearly and runtime generally increases. This is because a higher client sampling ratio increases both the communication cost per round and the probability that local updates use target samples, which may require replay during unlearning. \textsc{FedZOUL-NS} consistently uses less time and communication than \textsc{FedZOUL-S}: Gaussian coupling allows affected updates to be reused until the first rejection. However, its runtime and communication savings decrease from $24.3\%$ and $15.9\%$ at $m/M=0.1$ to $9.9\%$ and $2.1\%$ at $0.9$, respectively, indicating a smaller efficiency advantage at higher participation levels, where more frequent use of target samples can increase replay probabilities for both methods.

\subsubsection{Impact of Hyperparameters $K$ and $q$}
We examine the effects of $K$ and $q$ on FashionMNIST, with detailed results reported in the Appendix~\ref{app:Kq}. Smaller seed pools generally favor \textsc{FedZOO-S}, whereas larger pools tend to benefit \textsc{FedZOO-NS}. Increasing $q$ can improve accuracy for both methods, although the gains are not monotonic. Communication cost grows linearly with $K$ and is independent of $q$, while runtime increases with both parameters.

\section{Conclusion}
We present a framework for efficient federated zeroth-order optimization and exact unlearning based on total variation stability. \textsc{FedZOO-S} obtains stability through subsampling, and \textsc{FedZOUL-S} reuses unaffected training computation to achieve efficient exact unlearning. Building on this design, \textsc{FedZOO-NS} incorporates Gaussian randomization, allowing \textsc{FedZOUL-NS} to reuse some affected updates through coupling and further reduce replay for more efficient exact unlearning. We establish theoretical guarantees and demonstrate competitive utility with lower runtime, communication, and memory costs.

\bibliographystyle{IEEEtran}
\bibliography{references}

@inproceedings{qin2024federated,
    author = {Qin, Zhen and Chen, Daoyuan and Qian, Bingchen and Ding, Bolin and Li, Yaliang and Deng, Shuiguang},
    title = {{Federated Full-Parameter Tuning of Billion-Sized Language Models with Communication Cost under 18 Kilobytes}},
    publisher = {PMLR},
    address = {Vienna, Austria},
    booktitle = {41st International Conference on Machine Learning (ICML 2024)},
    month = {7},
    year = {2024}
}

@inproceedings{wang2025scalable,
    author = {Wang, Liangyu and Ren, Jie and Xu, Hang and Wang, Junxiao and Xie, Huanyi and Keyes, David E and Wang, Di},
    title = {{Scalable Zeroth-Order Fine-Tuning for Extremely Large Language Models with Limited GPU Memory}},
    publisher = {OpenReview.net},
    address = {Montr{\'{e}}al, Canada},
    booktitle = {Second Conference on Language Modeling (COLM 2025)},
    month = {10},
    year = {2025}
}

@article{tao2024communication,
    author = {Tao, Youming and Wang, Cheng-Long and Pan, Miao and Yu, Dongxiao and Cheng, Xiuzhen and Wang, Di},
    doi = {10.14778/3641204.3641220},
    title = {{Communication Efficient and Provable Federated Unlearning}},
    pages = {1119--1131},
    journal = {Proceedings of the VLDB Endowment},
    issn = {2150-8097},
    publisher = {VLDB Endowment},
    month = {1},
    number = {5},
    volume = {17},
    year = {2024}
}

@article{lecun1998gradient,
    author = {Lecun, Yann and Bottou, L{\'{e}}on and Bengio, Yoshua and Haffner, Patrick},
    doi = {10.1109/5.726791},
    title = {{Gradient-based learning applied to document recognition}},
    pages = {2278--2324},
    journal = {Proceedings of the IEEE},
    issn = {0018-9219},
    publisher = {IEEE},
    month = {11},
    number = {11},
    volume = {86},
    year = {1998}
}

@inproceedings{netzer2011reading,
    author = {Netzer, Yuval and Wang, Tao and Coates, Adam and Bissacco, Alessandro and Wu, Bo and Ng, Andrew Y},
    title = {{Reading digits in natural images with unsupervised feature learning}},
    address = {Granada, Spain},
    booktitle = {24th International Conference on Neural Information Processing Systems (NIPS 2011), NIPS Workshop on Deep Learning and Unsupervised Feature Learning},
    month = {12},
    year = {2011}
}

@techreport{krizhevsky2009learning,
    author = {Krizhevsky, Alex},
    title = {{Learning multiple layers of features from tiny images}},
    institution = {University of Toronto},
    location = {Toronto, Canada},
    month = {4},
    number = {TR-2009},
    type = {Technical Report},
    year = {2009}
}

@techreport{xiao2017fashion,
    author = {Xiao, Han and Rasul, Kashif and Vollgraf, Roland},
    doi = {10.48550/arXiv.1708.07747},
    title = {{Fashion-MNIST: a Novel Image Dataset for Benchmarking Machine Learning Algorithms}},
    institution = {arXiv},
    month = {8},
    number = {1708.07747},
    type = {cs.LG},
    year = {2017}
}

@inproceedings{li2025achieving,
    author = {Li, Zhe and Ying, Bicheng and Liu, Zidong and Dong, Chaosheng and Yang, Haibo},
    title = {{Achieving dimension-free communication in federated learning via zeroth-order optimization}},
    publisher = {OpenReview.net},
    volume = {2025},
    address = {Singapore, Singapore},
    booktitle = {13rd International Conference on Learning Representations (ICLR 2025)},
    month = {4},
    year = {2025}
}

@inproceedings{liu2022right,
    author = {Liu, Yi and Xu, Lei and Yuan, Xingliang and Wang, Cong and Li, Bo},
    doi = {10.1109/INFOCOM48880.2022.9796721},
    title = {{The right to be forgotten in federated learning: An efficient realization with rapid retraining}},
    pages = {1749--1758},
    publisher = {IEEE},
    address = {Virtual Conference},
    booktitle = {41st IEEE International Conference on Computer Communications (INFOCOM 2022)},
    month = {5},
    year = {2022}
}

@article{mora2026federated,
    author = {Mora, Alessio and Valerio, Lorenzo and Bellavista, Paolo},
    doi = {10.1109/TMC.2026.3666356},
    title = {{Federated Unlearning via Distilled Data}},
    pages = {10607 -- 10618},
    journal = {IEEE Transactions on Mobile Computing},
    issn = {1536-1233},
    publisher = {IEEE},
    number = {7},
    volume = {25},
    year = {2026}
}

@inproceedings{zhang2026oblivionis,
    author = {Zhang, Fuyao and Yan, Xinyu and Wu, Tiantong and Li, Wenjie and Chen, Tianxiang and Cao, Yang and Yan, Ran and Huang, Longtao and Lim, Wei Yang Bryan and Yang, Qiang},
    doi = {10.1609/aaai.v40i33.40045},
    title = {{Oblivionis: A Lightweight Learning and Unlearning Framework for Federated Large Language Models}},
    pages = {28185--28193},
    publisher = {AAAI},
    volume = {40},
    address = {Singapore, Singapore},
    booktitle = {40th AAAI Conference on Artificial Intelligence (AAAI 2026)},
    month = {1},
    year = {2026}
}

@inproceedings{liu2021federaser,
    author = {Liu, Gaoyang and Ma, Xiaoqiang and Yang, Yang and Wang, Chen and Liu, Jiangchuan},
    doi = {10.1109/IWQOS52092.2021.9521274},
    title = {{FedEraser: Enabling efficient client-level data removal from federated learning models}},
    publisher = {IEEE},
    address = {Chiyoda, Japan},
    booktitle = {29th IEEE/ACM International Symposium on Quality of Service (IWQOS 2021)},
    month = {6},
    year = {2021}
}

@inproceedings{pan2025federated,
    author = {Pan, Zibin and Wang, Zhichao and Li, Chi and Zheng, Kaiyan and Wang, Boqi and Tang, Xiaoying and Zhao, Junhua},
    doi = {10.1609/aaai.v39i19.34181},
    title = {{Federated unlearning with gradient descent and conflict mitigation}},
    pages = {19804--19812},
    publisher = {AAAI},
    address = {Philadelphia, PA},
    booktitle = {39th AAAI Conference on Artificial Intelligence (AAAI 2025)},
    month = {2},
    year = {2025}
}

@article{wang2024server,
    author = {Wang, Pengfei and Song, Wei and Qi, Heng and Zhou, Changjun and Li, Fuliang and Wang, Yong and Sun, Peng and Zhang, Qiang},
    doi = {10.1109/TSC.2024.3355188},
    title = {{Server-initiated federated unlearning to eliminate impacts of low-quality data}},
    pages = {1196--1211},
    journal = {IEEE Transactions on Services Computing},
    publisher = {IEEE},
    number = {3},
    volume = {17},
    year = {2024}
}

@article{su2024f2ul,
    author = {Su, Weijian and Wang, Pengfei and Wang, Zhiyi and Muhammad, Ameen and Tao, Tiwei and Tong, Xiangrong and Zhang, Qiang},
    doi = {10.1109/TMC.2024.3429228},
    title = {{F2ul: Fairness-aware federated unlearning for data trading}},
    pages = {13539--13555},
    journal = {IEEE Transactions on Mobile Computing},
    issn = {1536-1233},
    publisher = {IEEE},
    month = {12},
    number = {12},
    volume = {23},
    year = {2024}
}

@inproceedings{che2023fast,
    author = {Che, Tianshi and Zhou, Yang and Zhang, Zijie and Lyu, Lingjuan and Liu, Ji and Yan, Da and Dou, Dejing and Huan, Jun},
    title = {{Fast federated machine unlearning with nonlinear functional theory}},
    publisher = {PMLR},
    address = {Honolulu, HI},
    booktitle = {40th International Conference on Machine Learning (ICML 2023)},
    month = {7},
    year = {2023}
}

@article{zhang2023fedrecovery,
    author = {Zhang, Lefeng and Zhu, Tianqing and Zhang, Haibin and Xiong, Ping and Zhou, Wanlei},
    doi = {10.1109/TIFS.2023.3297905},
    title = {{Fedrecovery: Differentially private machine unlearning for federated learning frameworks}},
    pages = {4732--4746},
    journal = {IEEE Transactions on Information Forensics and Security},
    publisher = {IEEE},
    volume = {18},
    year = {2023}
}

@article{zhou2025federated,
    author = {Zhou, Changjun and Pan, Chenglin and Li, Minglu and Wang, Pengfei},
    doi = {10.1109/TMC.2025.3563265},
    title = {{Federated unlearning with fast recovery}},
    pages = {9709 -- 9725},
    journal = {IEEE Transactions on Mobile Computing},
    issn = {1536-1233},
    publisher = {IEEE},
    number = {10},
    volume = {24},
    year = {2025}
}

@inproceedings{fraboni2024sifu,
    author = {Fraboni, Yann and Van Waerebeke, Martin and Scaman, Kevin and Vidal, Richard and Kameni, Laetitia and Lorenzi, Marco},
    title = {{Sifu: Sequential informed federated unlearning for efficient and provable client unlearning in federated optimization}},
    pages = {3457--3465},
    publisher = {PMLR},
    address = {Valencia, Spain},
    booktitle = {27th International Conference on Artificial Intelligence and Statistics (AISTATS 2024)},
    month = {5},
    year = {2024}
}

@article{fang2022communication,
    author = {Fang, Wenzhi and Yu, Ziyi and Jiang, Yuning and Shi, Yuanming and Jones, Colin N and Zhou, Yong},
    doi = {10.1109/TSP.2022.3214122},
    title = {{Communication-efficient stochastic zeroth-order optimization for federated learning}},
    pages = {5058--5073},
    journal = {IEEE Transactions on Signal Processing},
    issn = {1053-587X},
    publisher = {IEEE},
    volume = {70},
    year = {2022}
}

@inproceedings{ling2024convergence,
    author = {Ling, Zhenqing and Chen, Daoyuan and Yao, Liuyi and Li, Yaliang and Shen, Ying},
    doi = {10.1145/3637528.3671865},
    title = {{On the convergence of zeroth-order federated tuning for large language models}},
    pages = {1827--1838},
    publisher = {ACM},
    address = {Barcelona, Spain},
    booktitle = {30th ACM SIGKDD Conference on Knowledge Discovery and Data Mining (KDD 2024)},
    month = {8},
    year = {2024}
}

@inproceedings{li2026converge,
    author = {Li, Zhe and Ying, Bicheng and Liu, Zidong and Dong, Chaosheng and Yang, Haibo},
    title = {{Converge faster, talk less: Hessian-informed federated zeroth-order optimization}},
    publisher = {OpenReview.net},
    volume = {2026},
    address = {Rio de Janeiro, Brazil},
    booktitle = {14th International Conference on Learning Representations (ICLR 2026)},
    month = {4},
    year = {2026}
}

@inproceedings{zhang2024dpzero,
    author = {Zhang, Liang and Li, Bingcong and Thekumparampil, Kiran Koshy and Oh, Sewoong and He, Niao},
    title = {{DPZero: private fine-tuning of language models without backpropagation}},
    pages = {59210--59246},
    publisher = {PMLR},
    address = {Vienna, Austria},
    booktitle = {41st International Conference on Machine Learning (ICML 2024)},
    month = {7},
    year = {2024}
}

@article{liu2026differentially,
    author = {Liu, Zhihao and Lou, Jian and Bao, Wenjie and Hu, Yuke and Wang, Wenli and Qin, Zhan and Ren, Kui},
    doi = {10.1109/TIFS.2026.3679194},
    title = {{Differentially Private Zeroth-Order Methods for Scalable Large Language Model Fine-tuning}},
    pages = {4163 -- 4175},
    journal = {IEEE Transactions on Information Forensics and Security},
    publisher = {IEEE},
    volume = {21},
    year = {2026}
}

@inproceedings{bao2025unlocking,
    author = {Bao, Ergute and Jiang, Yangfan and Wei, Fei and Xiao, Xiaokui and Li, Zitao and Li, Yaliang and Ding, Bolin},
    title = {{Unlocking the Power of Differentially Private Zeroth-order Optimization for Fine-tuning $\$LLMs$\$}},
    pages = {1569--1588},
    publisher = {USENIX},
    address = {Seattle, WA},
    booktitle = {34th USENIX Security Symposium (USENIX Security 2025)},
    month = {8},
    year = {2025}
}

@inproceedings{mcmahan2017communication,
    author = {McMahan, H. Brendan and Moore, Eider and Ramage, Daniel and Hampson, Seth and Arcas, Blaise Ag{\"{u}}era y},
    title = {{Communication-efficient learning of deep networks from decentralized data}},
    pages = {1273--1282},
    address = {Fort Lauderdale, FL},
    booktitle = {20th International Conference on Artificial Intelligence and Statistics (AISTATS 2017)},
    month = {4},
    year = {2017}
}

@inproceedings{fu2020fractrain,
    author = {Fu, Yonggan and You, Haoran and Zhao, Yang and Wang, Yue and Li, Chaojian and Gopalakrishnan, Kailash and Wang, Zhangyang and Lin, Yingyan},
    title = {{FracTrain: Fractionally Squeezing Bit Savings Both Temporally and Spatially for Efficient DNN Training}},
    pages = {12127--12139},
    editor = {Larochelle, Hugo and Ranzato, M. and Hadsell, R. and Balcan, Maria-Florina F and Lin, Hsuan-Tien},
    publisher = {Curran Associates Inc.},
    volume = {33},
    address = {Virtual Conference},
    booktitle = {34th Conference on Neural Information Processing Systems (NeurIPS 2020)},
    month = {12},
    year = {2020}
}

@inproceedings{huang2020deepadapter,
    author = {Huang, Yakun and Qiao, Xiuquan and Tang, Jian and Ren, Pei and Liu, Ling and Pu, Calton and Chen, Junliang},
    doi = {10.1109/INFOCOM41043.2020.9155379},
    title = {{Deepadapter: A collaborative deep learning framework for the mobile web using context-aware network pruning}},
    pages = {834--843},
    publisher = {IEEE},
    address = {Virtual Conference},
    booktitle = {39th IEEE International Conference on Computer Communications (INFOCOM 2020)},
    month = {7},
    year = {2020}
}

@inproceedings{zhang2019your,
    author = {Zhang, Linfeng and Song, Jiebo and Gao, Anni and Chen, Jingwei and Bao, Chenglong and Ma, Kaisheng},
    doi = {10.1109/ICCV.2019.00381},
    title = {{Be your own teacher: Improve the performance of convolutional neural networks via self distillation}},
    pages = {3712--3721},
    publisher = {IEEE},
    address = {Seoul, South Korea},
    booktitle = {IEEE/CVF International Conference on Computer Vision (ICCV 2019)},
    month = {10},
    year = {2019}
}

@inproceedings{zhang2023fedpetuning,
    author = {Zhang, Zhuo and Yang, Yuanhang and Dai, Yong and Wang, Qifan and Yu, Yue and Qu, Lizhen and Xu, Zenglin},
    doi = {10.18653/v1/2023.findings-acl.632},
    title = {{Fedpetuning: When federated learning meets the parameter-efficient tuning methods of pre-trained language models}},
    pages = {9963--9977},
    publisher = {ACL},
    address = {Toronto, Canada},
    booktitle = {Findings of the Association for Computational Linguistics: ACL 2023},
    month = {7},
    year = {2023}
}

@book{voigt2024EU,
    author = {Voigt, Paul and Von dem Bussche, Axel},
    doi = {10.1007/978-3-031-62328-8},
    title = {{The EU General Data Protection Regulation (GDPR): A Practical Guide}},
    isbn = {978-3-031-62327-1},
    pages = {339},
    publisher = {Springer},
    year = {2024}
}

@inproceedings{maini2024tofu,
    author = {Maini, Pratyush and Feng, Zhili and Schwarzschild, Avi and Lipton, Zachary C. and Kolter, J. Zico},
    title = {{TOFU: A Task of Fictitious Unlearning for LLMs}},
    publisher = {Curran Associates Inc.},
    volume = {37},
    address = {Vancouver, Canada},
    booktitle = {38th Conference on Neural Information Processing Systems (NeurIPS 2024), NeurIPS 2024 Workshop on  Red Teaming GenAI},
    month = {12},
    year = {2024}
}

@inproceedings{ullah2021machine,
    author = {Ullah, Enayat and Mai, Tung and Rao, Anup and Rossi, Ryan A and Arora, Raman},
    title = {{Machine unlearning via algorithmic stability}},
    pages = {4126--4142},
    publisher = {PMLR},
    address = {Boulder, CO},
    booktitle = {34th Annual Conference on Learning Theory (COLT 2021)},
    month = {8},
    year = {2021}
}

@incollection{vershynin2012introduction,
    author = {Vershynin, Roman},
    doi = {10.1017/CBO9780511794308.006},
    title = {{Introduction to the non-asymptotic analysis of random matrices}},
    pages = {210--268},
    booktitle = {Compressed Sensing: Theory and Applications},
    publisher = {Cambridge University Press},
    year = {2012}
}

@inproceedings{chen2026fedsubmuon,
    author = {Chen, Shaolong and Tao, Youming and Chen, Shuzhen and Dressler, Falko and Ye, Qingqing and Wang, Di},
    title = {{FedSubMuon: Communication-Efficient Federated LLM Fine-Tuning via Structured Subspace Muon}},
    publisher = {ACL},
    address = {Budapest, Hungary},
    booktitle = {Conference on Empirical Methods in Natural Language Processing (EMNLP 2026)},
    month = {10},
    year = {2026}
}

@article{chen2025adaptive,
    author = {Chen, Shuzhen and Yuan, Yuan and Tao, Youming and Wang, Tianzhu and Cai, Zhipeng and Yu, Dongxiao},
    doi = {10.1016/j.tcs.2024.114987},
    title = {{Adaptive pruning-based Newton's method for distributed learning}},
    pages = {114987},
    journal = {Theoretical Computer Science},
    issn = {0304-3975},
    publisher = {Elsevier},
    month = {2},
    volume = {1026},
    year = {2025}
}

@article{tao2024private,
    author = {Tao, Youming and Chen, Shuzhen and Zhang, Congwei and Wang, Di and Yu, Dongxiao and Cheng, Xiuzhen and Dressler, Falko},
    doi = {10.1109/TMC.2024.3411295},
    title = {{Private Over-the-Air Federated Learning at Band-Limited Edge}},
    pages = {12444--12460},
    journal = {IEEE Transactions on Mobile Computing},
    issn = {1536-1233},
    publisher = {IEEE},
    month = {12},
    number = {12},
    volume = {23},
    year = {2024}
}

@inproceedings{wu2026trajectoryaware,
    author = {Wu, Hengliang and Yang, Jiale and Chen, Shuzhen and Wang, Di and Yu, Dongxiao and Tao, Youming},
    title = {{Trajectory-Aware Certified Decentralized Unlearning via SGD Stability}},
    publisher = {PMLR},
    address = {Seoul, South Korea},
    booktitle = {43rd International Conference on Machine Learning (ICML 2026), Poster Session},
    month = {7},
    year = {2026}
}

@inproceedings{wu2026certified,
    author = {Wu, Hengliang and Tao, Youming and Zhou, Anhao and Chen, Shuzhen and Dressler, Falko and Yu, Dongxiao},
    doi = {10.1109/INFOCOM59046.2026.11571416},
    title = {{Certified Unlearning in Decentralized Federated Learning}},
    pages = {1--10},
    publisher = {IEEE},
    isbn = {979-8-3315-4961-9},
    address = {Tokyo, Japan},
    booktitle = {45th IEEE International Conference on Computer Communications (INFOCOM 2026)},
    month = {5},
    year = {2026}
}

@inproceedings{tao2026autosub,
    author = {Tao, Youming and Simsek, Meryem and Dressler, Falko},
    note = {to appear},
    title = {{AutoSub: Autonomous Submodel Averaging}},
    publisher = {IEEE},
    address = {Kobe, Japan},
    booktitle = {IEEE Annual Congress on Artificial Intelligence of Things (AioT 2026)},
    month = {12},
    year = {2026}
}

@article{tao2026one-shot,
    author = {Tao, Youming and Yin, Weijun and Xie, Huanyi and Yao, Yanqing and Yu, Dongxiao and Cheng, Xiuzhen and Wang, Di and Dressler, Falko},
    doi = {10.1109/TON.2026.3745736},
    note = {to appear},
    title = {{One-Shot Federated Model Editing for Device Dynamics at the Edge}},
    journal = {IEEE Transactions on Networking},
    issn = {2998-4157},
    publisher = {IEEE},
    month = {10},
    year = {2026}
}

@techreport{li2026not,
    author = {Li, Jie and Tao, Youming and He, Bowei and Ding, Meng},
    title = {{Not All Change Is Damage: Update-Matched Unlearning for Large Language Models}},
    institution = {arXiv},
    month = {10},
    number = {2610.21230},
    type = {cs.LG},
    url = {https://arxiv.org/abs/2610.21230},
    year = {2026}
}

@inproceedings{reisizadeh2020fedpaq,
    author = {Reisizadeh, Amirhossein and Mokhtari, Aryan and Hassani, Hamed and Jadbabaie, Ali and Pedarsani, Ramtin},
    title = {{FedPAQ: A Communication-Efficient Federated Learning Method with Periodic Averaging and Quantization}},
    pages = {2021--2031},
    publisher = {PMLR},
    volume = {108},
    address = {Virtual Conference},
    booktitle = {23rd International Conference on Artificial Intelligence and Statistics (AISTATS 2020)},
    month = {8},
    year = {2020}
}

@inproceedings{lin2018deep,
    author = {Lin, Yujun and Han, Song and Mao, Huizi and Wang, Yu and Dally, Bill},
    title = {{Deep Gradient Compression: Reducing the Communication Bandwidth for Distributed Training}},
    publisher = {ICLR},
    address = {Vancouver, Canada},
    booktitle = {6th International Conference on Learning Representations (ICLR 2018)},
    month = {4},
    year = {2018}
}

@inproceedings{lin2020ensemble,
    author = {Lin, Tao and Kong, Lingjing and Stich, Sebastian U and Jaggi, Martin},
    title = {{Ensemble distillation for robust model fusion in federated learning}},
    pages = {2351--2363},
    editor = {Ranzato, M. and Beygelzimer, A. and Dauphin, Y. and Liang, P.S. and Vaughan, J. Wortman},
    publisher = {Curran Associates Inc.},
    isbn = {978-1-71382-954-6},
    address = {Virtual Conference},
    booktitle = {34th Conference on Neural Information Processing Systems (NeurIPS 2020)},
    month = {12},
    year = {2020}
}

@inproceedings{horvath2021fjord,
    author = {Horvath, Samuel and Laskaridis, Stefanos and Almeida, Mario and Leontiadis, Ilias and Venieris, Stylianos I. and Lane, Nicholas D.},
    title = {{FjORD: Fair and Accurate Federated Learning under Heterogeneous Targets with Ordered Dropout}},
    editor = {Ranzato, Marc' aurelio and Beygelzimer, Alina and Dauphin, Yann and Liang, Percy and Vaughan, Jennifer Wortman},
    publisher = {Curran Associates Inc.},
    volume = {34},
    isbn = {978-1-71384-539-3},
    address = {Virtual Conference},
    booktitle = {35th Conference on Neural Information Processing Systems (NeurIPS 2021)},
    month = {12},
    year = {2021}
}

@inproceedings{ding2025lbifl,
    author = {Ding, Li and Zhang, Hao and Dai, Wenrui and Li, Chenglin and Lu, Weijia and Yang, Zhifei and Zhang, Xiaodong and Ma, Xiaofeng and Zou, Junni and Xiong, Hongkai},
    title = {{LBI-FL: Low-Bit Integerized Federated Learning with Temporally Dynamic Bit-Width Allocation}},
    pages = {13885--13899},
    publisher = {PMLR},
    volume = {267},
    address = {Vancouver, Canada},
    booktitle = {42nd International Conference on Machine Learning (ICML 2025)},
    month = {7},
    year = {2025}
}

@inproceedings{guastella2025sparsyfed,
    author = {Guastella, Adriano and Sani, Lorenzo and Iacob, Alex and Mora, Alessio and Bellavista, Paolo and Lane, Nicholas Donald},
    title = {{SparsyFed: Sparse Adaptive Federated Learning}},
    publisher = {OpenReview.net},
    volume = {2025},
    address = {Singapore, Singapore},
    booktitle = {13rd International Conference on Learning Representations (ICLR 2025)},
    month = {4},
    year = {2025}
}

@inproceedings{fan2025fedcot,
    author = {Fan, Tao and Chen, Weijing and Kang, Yan and Ma, Guoqiang and Gu, Hanlin and Song, Yuanfeng and Fan, Lixin and Yang, Qiang},
    doi = {10.18653/v1/2025.findings-emnlp.454},
    title = {{FedCoT: Federated Chain-of-Thought Distillation for Large Language Models}},
    pages = {8546--8557},
    publisher = {ACL},
    isbn = {979-8-89176-335-7},
    address = {Suzhou, China},
    booktitle = {30th Conference on Empirical Methods in Natural Language Processing (EMNLP 2025)},
    month = {11},
    year = {2025}
}

\newpage
\onecolumn

\appendices
\section{Additional experimental results}
\subsection{GPU Memory cost of FedZOUL-S/NS}
\label{app:memory}
We evaluate peak GPU memory in Figure~\ref{fig:memory-cost-fedzoul-s/ns-baselines}. \textsc{FedZOUL-S} and \textsc{FedZOUL-NS} achieve the lowest peak memory across all tasks by avoiding parameter-sized gradients and backpropagation activations, while maintaining performance comparable to re-optimization. On CIFAR-10, each method requires $128.08$ MB, $60.8\%$ below FedRR and $54.0\%$ below FUDD. On TOFU, their peak memory is $6.63$ GB for the 3B model, compared with $11.15$ GB for Oblivionis, and the reduction relative to Oblivionis increases from $24.3\%$ for the 1B model to $40.5\%$ for the 3B model. These results show that our methods maintain low peak memory across tasks, supporting deployment on resource-constrained edge devices. This advantage becomes more pronounced as model size increases, indicating favorable scalability.

\begin{figure}[h]
    \centering
    \includegraphics[height=0.285\linewidth]{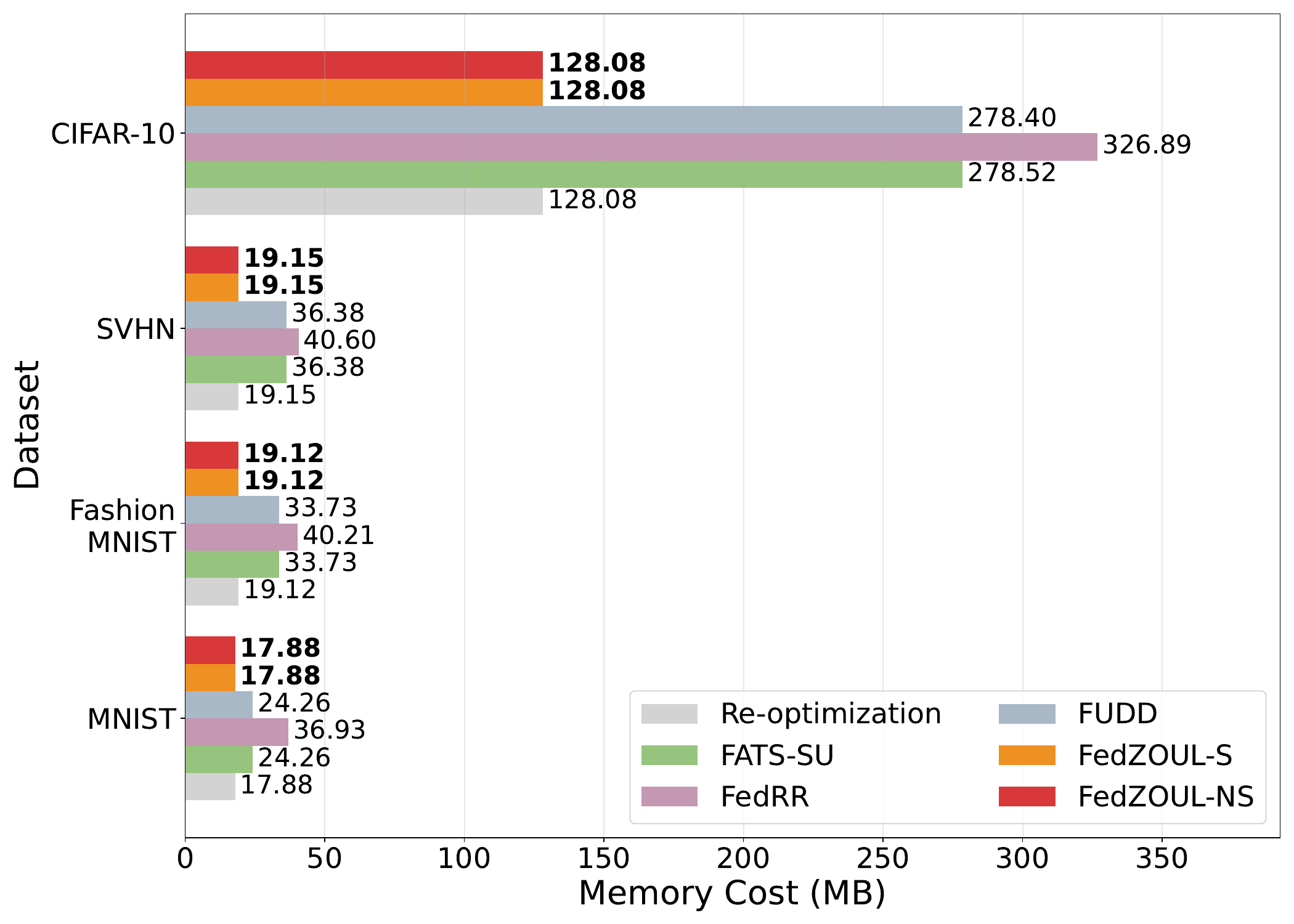}
    \hfill
    \includegraphics[height=0.285\linewidth]{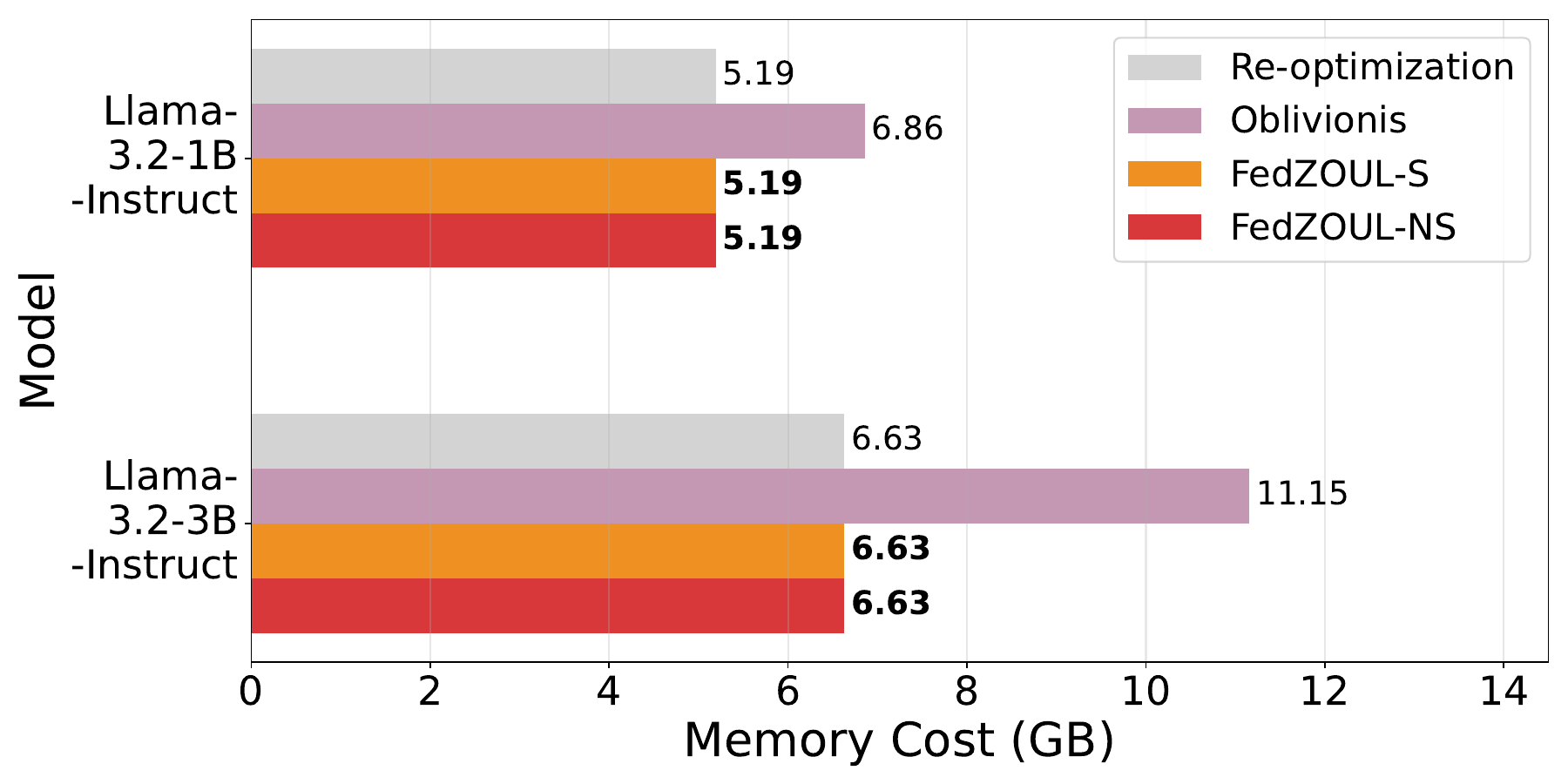}
    \caption{Peak GPU memory of FedZOUL-S/NS against re-optimization and first-order unlearning baselines under stream unlearning requests. Left: four vision datasets. Right: TOFU (Split95).}
    \label{fig:memory-cost-fedzoul-s/ns-baselines}
\end{figure}

\subsection{Impact of Hyperparameters $K$ and $q$}
\label{app:Kq}
We investigate how the accuracy, communication cost, and time cost of \textsc{FedZOO-S} and \textsc{FedZOO-NS} vary with $q$ and $K$ on FashionMNIST, as shown in Figure~\ref{fig:fashionmnist-fedzoo-s/ns-impact-k/q}. For \textsc{FedZOO-S}, increasing $q$ generally improves accuracy, although the gains are not monotonic. At $K=1024$, accuracy increases from $64.46\%$ at $q=1$ to $68.12\%$ at $q=3$, before decreasing to $67.71\%$ at $q=4$. \textsc{FedZOO-NS} also exhibits a non-monotonic response: increasing $q$ from 1 to 2 improves accuracy for every tested $K$. Its best accuracy is attained at $q=2$ for $K\in\{1024,2048,4096\}$ and at $q=3$ for $K=8192$.

The effect of $K$ differs between the two methods. \textsc{FedZOO-S} achieves its highest accuracy at $K=1024$ for $q\in\{2,3,4\}$, whereas $K=2048$ performs best at $q=1$. Increasing $K$ can reduce accuracy: at $q=3$, accuracy falls from $68.12\%$ to $64.29\%$ as $K$ increases from 1024 to 8192. For \textsc{FedZOO-NS}, larger seed pools generally yield higher accuracy, but the trend is not monotonic. For example, at $q=1$, accuracy increases from $55.38\%$ at $K=1024$ to $64.11\%$ at $K=4096$, then decreases to $60.19\%$ at $K=8192$.

In terms of system overhead, the per-round communication cost is independent of $q$ and increases linearly with $K$, with identical costs for both methods. Each of the $m=10$ participating clients uploads a $K$-dimensional float32 accumulator, while the server broadcasts the global accumulator, resulting in a per-round communication cost of $4(m+1)K=44K$ bytes. The average per-round runtime increases with both $q$ and $K$ in all tested configurations. At $K=4096$, increasing $q$ from 1 to 2 improves \textsc{FedZOO-NS} accuracy from $64.11\%$ to $71.26\%$, while increasing per-round runtime from $2.883$ to $3.073$~s and leaving communication cost unchanged. We use $K=4096$ as an intermediate pool size and $q=1$ to limit local computation.

\begin{figure}[t]
    \centering

    \begin{minipage}{0.75\linewidth}
    \centering
    \includegraphics[width=0.49\linewidth]{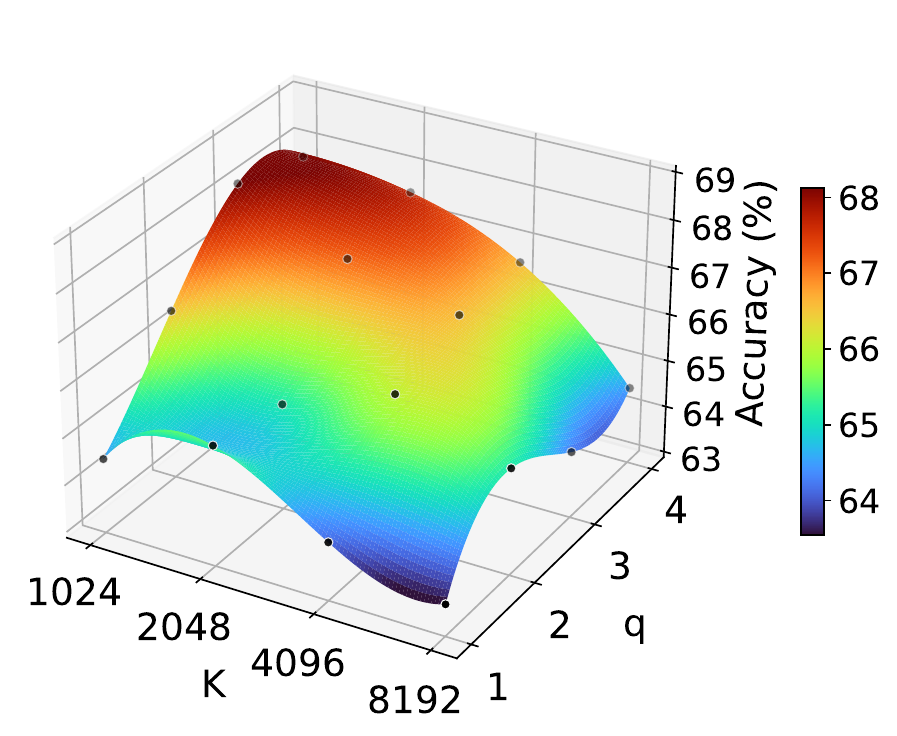}
    \hfill  
    \includegraphics[width=0.49\linewidth]{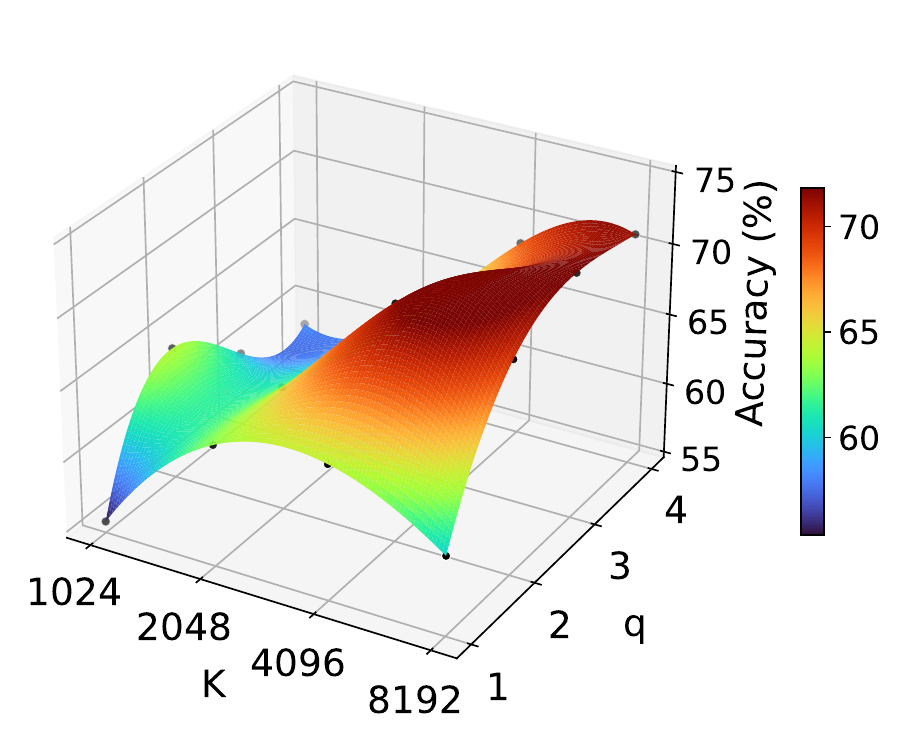}

    \vspace{0.1cm}

    \includegraphics[width=0.49\linewidth]{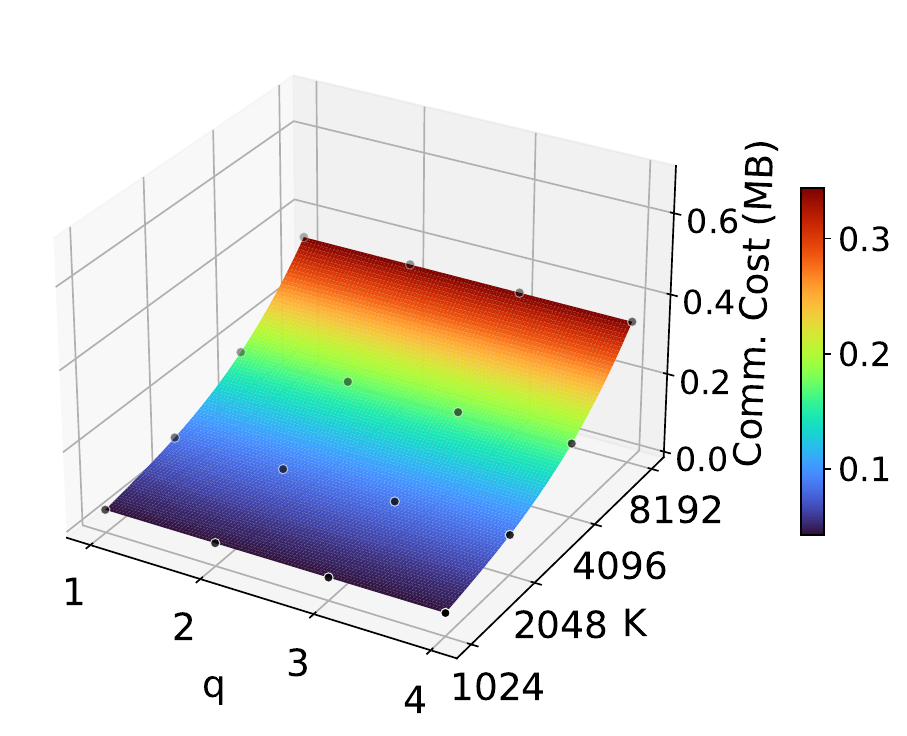}
    \hfill
    \includegraphics[width=0.49\linewidth]{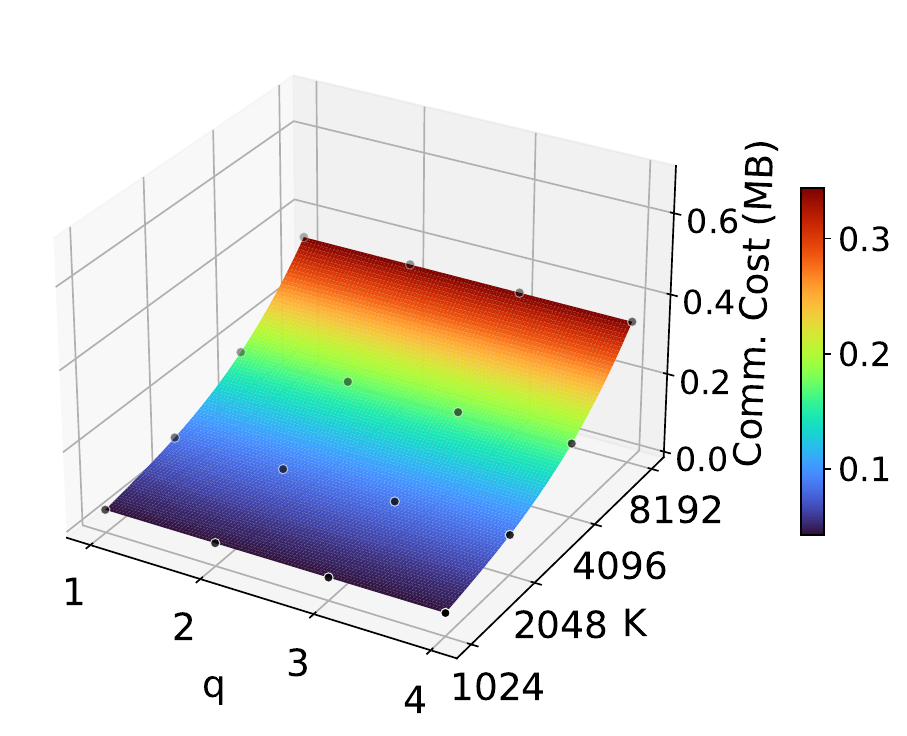}

    \includegraphics[width=0.49\linewidth]{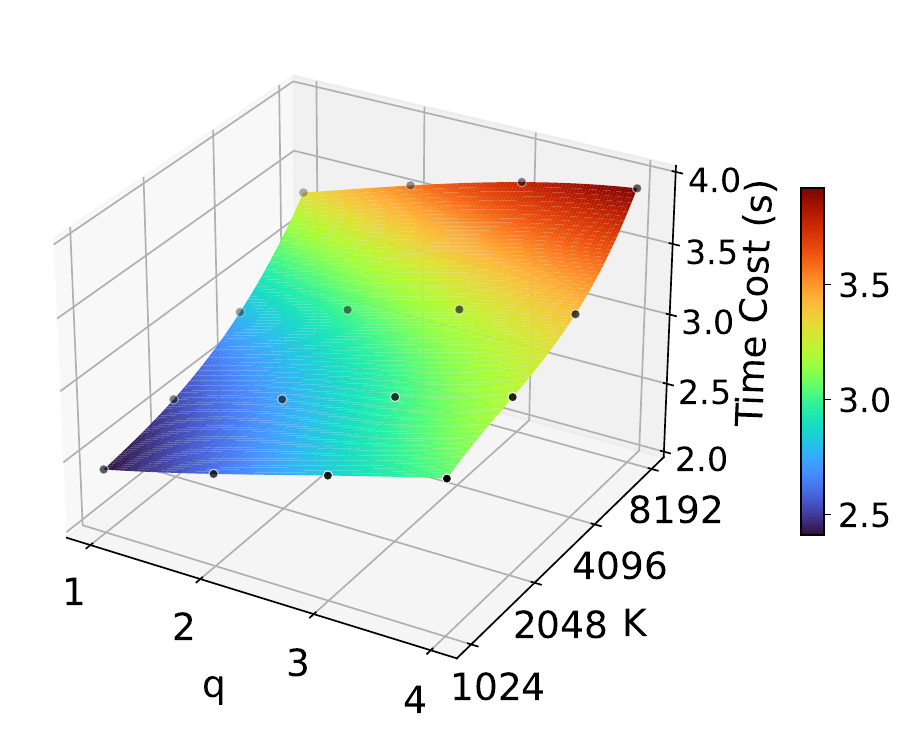}
    \hfill
    \includegraphics[width=0.49\linewidth]{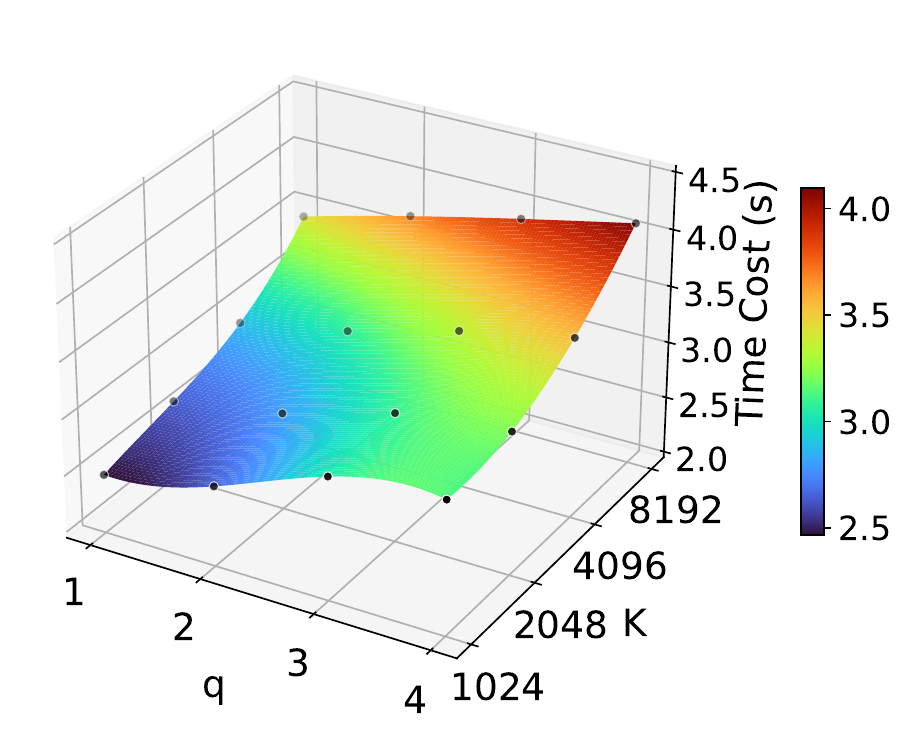}
    \end{minipage}

    \caption{Impact of the number of per-sample perturbation directions ($q$) and the scalar gradient accumulator size ($K$) on the accuracy, per-round communication cost, and per-round time cost of FedZOO-S and FedZOO-NS on FashionMNIST. Left column: FedZOO-S, right column: FedZOO-NS. Top row: accuracy, middle row: communication cost, bottom row: time cost.}
    \label{fig:fashionmnist-fedzoo-s/ns-impact-k/q}
\end{figure}

\newpage
\section{Complete Theoretical Analysis}
\label{sec:theory_appendix}

We now establish the theoretical guarantees of the proposed learning and
unlearning procedures. Two questions are central. First, after a sample is
deleted, we need to show that the state returned by the unlearning algorithm
has exactly the same distribution as a fresh execution on the retained data,
while avoiding a full rerun whenever possible. Second, because both training
algorithms reuse a finite set of perturbation directions, we need to
understand what optimization guarantee this restriction provides for the
original federated objective.

The two learning algorithms support exact unlearning in different ways. In
\textsc{FedZOO-S}, a deletion has no effect unless the removed sample was
actually selected during training. In \textsc{FedZOO-NS}, Gaussian noise
provides an additional possibility: even when the deleted sample changes a
local update, the old and retained-data updates can sometimes be coupled to
the same value. We first analyze these two mechanisms and their resulting
unlearning costs. We then study the optimization behavior of the learning
algorithms under the shared finite seed pool and relate the resulting
subspace guarantee to the full model space.

The perturbation pool is sampled independently of the federated data and is
shared by the neighboring executions considered below. We therefore condition
on its realization throughout the analysis. All conditional guarantees remain
valid after averaging over the seed-pool randomness.

We first ask how a deletion can change the state that has already been
produced. Exact unlearning requires more than a nearby model: after processing
the request, the entire maintained state must have the same distribution as if
the deleted sample had never been present. The useful computation is therefore
precisely the part of the realized trajectory that can be kept without changing
this law.

Let $T=RE$ denote the total number of local positions, and index them in
chronological order by
\begin{equation}
\label{eq:chronological_index_theory}
\tau(r,t)=(r-1)E+t.
\end{equation}
For a request arriving after $(R_{ul},t_{ul})$, write
$T_{ul}=\tau(R_{ul},t_{ul})$. A position contains the local updates of the
$m$ participating clients and, when $t=E$, the corresponding server
aggregation. Position zero denotes the initialization before the first client
set is sampled.

For $\bullet\in\{\mathsf S,\mathsf{NS}\}$, let
$\mathcal O_{\bullet}^{[\tau]}(\mathcal D)$ denote the joint law of the
optimization variables and all saved state after $\tau$ positions when
training on $\mathcal D$. We use
$\mathcal O_{\bullet,\mathrm{out}}(\mathcal D)$ for the marginal law of the
final reconstructed model. With this notation, exact unlearning at
$T_{ul}$ means that the state produced after processing the deletion has law
$\mathcal O_{\bullet}^{[T_{ul}]}(\mathcal D^-)$, and continuing training from
that state must yield $\mathcal O_{\bullet}^{[T]}(\mathcal D^-)$. Whenever
part of the trajectory is rerun, the newly generated state replaces the old
saved state over that part of the execution.

\subsection{FedZOO-S Stability and Exact Unlearning}
\label{sec:fedzoos_theory}

We first consider \textsc{FedZOO-S} and its unlearning procedure \textsc{FedZOUL-S}.

A deleted sample can change a \textsc{FedZOO-S} trajectory only if it was
actually drawn into a local mini-batch. This observation suggests coupling the
mini-batches before and after deletion so that they remain identical whenever
the deleted sample was not selected. The following elementary fact gives the
coupling we need.

\begin{lemma}[Deletion coupling for a uniform mini-batch]
\label{lem:uniform_batch_deletion}
Let $\mathcal D$ contain $n$ samples, let $x^{-}\in\mathcal D$, and let
$\mathcal D^{-}=\mathcal D\setminus\{x^{-}\}$, where $n-1\ge b$.
Draw $\mathcal B$ uniformly from the size-$b$ subsets of $\mathcal D$.
Construct $\mathcal B^{-}$ as follows:
\begin{equation}
\label{eq:batch_deletion_coupling}
\mathcal B^{-}
=
\begin{cases}
\mathcal B,
&x^{-}\notin\mathcal B,\\
\widehat{\mathcal B},
&x^{-}\in\mathcal B,
\end{cases}
\end{equation}
where, on the second branch, $\widehat{\mathcal B}$ is drawn uniformly from
the size-$b$ subsets of $\mathcal D^{-}$, independently of $\mathcal B$.
Then $\mathcal B^{-}$ is uniform over the size-$b$ subsets of
$\mathcal D^{-}$ and
\begin{equation}
\label{eq:batch_coupling_disagreement}
\mathbb P(\mathcal B^{-}\ne\mathcal B)
=
\frac bn.
\end{equation}
\end{lemma}

\begin{proof}[Proof of Lemma~\ref{lem:uniform_batch_deletion}]
Fix a size-$b$ subset $A\subseteq\mathcal D^{-}$. The first branch outputs
$A$ precisely when $\mathcal B=A$, while the second branch is entered with
probability $b/n$ and then outputs $A$ with probability
$1/\binom{n-1}{b}$. Hence
\begin{align}
\mathbb P(\mathcal B^{-}=A)
&=
\frac1{\binom nb}
+
\frac bn\frac1{\binom{n-1}{b}}
\notag\\
&=
\frac{n-b}{n}\frac1{\binom{n-1}{b}}
+
\frac bn\frac1{\binom{n-1}{b}}
=
\frac1{\binom{n-1}{b}}.
\end{align}
Since the two mini-batches cannot agree on the second branch, the
disagreement probability is $b/n$.
\end{proof}

When $\mathcal B^{-}=\mathcal B$, use the same independently sampled size-$q$
seed set for every sample in the two executions. On the second branch, draw
the retained execution's sample--seed correspondence afresh according to
\textsc{FedZOO-S}. Because seed sampling is independent of data sampling,
the second execution has the correct joint mini-batch--seed marginal.

Consider one deletion $x^{-}\in\mathcal D_{c^{-}}$ and write
$n_{c^{-}}=|\mathcal D_{c^{-}}|$. At local position $(r,t)$, define
\begin{equation}
\label{eq:fedzoos_hit_event}
\mathcal H_{r,t}(x^{-})
=
\{c^{-}\in\mathcal C^r,
\ x^{-}\in\mathcal B_{c^{-}}^{r,t}\}.
\end{equation}
For the first $\tau$ local positions, let
\begin{equation}
\label{eq:fedzoos_lambda}
\Lambda_{\mathsf S}(\tau,c^{-})
=
\min\left\{
1,
\frac{\tau mb}{Mn_{c^{-}}}
\right\}.
\end{equation}

The same coupling can now be followed through the whole training state. Use the
same initialization and client set in the two executions, couple unaffected
clients identically, and apply Lemma~\ref{lem:uniform_batch_deletion} whenever
the affected client participates. Until the deleted sample first appears in a
mini-batch, the sampled data, seed sets, coefficient increments, accumulators,
and reconstructed models are identical. If a first hit occurs, the retained-data
execution draws fresh local randomness from that position onward, while keeping
the already sampled client set for the remainder of the current round. Thus the
second execution has exactly the law of \textsc{FedZOO-S} on the retained data.

At any local position, a hit occurs with probability
\begin{equation}
\label{eq:fedzoos_hit_probability}
\frac{m}{M}\frac{b}{n_{c^{-}}}.
\end{equation}
The same client set is reused within a round, so these events need not be
independent. A union bound over the first $\tau$ positions nevertheless gives
\begin{equation}
\label{eq:fedzoos_state_disagreement}
\mathbb P(\text{the two states differ by position }\tau)
\le
\Lambda_{\mathsf S}(\tau,c^{-}).
\end{equation}

\begin{theorem}[TV stability of FedZOO-S]
\label{thm:fedzoos_tv}
For every admissible adjacent pair differing by deletion of
$x^{-}\in\mathcal D_{c^{-}}$,
\begin{equation}
\label{eq:fedzoos_output_tv}
\operatorname{TV}\left(
\mathcal O_{\mathsf S,\mathrm{out}}(\mathcal D),
\mathcal O_{\mathsf S,\mathrm{out}}(\mathcal D^{-})
\right)
\le
\Lambda_{\mathsf S}(T,c^{-}).
\end{equation}
If a target $0<\rho\le1$ is fixed before training and
\begin{equation}
\label{eq:fedzoos_subsampling_calibration}
\frac{Tmb}{M\underline n}
\le
\rho,
\end{equation}
then \textsc{FedZOO-S} is $\rho$-sample-level TV stable on the dataset class
satisfying~\eqref{eq:stream_size_lower_bound}.
\end{theorem}

\begin{proof}[Proof of Theorem~\ref{thm:fedzoos_tv}]
The construction above is a valid coupling of the complete training states on
$\mathcal D$ and $\mathcal D^{-}$. By
Eq.~\eqref{eq:fedzoos_state_disagreement}, their disagreement probability at
$\tau=T$ is at most $\Lambda_{\mathsf S}(T,c^{-})$. Projecting the coupled
states onto the final reconstructed model can only decrease total variation.
Under~\eqref{eq:fedzoos_subsampling_calibration},
$n_{c^{-}}\ge\underline n$ gives
$\Lambda_{\mathsf S}(T,c^{-})\le\rho$. Symmetry of total variation covers
both orientations of every adjacent pair.
\end{proof}

The bound follows the way a deletion enters the trajectory. At each local position, the target
client is selected with probability $m/M$, and conditional on participation,
the deleted sample enters its mini-batch with probability $b/n_{c^{-}}$.
Thus, increasing client participation or the mini-batch size makes it more
likely that the deletion changes the realized trajectory. No additional
noise is needed for this form of stability.

This stability statement also suggests the unlearning procedure itself. If the
deleted sample never appears before the request, the realized state already has
the retained-data law and can be kept. Otherwise, the two executions agree up
to the first affected mini-batch, so training only needs to be rerun from that
point. Let $\mathcal U_{\mathsf S}^{[T_{ul}]}$ denote the state returned by
\textsc{FedZOUL-S} at position $T_{ul}$ after processing the deletion, and
let $\mathcal U_{\mathsf S}^{T_{ul}\to T}$ denote the state after the
remaining training positions are completed.

\begin{theorem}[Exact unlearning of FedZOUL-S]
\label{thm:fedzouls_exact}
For every valid single-sample deletion
$x^{-}\in\mathcal D_{c^{-}}$ arriving after local position $T_{ul}$,
\begin{equation}
\label{eq:fedzouls_prefix_exact}
\mathcal U_{\mathsf S}^{[T_{ul}]}
\left(
\mathcal O_{\mathsf S}^{[T_{ul}]}(\mathcal D),
\{x^{-}\}
\right)
\overset{d}{=}
\mathcal O_{\mathsf S}^{[T_{ul}]}(\mathcal D^{-}).
\end{equation}
Consequently,
\begin{equation}
\label{eq:fedzouls_exact}
\mathcal U_{\mathsf S}^{T_{ul}\to T}
\left(
\mathcal O_{\mathsf S}^{[T_{ul}]}(\mathcal D),
\{x^{-}\}
\right)
\overset{d}{=}
\mathcal O_{\mathsf S}^{[T]}(\mathcal D^{-}),
\end{equation}
where both equalities hold jointly over the model and every maintained state.
If $\mathsf{Replay}_{\mathsf S}$ denotes the event that the deletion causes
training to be rerun from an affected position, then
\begin{equation}
\label{eq:fedzouls_replay_probability}
\mathbb P(\mathsf{Replay}_{\mathsf S})
\le
\Lambda_{\mathsf S}(T_{ul},c^{-})
\le
\frac{T_{ul}mb}{Mn_{c^{-}}}.
\end{equation}
Under~\eqref{eq:fedzoos_subsampling_calibration}, this probability is at most
$\rho$.
\end{theorem}

\begin{proof}[Proof of Theorem~\ref{thm:fedzouls_exact}]
Use the same coupling as in the proof of Theorem~\ref{thm:fedzoos_tv}
through position $T_{ul}$, and let $(r^\star,t^\star)$ be the first hit in
the original execution.

If no hit occurs, the two executions are identical up to the request time. After removing
$x^{-}$, \textsc{FedZOUL-S} keeps that state and continues from
$\mathsf{Next}(R_{ul},t_{ul})$. For the remainder of the current round it
keeps the saved client set and current local states, while all future
mini-batches and seed sets are drawn from the retained datasets. This is the
retained-data continuation of the second coupled execution.

If the first hit is $(r^\star,t^\star)$, the coupled executions agree through
the state immediately before that position. The unlearning algorithm restores
that common server and participating-client state and replays all
participating clients from $(r^\star,t^\star)$ on the retained datasets. The
saved client set $\mathcal C^{r^\star}$ is kept for the remainder of the
current round, while all local randomness from the hit onward and all
randomness in later rounds follow the retained-data transition law.

In both branches, the state returned at $T_{ul}$ has law
$\mathcal O_{\mathsf S}^{[T_{ul}]}(\mathcal D^{-})$. Completing the horizon
with the retained-data kernels proves~\eqref{eq:fedzouls_exact}. Training is rerun only if a hit appears before the request, giving
\eqref{eq:fedzouls_replay_probability}.
\end{proof}

Nothing in the preceding argument depends on there being only one request. It
can be applied repeatedly, provided future deletions are not chosen in response
to the randomness revealed by earlier training or unlearning. Consider a valid
request stream
\begin{equation}
\label{eq:nonadaptive_request_stream}
\bigl(c_j^{-},x_j^{-},T_{ul}^{(j)}\bigr)_{j=1}^{k},
\qquad
0\le T_{ul}^{(1)}\le\cdots\le T_{ul}^{(k)}\le T.
\end{equation}
The deleted samples and their request times may be fixed in
advance or generated by external randomness, but they do not depend on the
realized models, sampled mini-batches, Gaussian noise, or the outcomes of
previous deletion requests. Arrival positions refer to the original training
timeline and do not count any extra computation used to process earlier
requests.

Let $\mathcal D^{(1)}=\mathcal D$, and let $\mathcal D^{(j)}$ and
$\mathcal D^{(j+1)}$ be the datasets immediately before and after request
$j$. Thus $x_j^{-}\in\mathcal D_{c_j^{-}}^{(j)}$ is removed to form
$\mathcal D_{c_j^{-}}^{(j+1)}$, while the other local datasets are unchanged.
We keep the learning hyperparameters fixed throughout the stream and assume
that every retained client dataset has size at least
\begin{equation}
\label{eq:stream_size_lower_bound}
\underline n
\le
\min_{1\le j\le k+1,\,c\in[M]}|\mathcal D_c^{(j)}|,
\qquad
\underline n\ge b.
\end{equation}
For a single request we use the actual pre-request dataset size; $\underline n$
is only needed when one common bound is used for the whole stream.

\begin{corollary}[Sequential deletion requests for FedZOUL-S]
\label{cor:fedzouls_sequential}
Consider the valid non-adaptive single-sample request stream in
Eq.~\eqref{eq:nonadaptive_request_stream}, and let
$n_j=|\mathcal D_{c_j^{-}}^{(j)}|$ be the affected client's pre-request data
size. After every request $j$, the returned state at $T_{ul}^{(j)}$ has law
$\mathcal O_{\mathsf S}^{[T_{ul}^{(j)}]}(\mathcal D^{(j+1)})$.
If
$N_{\mathsf S}^{\mathrm{rep}}$ is the number of requests that trigger replay,
then
\begin{align}
\mathbb E[N_{\mathsf S}^{\mathrm{rep}}]
&\le
\sum_{j=1}^{k}
\min\left\{
1,
\frac{T_{ul}^{(j)}mb}{Mn_j}
\right\},
\label{eq:fedzouls_expected_replays}
\\
\mathbb P(N_{\mathsf S}^{\mathrm{rep}}\ge1)
&\le
\min\left\{
1,
\sum_{j=1}^{k}
\frac{T_{ul}^{(j)}mb}{Mn_j}
\right\}.
\label{eq:fedzouls_any_replay}
\end{align}
Under~\eqref{eq:stream_size_lower_bound} and
\eqref{eq:fedzoos_subsampling_calibration}, these bounds are at most
$k\rho$ and $\min\{1,k\rho\}$, respectively.
\end{corollary}

\begin{proof}[Proof of Corollary~\ref{cor:fedzouls_sequential}]
Fix the request stream and the perturbation pool. Let $X_j^{-}$ and $X_j^{+}$
be the complete saved states immediately before and after processing request
$j$, both at position $T_{ul}^{(j)}$. We prove by induction that
\begin{align}
X_j^{-}&\overset d=
\mathcal O_{\mathsf S}^{[T_{ul}^{(j)}]}(\mathcal D^{(j)}),
\notag\\
X_j^{+}&\overset d=
\mathcal O_{\mathsf S}^{[T_{ul}^{(j)}]}(\mathcal D^{(j+1)}).
\label{eq:fedzouls_stream_state_laws}
\end{align}
The first equality for $j=1$ follows from the initial learning run.
Given its validity for request $j$, apply
Theorem~\ref{thm:fedzouls_exact} with the fixed deletion $x_j^{-}$ and fresh
unlearning randomness to obtain the second equality. Continue with the
retained-data learning kernels and fresh randomness until
$T_{ul}^{(j+1)}$. This gives the first equality for request $j+1$ and closes
the induction. Equal arrival positions require no intervening learning.
These are marginal state laws over all learning and unlearning randomness;
they do not condition on the particular saved state produced by earlier
requests or on whether those requests required replay.

Let $I_j$ indicate whether request $j$ triggers replay. Applying the
single-request replay bound to the first marginal in
Eq.~\eqref{eq:fedzouls_stream_state_laws} gives
\[
\mathbb E[I_j]=\mathbb P(I_j=1)
\le\min\left\{1,\frac{T_{ul}^{(j)}mb}{Mn_j}\right\}.
\]
Since $N_{\mathsf S}^{\mathrm{rep}}=\sum_{j=1}^{k}I_j$, linearity of
expectation proves Eq.~\eqref{eq:fedzouls_expected_replays}; a union bound
proves Eq.~\eqref{eq:fedzouls_any_replay}. Independence of the $I_j$ is not
needed. The bounds are uniform in the perturbation pool, so averaging over
it preserves them. The calibrated bounds follow from $n_j\ge\underline n$
and $T_{ul}^{(j)}\le T$.
\end{proof}

For \textsc{FedZOO-NS}, subsampling plays the same first role: if the deleted
sample is not selected, nothing changes. The difference is what happens after a
hit. Replacing that sample changes the mean of the noisy coefficient vector but
not its covariance, so the old and retained-data update distributions still
overlap. This overlap can keep the two executions together even at an affected
step. We quantify that one-step discrepancy first and then accumulate it over
training.

\subsection{FedZOO-NS Stability}
\label{sec:fedzoons_stability}

For probability measures $P\ll Q$, write
\begin{align}
D_{\mathrm{KL}}(P\|Q)
&=
\int\log\left(\frac{dP}{dQ}\right)dP,
\notag\\
D_2(P\|Q)
&=
\log\int\left(\frac{dP}{dQ}\right)^2dQ,
\end{align}
and let $\chi^2(P\|Q)=\exp(D_2(P\|Q))-1$.

\begin{lemma}[Single-deletion replacement coupling]
\label{lem:ns_replacement_batch}
Let $\mathcal D=\mathcal D^{-}\cup\{x^{-}\}$ with
$|\mathcal D|=n$ and $|\mathcal D^{-}|=n-1\ge b$.
Draw $\mathcal B$ uniformly from the size-$b$ subsets of $\mathcal D$ and
construct
\begin{equation}
\label{eq:ns_replacement_batch}
\mathcal B'
=
\begin{cases}
\mathcal B,
&x^{-}\notin\mathcal B,\\
(\mathcal B\setminus\{x^{-}\})\cup\{Y\},
&x^{-}\in\mathcal B,
\end{cases}
\end{equation}
where, on the second branch, $Y$ is uniform on
$\mathcal D^{-}\setminus(\mathcal B\setminus\{x^{-}\})$, independently of
the seed sets and Gaussian noise. Then $\mathcal B'$ is uniform over the
size-$b$ subsets of $\mathcal D^{-}$, conditional on either branch and
unconditionally. Keeping the seed sets of retained samples and assigning
the seed set of $x^{-}$ to $Y$ gives the joint mini-batch--seed law of a
retained-data local step.
\end{lemma}

\begin{proof}[Proof of Lemma~\ref{lem:ns_replacement_batch}]
Conditional on $x^{-}\notin\mathcal B$, the original mini-batch is already a
uniform size-$b$ subset of $\mathcal D^{-}$. Conditional on a hit, fix a
size-$b$ subset $A\subseteq\mathcal D^{-}$. For each $y\in A$, the output is
$A$ when $\mathcal B\setminus\{x^{-}\}=A\setminus\{y\}$ and $Y=y$. Hence
\begin{equation}
\label{eq:ns_set_batch_probability}
\mathbb P(\mathcal B'=A\mid x^{-}\in\mathcal B)
=
\frac{b}{\binom{n-1}{b-1}(n-b)}
=
\frac1{\binom{n-1}{b}}.
\end{equation}
The two branches therefore have the same retained mini-batch law.

Conditional on any original batch and replacement, the seed sets assigned
to the $b$ resulting samples remain independent, each with the uniform
size-$q$ law on $\mathbb S$. Indeed, replacement only relabels these seed
sets and is independent of their values. This conditional law depends on
neither the original batch nor the replacement. Averaging over their
preimages proves the joint mini-batch--seed claim.
\end{proof}

The mini-batch itself remains an unordered set. For the coupling argument, we only need to match the coefficients associated with corresponding sample--seed pairs in the two executions. If the deleted sample is not selected, the two batches and all their seed sets are identical. If it is selected, the replacement sample inherits its seed set, while every retained sample keeps its original seeds. We can therefore compare the two clean coefficient vectors after placing corresponding sample--seed pairs in the same deterministic order. This ordering is only a representation: permuting the pairs together with their labels leaves both the Gaussian law and the accumulated seed-coordinate update unchanged.

Under this correspondence, only the $q$ coefficients that belonged to the deleted sample can change. Since each clipped and normalized clean coefficient lies in $[-C/(bq),C/(bq)]$, the two clean vectors satisfy
\begin{equation}
\label{eq:ns_scalar_sensitivity}
\|\boldsymbol h-\boldsymbol h'\|_2
\le
\frac{2C}{b\sqrt q}.
\end{equation}
Here $\boldsymbol h'$ denotes the retained-data coefficients written in the matched ordering. This is the only sensitivity bound needed below.

For output stability, it is enough to follow the resulting seed-coordinate increment, because the local model and accumulator updates are deterministic once that increment is fixed. The saved sample identities and sample--seed correspondence matter for exact unlearning, but they do not affect the distributional comparison of the model updates themselves.

\begin{lemma}[One-step divergence after mini-batch subsampling]
\label{lem:ns_one_step_divergence}
Fix a common pre-update state and suppose the affected client participates.
Let $n$ be its local data size before deletion. Let $P$ and $Q$ be the
conditional laws of the resulting seed-coordinate increment on
$\mathcal D_c$ and $\mathcal D_c^{-}$, respectively. Then
\begin{equation}
\label{eq:ns_one_step_divergence}
D_{\mathrm{KL}}(P\|Q)
\le
D_2(P\|Q)
\le
\varepsilon_{\mathrm{step}}(n,\sigma_\rho),
\end{equation}
where
\begin{equation}
\label{eq:ns_step_epsilon}
\varepsilon_{\mathrm{step}}(n,\sigma)
=
\log\left[
1+
\left(\frac bn\right)^2
\left(
\exp\left(\frac{4C^2/(b^2q)}{\sigma^2}\right)-1
\right)
\right].
\end{equation}
\end{lemma}

\begin{proof}[Proof of Lemma~\ref{lem:ns_one_step_divergence}]
Let $\beta=b/n$, and let $P_1$ denote the law of the seed-coordinate
increment conditional on $x^{-}$ belonging to the original mini-batch.
Conditional on a non-hit, the batch and its sample-specific seed sets have
the retained-data law. Since the increment does not depend on how those
pairs are enumerated,
\begin{equation}
\label{eq:ns_step_mixture}
P=(1-\beta)Q+\beta P_1.
\end{equation}

On a hit, couple the batches by Lemma~\ref{lem:ns_replacement_batch}.
Let $\lambda$ collect the original batch, its replacement $Y$, and all
sample-specific seed sets, and let $\pi$ be their joint law conditional on
the hit. Given $\lambda$, write both coefficient vectors in the original
pair coordinates. Their Gaussian laws are
\[
G_\lambda=\mathcal N(\boldsymbol h_\lambda,\sigma_\rho^2I_{bq}),
\qquad
G'_\lambda=\mathcal N(\boldsymbol h'_{\lambda},
\sigma_\rho^2I_{bq}).
\]
Define augmented measures
$\overline P_1(d\lambda,dv)=\pi(d\lambda)G_\lambda(dv)$ and
$\overline Q(d\lambda,dv)=\pi(d\lambda)G'_\lambda(dv)$.
They share the same latent measure. Conditional Gaussian integration gives
\begin{align}
\chi^2(\overline P_1\|\overline Q)
&=
\mathbb E_{\lambda\sim\pi}
\left[
\exp\!\left(
\frac{\|\boldsymbol h_\lambda-
\boldsymbol h'_{\lambda}\|_2^2}{\sigma_\rho^2}
\right)-1
\right]
\notag\\
&\le
\exp\!\left(\frac{4C^2/(b^2q)}{\sigma_\rho^2}\right)-1.
\label{eq:ns_augmented_hit_chi_square}
\end{align}

Map both coefficient vectors to their seed-coordinate increments by summing
coefficients that use the same seed. Under the matched sample--seed
correspondence, the same labeled coefficients are summed on the two sides;
only their values may differ. Lemma~\ref{lem:ns_replacement_batch} gives the
correct retained batch--seed law, so the second pushforward is exactly $Q$.
Data processing therefore yields
\begin{equation}
\label{eq:ns_hit_chi_square}
\chi^2(P_1\|Q)
\le
\exp\!\left(\frac{4C^2/(b^2q)}{\sigma_\rho^2}\right)-1.
\end{equation}

Writing $u=dP_1/dQ$, Eq.~\eqref{eq:ns_step_mixture} implies
\begin{align}
\exp(D_2(P\|Q))
&=
\int((1-\beta)+\beta u)^2dQ
\notag\\
&=
1+\beta^2\chi^2(P_1\|Q).
\end{align}
Combining this identity with Eq.~\eqref{eq:ns_hit_chi_square} and
$D_{\mathrm{KL}}\le D_2$ proves the claim. The argument concerns the actual seed-coordinate update; no distributional comparison of an arbitrary storage ordering is needed.
\end{proof}

\begin{theorem}[Output TV stability of FedZOO-NS]
\label{thm:fedzoons_output_tv}
Consider an adjacent pair $\mathcal D$ and $\mathcal D^{-}$ differing by the
deletion of $x^{-}\in\mathcal D_{c^{-}}$, and let
$n_{c^{-}}=|\mathcal D_{c^{-}}|$ before deletion. Conditionally on the fixed
seed pool,
\begin{align}
&\operatorname{TV}\left(
\mathcal O_{\mathsf{NS},\mathrm{out}}(\mathcal D),
\mathcal O_{\mathsf{NS},\mathrm{out}}(\mathcal D^{-})
\right)
\notag\\
&\quad\le
\min\left\{
1,
\sqrt{
\frac{Tm}{2M}
\varepsilon_{\mathrm{step}}
(n_{c^{-}},\sigma_\rho)
}
\right\}.
\label{eq:fedzoons_output_tv_exact}
\end{align}
The same bound holds unconditionally over the random seed pool.
\end{theorem}

\begin{proof}[Proof of Theorem~\ref{thm:fedzoons_output_tv}]
For the stability proof, condition on the sampled client sets and the seed-coordinate increments generated before the current local step. Mini-batch identities and the storage order of sample--seed coefficients need not be retained in this conditioning. Given
$\boldsymbol Z$, the initial state, and the preceding increments, the current
model is deterministic, so Lemma~\ref{lem:ns_one_step_divergence} applies to the next update from the same pre-update model in the two executions.

Fix this history before round $r$. The client set $\mathcal C^r$ has the
same data-independent law in both executions. If
$c^{-}\notin\mathcal C^r$, all conditional transitions in this round are
identical. If $c^{-}\in\mathcal C^r$, the KL chain rule and
Lemma~\ref{lem:ns_one_step_divergence} give
\begin{equation}
\label{eq:ns_selected_round_kl}
D_{\mathrm{KL}}(P_r\|Q_r
\mid\mathcal C^r,\mathcal H_{r-1})
\le
E\varepsilon_{\mathrm{step}}
(n_{c^{-}},\sigma_\rho).
\end{equation}
Averaging over the common client-set law yields
\begin{equation}
\label{eq:ns_round_kl}
D_{\mathrm{KL}}(P_r\|Q_r\mid\mathcal H_{r-1})
\le
\frac mM
E\varepsilon_{\mathrm{step}}
(n_{c^{-}},\sigma_\rho).
\end{equation}
Applying the KL chain rule over the $R$ rounds gives
\begin{equation}
\label{eq:ns_total_kl}
D_{\mathrm{KL}}(P_{1:R}\|Q_{1:R})
\le
\frac{Tm}{M}
\varepsilon_{\mathrm{step}}
(n_{c^{-}},\sigma_\rho).
\end{equation}
Data processing to the final model and Pinsker's inequality prove
\eqref{eq:fedzoons_output_tv_exact}. The bound is uniform in
$\boldsymbol Z$, so averaging over the pool gives the unconditional claim.
\end{proof}

\begin{corollary}[Noise level for a target stability $\rho$]
\label{cor:fedzoons_noise_calibration}
Fix $0<\rho\le1$ before training and choose the standard deviation of every
normalized scalar-noise coordinate as
\begin{equation}
\label{eq:fedzoons_sigma_calibration}
\sigma_\rho
=
\max\left\{
\frac{2C}{b\sqrt q},
\frac{2C}{\underline n\sqrt q\,\rho}
\sqrt{\frac{Tm}{M}}
\right\}.
\end{equation}
Then \textsc{FedZOO-NS} is $\rho$-sample-level TV stable on the dataset class
satisfying~\eqref{eq:stream_size_lower_bound}. More precisely,
\begin{align}
&\operatorname{TV}\left(
\mathcal O_{\mathsf{NS},\mathrm{out}}(\mathcal D),
\mathcal O_{\mathsf{NS},\mathrm{out}}(\mathcal D^{-})
\right)
\notag\\
&\quad\le
\frac{2C}{n_{c^{-}}\sqrt q\,\sigma_\rho}
\sqrt{\frac{Tm}{M}}
\le
\rho.
\label{eq:fedzoons_simple_tv}
\end{align}
\end{corollary}

\begin{proof}[Proof of Corollary~\ref{cor:fedzoons_noise_calibration}]
The first term in~\eqref{eq:fedzoons_sigma_calibration} ensures that the squared mean shift between the two Gaussian coefficient vectors is at most their noise variance. Using
$e^x-1\le2x$ for $0\le x\le1$ and $\log(1+x)\le x$ gives
\begin{equation}
\varepsilon_{\mathrm{step}}(n,\sigma_\rho)
\le
\frac{8C^2}{n^2q\sigma_\rho^2}.
\end{equation}
Substitution into Theorem~\ref{thm:fedzoons_output_tv} yields the first
inequality in~\eqref{eq:fedzoons_simple_tv}. The second term in
\eqref{eq:fedzoons_sigma_calibration}, together with
$n_{c^{-}}\ge\underline n$, gives the target bound.
\end{proof}

The NS bound reflects two events. A deletion can matter only when the affected client participates and the deleted sample enters its mini-batch. Even then, the two noisy updates may still coincide under the Gaussian coupling; this becomes more likely as the noise level grows relative to the possible change in the clean coefficients. Corollary~\ref{cor:fedzoons_noise_calibration} chooses the noise so that these effects accumulate to at most the target $\rho$.

One detail in this bound is worth making explicit. The target client is
sampled once per round, and that client set is reused for all $E$ local steps.
The factor $m/M$ therefore appears once for the whole round, rather than as an
independent subsampling factor at every local position. This is why the
sufficient noise scale depends on $\sqrt{m/M}$.

TV stability concerns only the distribution of the output. Exact unlearning is stronger: the saved state after a deletion must itself have the law of a fresh run on the retained data. For \textsc{FedZOO-NS}, the Gaussian noise lets us enforce this transition by transition. At an affected local step, we first try to keep the previously generated noisy coefficients through rejection sampling. If that fails, a reflection maps them to the correct retained-data Gaussian law, after which training continues on the retained data. This is the same coupling principle used for stability-based exact unlearning in \cite{ullah2021machine}, adapted here to the sample--seed coefficients and accumulators maintained by the federated algorithm.

\subsection{FedZOUL-NS Exact Unlearning}
\label{sec:fedzoulns_exactness}

\begin{lemma}[Gaussian maximal coupling by hyperplane reflection]
\label{lem:gaussian_hyperplane_coupling}
Let
$P=\mathcal N(\boldsymbol\mu,\sigma^2I_p)$ and
$Q=\mathcal N(\boldsymbol\nu,\sigma^2I_p)$, with densities $p$ and $q$.
For $\boldsymbol\mu\ne\boldsymbol\nu$, define
\begin{align}
\boldsymbol m
&=
\frac{\boldsymbol\mu+\boldsymbol\nu}{2},
&
\boldsymbol e
&=
\frac{\boldsymbol\nu-\boldsymbol\mu}
{\|\boldsymbol\nu-\boldsymbol\mu\|_2},
\notag\\
\mathsf R_{\boldsymbol\mu,\boldsymbol\nu}(\boldsymbol x)
&=
\boldsymbol x
-2\langle\boldsymbol x-\boldsymbol m,\boldsymbol e\rangle
\boldsymbol e.
\label{eq:general_hyperplane_reflection}
\end{align}
Draw $X\sim P$ and $U\sim\operatorname{Unif}(0,1)$ independently, and set
\begin{equation}
\label{eq:gaussian_hyperplane_coupling}
Y
=
\begin{cases}
X,
&U\le\min\{1,q(X)/p(X)\},\\
\mathsf R_{\boldsymbol\mu,\boldsymbol\nu}(X),
&\text{otherwise}.
\end{cases}
\end{equation}
Then $Y\sim Q$, and
\begin{align}
\mathbb P(X\ne Y)
&=
\operatorname{TV}(P,Q)
\notag\\
&=
2\Phi\left(
\frac{\|\boldsymbol\mu-\boldsymbol\nu\|_2}{2\sigma}
\right)-1
\notag\\
&\le
\frac{\|\boldsymbol\mu-\boldsymbol\nu\|_2}
{\sqrt{2\pi}\sigma},
\label{eq:gaussian_coupling_rejection}
\end{align}
where $\Phi$ is the standard Gaussian CDF. If
$\boldsymbol\mu=\boldsymbol\nu$, set $Y=X$.
\end{lemma}

\begin{proof}[Proof of Lemma~\ref{lem:gaussian_hyperplane_coupling}]
The map $\mathsf R_{\boldsymbol\mu,\boldsymbol\nu}$ is an involutive
isometry with absolute Jacobian determinant one and exchanges the two means.
Thus,
\begin{equation}
\label{eq:reflection_swaps_densities}
p(\mathsf R_{\boldsymbol\mu,\boldsymbol\nu}(\boldsymbol x))
=
q(\boldsymbol x),
\qquad
q(\mathsf R_{\boldsymbol\mu,\boldsymbol\nu}(\boldsymbol x))
=
p(\boldsymbol x).
\end{equation}
For any measurable set $A$, the accepted branch contributes
$\int_A\min\{p,q\}$, while a change of variables in the reflected branch
shows that it contributes $\int_A(q-p)_+$. Their sum is $Q(A)$. The
rejection probability is $\int(p-q)_+=\operatorname{TV}(P,Q)$.
For distinct means, $p(\boldsymbol x)>q(\boldsymbol x)$ precisely when
$\langle\boldsymbol x-\boldsymbol m,\boldsymbol e\rangle<0$.
Integrating the two one-dimensional Gaussian projections on this halfspace
gives the CDF formula. The final inequality follows from
$\Phi(u)-1/2\le u/\sqrt{2\pi}$. Importantly, the target law $Q$ is obtained
by combining the acceptance and reflection branches. Conditional on
rejection, the reflected value instead has density
$(q-p)_+/\operatorname{TV}(P,Q)$ when this denominator is positive; it is
not an independent draw from the full $Q$.
\end{proof}

For exact unlearning, the saved state also includes sample identities and seed assignments. The correspondence described above is enough to compare the two executions: retained samples keep their seeds, and the replacement sample inherits the deleted sample's seeds. Whenever two coefficient vectors are compared, corresponding sample--seed pairs are written in the same deterministic order. Changing this order together with the labels does not change either the Gaussian law or the seed-coordinate update.

\begin{lemma}[One-step exact transition for FedZOO-NS]
\label{lem:fedzoulns_one_position}
Fix the seed pool, a local position $(r,t)$, the participating client set, and common server and client optimization variables immediately before this position. Suppose the two datasets differ only by deleting $x^{-}\in\mathcal D_{c^{-}}$. Couple the affected mini-batch as in Lemma~\ref{lem:ns_replacement_batch}, keep the corresponding seed assignments, and apply the rejection--reflection rule of Eqs.~\eqref{eq:rr_acceptance}--\eqref{eq:rr_reflection} when the affected client uses the deleted sample. Then the resulting mini-batches, seed assignments, noisy coefficients, and post-update variables have exactly the conditional law of one \textsc{FedZOO-NS} step on the retained data. If the rejection test accepts, the local model and accumulator remain unchanged.
\end{lemma}

\begin{proof}[Proof of Lemma~\ref{lem:fedzoulns_one_position}]
All unaffected clients can use exactly the same transition. For the affected client, condition on the coupled original and retained mini-batches and their seed assignments. In the matched ordering, the original noisy coefficient vector has law
\[
\mathcal N(\boldsymbol h,\sigma_\rho^2I_{bq}),
\]
while the retained-data step requires
\[
\mathcal N(\boldsymbol h',\sigma_\rho^2I_{bq}).
\]
Lemma~\ref{lem:gaussian_hyperplane_coupling} shows that the acceptance and reflection branches together produce exactly this second Gaussian law. Reordering the transported coefficients together with their sample and seed labels, if needed by the deterministic storage convention, does not change that law.

Lemma~\ref{lem:ns_replacement_batch} already gives the correct retained mini-batch and seed-assignment distribution. Integrating over these coupled sampling variables therefore yields the correct joint law of the saved local information. The accumulator and model are deterministic functions of the noisy coefficients and the current optimization variables, so their post-update law is also correct. On acceptance, the coefficient values are unchanged under the sample--seed correspondence; because the same values are summed for each seed, the seed-coordinate increment, and hence the model and accumulator endpoints, are unchanged as well.
\end{proof}

Let $\mathcal U_{\mathsf{NS}}^{[T_{ul}]}$ denote the complete state returned by \textsc{FedZOUL-NS} at $T_{ul}$ after processing the deletion, and let
$\mathcal U_{\mathsf{NS}}^{T_{ul}\to T}$ include completion of the horizon.

\begin{theorem}[Exact unlearning of FedZOUL-NS]
\label{thm:fedzoulns_exact}
For every valid single-sample deletion
$x^{-}\in\mathcal D_{c^{-}}$ arriving after local position $T_{ul}$,
\begin{equation}
\label{eq:fedzoulns_prefix_exact}
\mathcal U_{\mathsf{NS}}^{[T_{ul}]}
\left(
\mathcal O_{\mathsf{NS}}^{[T_{ul}]}(\mathcal D),
\{x^{-}\}
\right)
\overset{d}{=}
\mathcal O_{\mathsf{NS}}^{[T_{ul}]}(\mathcal D^{-}).
\end{equation}
Consequently,
\begin{equation}
\label{eq:fedzoulns_exact}
\mathcal U_{\mathsf{NS}}^{T_{ul}\to T}
\left(
\mathcal O_{\mathsf{NS}}^{[T_{ul}]}(\mathcal D),
\{x^{-}\}
\right)
\overset{d}{=}
\mathcal O_{\mathsf{NS}}^{[T]}(\mathcal D^{-}).
\end{equation}
Both equalities hold jointly over the reconstructed model, global and local
accumulators, sampled client sets, mini-batches, sample--seed
correspondences, noisy coefficient vectors, and every other maintained
state. The result holds for every $\sigma_\rho>0$.
\end{theorem}

\begin{proof}[Proof of Theorem~\ref{thm:fedzoulns_exact}]
Condition on the shared seed pool and process the executed training positions in chronological order. Before the first rejection, the two executions have the same optimization variables. At the next position, the original mini-batch, seed assignments, and Gaussian noise have their usual fresh distributions. Lemma~\ref{lem:fedzoulns_one_position} then turns this step into the correct retained-data transition while preserving the current model and accumulator whenever the rejection test accepts.

If the test accepts, only the saved sample identity may change: the noisy coefficient values are kept under the sample--seed correspondence, so the induced update is unchanged. If the first rejection occurs at $(r^\star,t^\star)$, the reflected coefficients give the correct retained-data update at that position. From $\mathsf{Next}(r^\star,t^\star)$ onward, the algorithm generates fresh randomness on the retained data. The already sampled client set is kept for the rest of the current round, as it would be in the original training procedure, and later rounds draw fresh client sets. When $t^\star=E$, the server first recomputes the round aggregation from the corrected local endpoints, exactly as specified in Algorithm~\ref{algo:FedZOULNS}.

Applying this argument successively over the positions up to $T_{ul}$ shows that the state returned after the deletion follows the same transition kernels, in the same order, as a fresh \textsc{FedZOO-NS} execution on $\mathcal D^{-}$. Hence it has law $\mathcal O_{\mathsf{NS}}^{[T_{ul}]}(\mathcal D^{-})$. Continuing with the same retained-data kernels proves Eq.~\eqref{eq:fedzoulns_exact}. Finally, averaging over the independently generated seed pool removes the conditioning.
\end{proof}

Define
\begin{equation}
\label{eq:fedzoulns_single_hit_rejection}
\vartheta_\rho
=
2\Phi\left(
\frac{2C/(b\sqrt q)}{2\sigma_\rho}
\right)-1.
\end{equation}
This bounds the rejection probability conditional on a common
pre-update model, coupled sampling variables, and a hit, before the current
Gaussian noise and acceptance uniform are exposed. It is not a bound
conditional on the realized noisy coefficient vector.

\begin{proposition}[Probability of replay]
\label{prop:fedzoulns_replay}
Let $\mathsf{Replay}_{\mathsf{NS}}$ denote the event that at least one
verification test rejects. A rejection at $T_{ul}$ is included even though there is no later position to rerun: the endpoint at that position must still be corrected.
For a deletion arriving at $T_{ul}$,
\begin{align}
\mathbb P(\mathsf{Replay}_{\mathsf{NS}})
&\le
\Gamma_{\mathsf{NS}}(T_{ul},c^{-})
\notag\\
&:=
\min\left\{
1,
\frac{T_{ul}mb}{Mn_{c^{-}}}\vartheta_\rho
\right\}
\label{eq:fedzoulns_replay_general}
\\
&\le
\min\left\{
1,
\frac{2CT_{ul}m}
{\sqrt{2\pi}Mn_{c^{-}}\sqrt q\,\sigma_\rho}
\right\}.
\label{eq:fedzoulns_replay_sensitivity}
\end{align}
Under~\eqref{eq:fedzoons_sigma_calibration},
\begin{align}
\mathbb P(\mathsf{Replay}_{\mathsf{NS}})
&\le
\min\left\{
1,
\frac{\rho T_{ul}}{\sqrt{2\pi T}}
\sqrt{\frac mM}
\right\}
\notag\\
&\le
\min\left\{
1,
\rho\sqrt{\frac{Tm}{2\pi M}}
\right\}.
\label{eq:fedzoulns_replay_calibrated}
\end{align}
\end{proposition}

\begin{proof}[Proof of Proposition~\ref{prop:fedzoulns_replay}]
Use the causal joint construction in the proof of
Theorem~\ref{thm:fedzoulns_exact}. Given the paired past before a local
position and its current client set, the original batch and noise have
the prescribed fresh laws. If no earlier rejection has occurred and the
affected client participates, the probability of a first rejection at this
position is at most $(b/n_{c^{-}})\vartheta_\rho$, by
Lemma~\ref{lem:gaussian_hyperplane_coupling} and
Eq.~\eqref{eq:ns_scalar_sensitivity}. Otherwise, a first rejection at this
position is impossible. Dropping the no-earlier-rejection indicator and
averaging over client participation bounds this probability by
$(m/M)(b/n_{c^{-}})\vartheta_\rho$. Summing over the first $T_{ul}$ positions
proves Eq.~\eqref{eq:fedzoulns_replay_general}. This calculation uses the marginal law at each position and does not condition on a particular saved training state. Equation
\eqref{eq:gaussian_coupling_rejection} and
$\|\boldsymbol h-\boldsymbol h'\|_2\le 2C/(b\sqrt q)$ prove
\eqref{eq:fedzoulns_replay_sensitivity}. Substituting
\eqref{eq:fedzoons_sigma_calibration}, using
$n_{c^{-}}\ge\underline n$, and then $T_{ul}\le T$ gives
\eqref{eq:fedzoulns_replay_calibrated}.
\end{proof}

\begin{corollary}[Sequential deletion requests for FedZOUL-NS]
\label{cor:fedzoulns_sequential}
Consider the valid non-adaptive single-sample request stream in
Eq.~\eqref{eq:nonadaptive_request_stream}. After every request $j$, the
state returned at $T_{ul}^{(j)}$ has law
$\mathcal O_{\mathsf{NS}}^{[T_{ul}^{(j)}]}(\mathcal D^{(j+1)})$.
In $\Gamma_{\mathsf{NS}}(T_{ul}^{(j)},c_j^{-})$, evaluate the affected
client's pre-request data size as $n_j=|\mathcal D_{c_j^{-}}^{(j)}|$.
If $N_{\mathsf{NS}}^{\mathrm{rep}}$ counts requests for which the rejection test fails, including a rejection at the request endpoint, then
\begin{align}
\mathbb E[N_{\mathsf{NS}}^{\mathrm{rep}}]
&\le
\sum_{j=1}^{k}
\Gamma_{\mathsf{NS}}(T_{ul}^{(j)},c_j^{-}),
\label{eq:fedzoulns_expected_replays}
\\
\mathbb P(N_{\mathsf{NS}}^{\mathrm{rep}}\ge1)
&\le
\min\left\{
1,
\sum_{j=1}^{k}
\Gamma_{\mathsf{NS}}(T_{ul}^{(j)},c_j^{-})
\right\}.
\label{eq:fedzoulns_any_replay}
\end{align}
Under~\eqref{eq:stream_size_lower_bound} and
\eqref{eq:fedzoons_sigma_calibration}, define
\begin{equation}
\overline\Gamma_{\mathsf{NS}}
=
\min\left\{
1,
\rho\sqrt{\frac{Tm}{2\pi M}}
\right\}.
\end{equation}
Then the two bounds are at most
$k\overline\Gamma_{\mathsf{NS}}$ and
$\min\{1,k\overline\Gamma_{\mathsf{NS}}\}$, respectively.
\end{corollary}

\begin{proof}[Proof of Corollary~\ref{cor:fedzoulns_sequential}]
Fix the external request stream and the perturbation pool. Use the same
induction as in Corollary~\ref{cor:fedzouls_sequential}, with
Theorem~\ref{thm:fedzoulns_exact} providing the single-request transition. If $X_j^{-}$ and
$X_j^{+}$ denote the complete states immediately before and after processing request $j$,
then
\begin{align}
X_j^{-}&\overset d=
\mathcal O_{\mathsf{NS}}^{[T_{ul}^{(j)}]}(\mathcal D^{(j)}),
\notag\\
X_j^{+}&\overset d=
\mathcal O_{\mathsf{NS}}^{[T_{ul}^{(j)}]}(\mathcal D^{(j+1)}).
\label{eq:fedzoulns_stream_state_laws}
\end{align}
Fresh retained-data training between successive arrival positions preserves
these marginal laws. Each request uses fresh replacement draws and acceptance uniforms. The induction is on the marginal law of the current state; it does not condition on the realized saved trajectory or on earlier acceptance outcomes.

Let $I_j$ indicate the first-rejection event for request $j$.
Proposition~\ref{prop:fedzoulns_replay}, applied to the pre-request marginal,
gives
\[
\mathbb E[I_j]=\mathbb P(I_j=1)
\le\Gamma_{\mathsf{NS}}(T_{ul}^{(j)},c_j^{-}).
\]
Linearity of expectation yields
Eq.~\eqref{eq:fedzoulns_expected_replays}, and a union bound yields
Eq.~\eqref{eq:fedzoulns_any_replay}. The indicators need not be independent. Averaging over the perturbation pool preserves the
bounds, and the calibrated claims follow from
Eq.~\eqref{eq:fedzoulns_replay_calibrated}.
\end{proof}

It is useful to distinguish the final-model stability guarantee from the
probability that training must be rerun. The former is bounded by $\rho$.
The latter comes from the sequential coupling of local updates and is bounded
by a quantity of order $\rho\sqrt{Tm/M}$, capped at one. The two need not
coincide because \textsc{FedZOUL-NS} couples the training trajectory one local
transition at a time rather than maximally coupling the final-model laws. The
resulting replay bound is nevertheless no larger than the corresponding
subsampling bound because $\vartheta_\rho\le1$.

Exactness tells us which state must be produced after a deletion. The next
question is how much of the computation that has already been performed must be
done again. We count only this additional work up to the request position;
ordinary training after the request is not included. For \textsc{FedZOUL-S},
rerunning starts from the first position that used the deleted sample. For
\textsc{FedZOUL-NS}, affected updates are first checked through the Gaussian
coupling, and later training is rerun only if one of these checks rejects.

\subsection{Unlearning Cost}
\label{sec:expected_unlearning_cost}

We count loss evaluations, model reconstruction, and communication
separately. A loss evaluation is one call to $\ell(\boldsymbol w;x)$ for
one sample. One scalar central difference therefore uses two evaluations, so rerunning one
local position across all $m$ participating clients uses $2mbq$ loss
evaluations. Rerunning all $T_{ul}$ executed local positions uses
\begin{equation}
\label{eq:full_prefix_query_cost}
Q_{\mathrm{full}}(T_{ul})=2mbqT_{ul}.
\end{equation}
These are aggregate oracle counts. They do not include restoring models, looking up saved coefficients, drawing replacement samples, or transferring checkpoints.

For \textsc{FedZOUL-NS}, checking an affected position does not require
recomputing all $bq$ coefficients. Under the sample--seed correspondence, only
the $q$ pairs attached to the deleted sample can change. Let $J$ denote these
coordinates and write $\boldsymbol d=\boldsymbol h'-\boldsymbol h$. Then
$\boldsymbol d$ is supported on $J$, and for the saved noisy vector
$\widetilde{\boldsymbol h}$ with
$\boldsymbol m=(\boldsymbol h+\boldsymbol h')/2$,
\begin{equation}
\label{eq:sparse_likelihood_ratio}
\log\frac{q(\widetilde{\boldsymbol h})}{p(\widetilde{\boldsymbol h})}
=
\frac{\langle
\widetilde{\boldsymbol h}_{J}-\boldsymbol m_{J},
\boldsymbol d_{J}\rangle}{\sigma_\rho^2}.
\end{equation}
The reflection direction is supported on the same coordinates. Hence the
acceptance test and, when needed, the reflection use only the $q$ original and
$q$ replacement central differences, for at most $4q$ loss evaluations and
$O(q)$ coefficient operations once the pre-update model and saved coefficients
are available. The corresponding correction to the local accumulator is
obtained by summing the coefficient changes by seed; unchanged retained pairs
need no new loss evaluation. A literal implementation that recomputes both full
coefficient vectors uses up to $4bq$ evaluations instead, but the two
implementations are algebraically equivalent.

Let $\tau_{\mathsf S}^{\star}$ be the first executed position at which the deleted sample is used, with
$\tau_{\mathsf S}^{\star}=\infty$ if no such position exists. Define
\begin{equation}
\label{eq:fedzouls_replayed_positions}
L_{\mathsf S}
=
(T_{ul}-\tau_{\mathsf S}^{\star}+1)
\mathbf1\{\tau_{\mathsf S}^{\star}\le T_{ul}\}.
\end{equation}
For \textsc{FedZOUL-NS}, let $\tau_{\mathsf{NS}}^{\star}$ be the first
rejected affected position and set
\begin{equation}
\label{eq:fedzoulns_replayed_positions}
L_{\mathsf{NS}}
=
(T_{ul}-\tau_{\mathsf{NS}}^{\star})
\mathbf1\{\tau_{\mathsf{NS}}^{\star}\le T_{ul}\},
\end{equation}
because reflection already produces the correct retained-data update at the rejected position, so replay begins at
its successor. Both lengths are defined to be zero when their trigger is
absent. Let $N_{\mathrm{ver}}$ count the affected local updates actually checked before the procedure stops, including the first rejected update when one exists. A rejection at the last executed position can trigger the NS correction even though no later position needs to be rerun. Proposition~\ref{prop:fedzoulns_replay} bounds the probability of this event.

\begin{theorem}[Expected zeroth-order cost of one deletion]
\label{thm:expected_unlearning_queries}
For a fixed single-sample deletion at $1\le T_{ul}\le T$, using the sparse check above,
\begin{align}
\mathbb E[L_{\mathsf S}]
&\le
\sum_{u=1}^{T_{ul}}
\Lambda_{\mathsf S}(u,c^{-})
\notag\\
&\le
\min\left\{
T_{ul},
\frac{\frac{mbT_{ul}}{Mn_{c^{-}}}(T_{ul}+1)}{2}
\right\},
\label{eq:expected_s_replayed_positions}
\end{align}
and the aggregate loss-evaluation cost of \textsc{FedZOUL-S} satisfies
\begin{align}
\mathbb E[Q_{\mathsf S}]
&=
2mbq\,\mathbb E[L_{\mathsf S}]
\notag\\
&\le
Q_{\mathrm{full}}(T_{ul})
\Lambda_{\mathsf S}(T_{ul},c^{-}).
\label{eq:expected_s_query_cost}
\end{align}
Under~\eqref{eq:fedzoos_subsampling_calibration},
\begin{equation}
\label{eq:expected_s_query_fraction}
\mathbb E[Q_{\mathsf S}]
\le
\rho Q_{\mathrm{full}}(T_{ul}).
\end{equation}

For \textsc{FedZOUL-NS},
\begin{align}
\mathbb E[N_{\mathrm{ver}}]
&\le
\frac{mbT_{ul}}{Mn_{c^{-}}},
\label{eq:expected_ns_verifications}
\\
\mathbb E[L_{\mathsf{NS}}]
&\le
\sum_{u=1}^{T_{ul}-1}
\Gamma_{\mathsf{NS}}(u,c^{-})
\notag\\
&\le
\min\left\{
T_{ul}-1,
\frac{\frac{mb\vartheta_\rho}{Mn_{c^{-}}}T_{ul}(T_{ul}-1)}{2}
\right\}.
\label{eq:expected_ns_replayed_positions}
\end{align}
Its aggregate loss-evaluation cost satisfies
\begin{align}
\mathbb E[Q_{\mathsf{NS}}]
&\le
4q\,\frac{mbT_{ul}}{Mn_{c^{-}}}
+
2mbq
\sum_{u=1}^{T_{ul}-1}
\Gamma_{\mathsf{NS}}(u,c^{-})
\label{eq:expected_ns_query_cost}
\\
&\le
Q_{\mathrm{full}}(T_{ul})
\left[
\frac{2}{Mn_{c^{-}}}
+
\Gamma_{\mathsf{NS}}(T_{ul},c^{-})
\right].
\label{eq:expected_ns_query_fraction_general}
\end{align}
Under~\eqref{eq:fedzoons_sigma_calibration},
\begin{equation}
\frac{\mathbb E[Q_{\mathsf{NS}}]}
{Q_{\mathrm{full}}(T_{ul})}
\le
\frac{2}{Mn_{c^{-}}}
+
\min\left\{
1,
\frac{\rho T_{ul}}{\sqrt{2\pi T}}
\sqrt{\frac mM}
\right\}.
\label{eq:expected_ns_query_fraction_calibrated}
\end{equation}
Finally, when comparing one request under the same realized
client and mini-batch sampling schedule,
\begin{equation}
\label{eq:ns_no_more_replay_than_s}
L_{\mathsf{NS}}
\le
L_{\mathsf S}
\qquad\text{pathwise}.
\end{equation}
This compares replay lengths, not total query counts or running times.
For $T_{ul}=0$, all query and replay counts are zero; the ratios above are
used only for $T_{ul}>0$.
\end{theorem}

\begin{proof}[Proof of Theorem~\ref{thm:expected_unlearning_queries}]
For \textsc{FedZOUL-S},
\begin{equation}
L_{\mathsf S}
=
\sum_{u=1}^{T_{ul}}
\mathbf1\{\tau_{\mathsf S}^{\star}\le u\}.
\end{equation}
Equation~\eqref{eq:fedzoos_state_disagreement} gives
$\mathbb P(\tau_{\mathsf S}^{\star}\le u)
\le\Lambda_{\mathsf S}(u,c^{-})$.
Taking expectations proves
\eqref{eq:expected_s_replayed_positions}. Also,
$L_{\mathsf S}\le T_{ul}\mathbf1\{\mathsf{Replay}_{\mathsf S}\}$,
which gives~\eqref{eq:expected_s_query_cost} and then
\eqref{eq:expected_s_query_fraction}.

For \textsc{FedZOUL-NS}, verification occurs only when the target client is
selected and the target sample appears in its mini-batch. The expected number
of such hits is $\frac{mbT_{ul}}{Mn_{c^{-}}}$, and stopping after the first rejection can
only reduce it. The sparse calculation above gives the verification
cost and gives the pathwise accounting inequality
\begin{equation}
Q_{\mathsf{NS}}
\le 4qN_{\mathrm{ver}}+2mbqL_{\mathsf{NS}}.
\label{eq:ns_pathwise_query_accounting}
\end{equation}
The reflected position requires no further loss evaluations. Moreover,
\begin{equation}
L_{\mathsf{NS}}
=
\sum_{u=1}^{T_{ul}-1}
\mathbf1\{\tau_{\mathsf{NS}}^{\star}\le u\}.
\end{equation}
Proposition~\ref{prop:fedzoulns_replay} gives
$\mathbb P(\tau_{\mathsf{NS}}^{\star}\le u)
\le\Gamma_{\mathsf{NS}}(u,c^{-})$.
This proves~\eqref{eq:expected_ns_replayed_positions} and the first query
bound. Monotonicity of $\Gamma_{\mathsf{NS}}$ gives the second. The bound under the chosen noise level follows from~\eqref{eq:fedzoulns_replay_calibrated}.

Under a common sampling schedule, a rejection can occur only at a position
where the deleted sample is used. Hence the first rejected position is no
earlier than the first affected position, and replay after reflection starts
one position later. This proves the pathwise inequality.
\end{proof}

For \textsc{FedZOUL-S}, the expected loss-evaluation cost is controlled by the
chance that the deleted sample appeared before the request. Under the chosen
stability level, this is at most a $\rho$ fraction of rerunning all executed
positions. \textsc{FedZOUL-NS} adds a small verification cost, but Gaussian
acceptance can avoid rerunning an affected update and everything that follows
it. Equation~\eqref{eq:ns_no_more_replay_than_s} compares only the number of
local positions rerun; model reconstruction and communication are considered
next.

Communication also depends on where the first hit or rejection occurs, because only round boundaries reached after that point need to be processed again. Server aggregation occurs at positions $E,2E,\ldots$. Hence the expected number of completed aggregations repeated by \textsc{FedZOUL-S} is at most
\begin{equation}
\label{eq:expected_s_replayed_rounds}
\sum_{r=1}^{\lfloor T_{ul}/E\rfloor}
\Lambda_{\mathsf S}(rE,c^{-}),
\end{equation}
and for \textsc{FedZOUL-NS} it is at most
\begin{equation}
\label{eq:expected_ns_replayed_rounds}
\sum_{r=1}^{\lfloor T_{ul}/E\rfloor}
\Gamma_{\mathsf{NS}}(rE,c^{-}).
\end{equation}
If the request arrives partway through a round, that last partial round is also visited whenever the first hit or rejection has already occurred. Thus the expected number of visited round blocks is bounded by
\begin{align}
&\sum_{r=1}^{\lfloor T_{ul}/E\rfloor}
\Lambda_{\mathsf S}(rE,c^{-})
+\mathbf1\{E\nmid T_{ul}\}\Lambda_{\mathsf S}(T_{ul},c^{-})
\label{eq:expected_s_visited_rounds}
\end{align}
for \textsc{FedZOUL-S}, and by
\begin{align}
&\sum_{r=1}^{\lfloor T_{ul}/E\rfloor}
\Gamma_{\mathsf{NS}}(rE,c^{-})
+\mathbf1\{E\nmid T_{ul}\}\Gamma_{\mathsf{NS}}(T_{ul},c^{-})
\label{eq:expected_ns_visited_rounds}
\end{align}
for \textsc{FedZOUL-NS}. These bounds use only the probability that the first hit or rejection has occurred by the corresponding boundary; no independence across rounds is required.

These round counts determine the communication needed for the $K$-dimensional accumulator. The seed pool is already shared, the server stores the global accumulator checkpoints, and each client stores its own local state. With separate downlinks to the $m$ clients, every visited round block needs at most one $K$-scalar global-accumulator download per client, and every completed aggregation that is repeated needs at most one $K$-scalar endpoint upload per client. Therefore, the expected additional accumulator communication of \textsc{FedZOUL-S} is at most
\begin{align}
&mK\left[
2\sum_{r=1}^{\lfloor T_{ul}/E\rfloor}
\Lambda_{\mathsf S}(rE,c^{-})
+\mathbf1\{E\nmid T_{ul}\}\Lambda_{\mathsf S}(T_{ul},c^{-})
\right].
\label{eq:expected_s_accumulator_traffic}
\end{align}
For \textsc{FedZOUL-NS}, the corresponding bound is
\begin{align}
&mK\left[
2\sum_{r=1}^{\lfloor T_{ul}/E\rfloor}
\Gamma_{\mathsf{NS}}(rE,c^{-})
+\mathbf1\{E\nmid T_{ul}\}\Gamma_{\mathsf{NS}}(T_{ul},c^{-})
\right]
+K\frac{mbT_{ul}}{Mn_{c^{-}}}.
\label{eq:expected_ns_accumulator_traffic}
\end{align}
The last term appears because checking an affected NS update may require the global accumulator that was current before that update. One such checkpoint is enough to reconstruct the corresponding model,
\begin{equation}
\boldsymbol w_c^{r,t-1}
=\boldsymbol w^0-\eta\sum_{k=1}^{K}
\bigl(a_k^{r-1}+a_{c,k}^{r,t-1}\bigr)\boldsymbol z_k,
\end{equation}
and the expected number of affected updates checked is at most $\frac{mbT_{ul}}{Mn_{c^{-}}}$. If a common downlink broadcast is counted once rather than once per recipient, the factors multiplying the completed-round sums become $(m+1)K$ instead of $2mK$, while the partial-round download contributes only $K$. Sample identifiers, control messages, and protocol headers are not included because their size depends on the implementation.

These bounds also clarify why replay length alone does not determine communication. A rejection at a completed final round can require correcting that aggregation even when there is no later local step to rerun. Conversely, NS may check several affected updates without rejecting any of them, so checkpoint downloads can be positive even when no training step is rerun.

\begin{corollary}[Cost over a request stream]
\label{cor:stream_unlearning_cost}
Consider the valid non-adaptive single-sample request stream in Eq.~\eqref{eq:nonadaptive_request_stream}. Let $T_{ul}^{(j)}$ be the executed horizon when request $j$ arrives and let $n_j$ be the affected client's data size immediately before that request. Then the cumulative expected loss-evaluation costs satisfy
\begin{align}
\mathbb E[Q_{\mathsf S}^{1:k}]
&\le
2mbq\sum_{j=1}^{k}\sum_{u=1}^{T_{ul}^{(j)}}
\min\left\{1,\frac{umb}{Mn_j}\right\},
\label{eq:stream_s_query_cost}
\\
\mathbb E[Q_{\mathsf{NS}}^{1:k}]
&\le
4q\sum_{j=1}^{k}
\frac{mbT_{ul}^{(j)}}{Mn_j}
\notag\\
&\quad+
2mbq\sum_{j=1}^{k}\sum_{u=1}^{T_{ul}^{(j)}-1}
\min\left\{1,\frac{umb\vartheta_\rho}{Mn_j}\right\}.
\label{eq:stream_ns_query_cost}
\end{align}
Empty sums are interpreted as zero.

With separate downlinks, the cumulative expected accumulator communication of \textsc{FedZOUL-S} is at most
\begin{align}
&mK\sum_{j=1}^{k}
\left[
2\sum_{r=1}^{\lfloor T_{ul}^{(j)}/E\rfloor}
\min\left\{1,\frac{rEmb}{Mn_j}\right\}
\right.
\notag\\[-0.2em]
&\hspace{7em}\left.
+\mathbf1\{E\nmid T_{ul}^{(j)}\}
\min\left\{1,\frac{T_{ul}^{(j)}mb}{Mn_j}\right\}
\right],
\label{eq:stream_s_accumulator_cost}
\end{align}
while for \textsc{FedZOUL-NS} it is at most
\begin{align}
&mK\sum_{j=1}^{k}
\left[
2\sum_{r=1}^{\lfloor T_{ul}^{(j)}/E\rfloor}
\min\left\{1,\frac{rEmb\vartheta_\rho}{Mn_j}\right\}
\right.
\notag\\[-0.2em]
&\hspace{7em}\left.
+\mathbf1\{E\nmid T_{ul}^{(j)}\}
\min\left\{1,\frac{T_{ul}^{(j)}mb\vartheta_\rho}{Mn_j}\right\}
\right]
\notag\\
&\qquad+
K\sum_{j=1}^{k}\frac{mbT_{ul}^{(j)}}{Mn_j}.
\label{eq:stream_ns_accumulator_cost}
\end{align}
The corresponding broadcast bounds replace the factor $2mK$ on each completed-round sum by $(m+1)K$ and the factor $mK$ on each partial-round term by $K$.
\end{corollary}

\begin{proof}[Proof of Corollary~\ref{cor:stream_unlearning_cost}]
Fix the external request stream. By Corollaries~\ref{cor:fedzouls_sequential} and~\ref{cor:fedzoulns_sequential}, immediately before each request the current state has the same marginal law as the corresponding learning algorithm run on the data then present. The single-request bounds above therefore apply to every request with its current horizon $T_{ul}^{(j)}$ and affected-client size $n_j$. Summing these bounds and using linearity of expectation gives Eqs.~\eqref{eq:stream_s_query_cost}--\eqref{eq:stream_ns_accumulator_cost}. No independence across requests is needed. If the request stream itself is random but chosen independently of the algorithmic randomness, the same bounds hold after averaging over the stream.
\end{proof}

Loss evaluations and communication capture the dominant costs, but processing a deletion also entails ordinary arithmetic. Reconstructing one model from the $K$ shared directions costs $O(Kd)$ operations in a direct implementation. Rerunning one local step costs $O(bqd+K)$ beyond the loss evaluations, a sparse NS check costs $O(qd+q)$, and one server aggregation costs $O(mK)$. Combining these per-operation costs with the expected numbers of checked updates, rerun positions, visited round blocks, and repeated aggregations above gives the corresponding expected arithmetic work. The chronological search for the first affected mini-batch can require up to $O(bT_{ul})$ comparisons with a direct scan; an inverted index can reduce lookup time at the price of maintaining that index whenever stored mini-batches change.

These are aggregate work bounds rather than elapsed-time guarantees. Under ideal parallel execution, the forward-evaluation component alone is $2bq$ evaluations per rerun S step and $4q$ evaluations per checked NS update plus $2bq$ evaluations per rerun NS step. Model reconstruction, saved-state access, communication, and synchronization remain additional costs. In particular, the TV parameter $\rho$ directly bounds the expected S query fraction under the chosen subsampling level, whereas for NS the final-model TV guarantee and the transitionwise replay probability are different quantities. Consequently, neither the replay-length comparison nor the oracle-cost fraction by itself implies that one method has lower total unlearning time than the other.

We next ask whether the same finite seed pool that makes communication and
unlearning efficient still permits meaningful optimization. The objective being
evaluated is unchanged, but every update lies in the span of the shared
perturbation directions. We therefore first express the actual training
trajectory in these seed coordinates, prove convergence in the resulting
subspace, and then determine when that guarantee also controls the full
model-space gradient.

\subsection{Fixed-Seed Parameterization}
\label{sec:fixed_seed_analysis}

Every update is a linear combination of the shared perturbation directions, so the entire trajectory stays in the affine space generated by the seed pool. It is convenient to use the accumulated seed coefficients themselves as optimization variables; this makes clear that the algorithm restricts the reachable model space, not the underlying federated objective.

Let
\begin{equation}
\label{eq:fixed_seed_matrix}
\boldsymbol Z
=
[\boldsymbol z_1,\ldots,\boldsymbol z_K]
\in\mathbb R^{d\times K}
\end{equation}
be the perturbation matrix generated by the shared seed pool
$\mathbb S=\{s_1,\ldots,s_K\}$, where
$\boldsymbol z_k=\mathsf{Perturb}(s_k)$.
Define the normalized seed matrix
\begin{equation}
\label{eq:normalized_seed_matrix}
\boldsymbol A
=
\frac{1}{\sqrt K}\boldsymbol Z.
\end{equation}
For the global and local accumulators in
Section~\ref{sec:seed_encoded_zoo}, define
\begin{align}
\boldsymbol u^r
&=
-\eta\sqrt K\,\boldsymbol a^r,
\label{eq:global_seed_parameter}
\\
\boldsymbol u_c^{r,t}
&=
\boldsymbol u^{r-1}
-\eta\sqrt K\,\boldsymbol a_c^{r,t}.
\label{eq:local_seed_parameter}
\end{align}
The reconstructed models are therefore
\begin{align}
\boldsymbol w^r
&=
\boldsymbol w^0+\boldsymbol A\boldsymbol u^r,
\notag\\
\boldsymbol w_c^{r,t}
&=
\boldsymbol w^0+\boldsymbol A\boldsymbol u_c^{r,t}.
\label{eq:seed_parameterized_models}
\end{align}
Thus, every model lies in
$\mathcal W_{\boldsymbol Z}=\boldsymbol w^0+
\operatorname{range}(\boldsymbol Z)$.

In these coordinates, the local and global objectives are simply
\begin{align}
F_{\boldsymbol Z,c}(\boldsymbol u)
&=
f_c(\boldsymbol w^0+\boldsymbol A\boldsymbol u),
\notag\\
F_{\boldsymbol Z}(\boldsymbol u)
&=
\frac1M\sum_{c=1}^{M}F_{\boldsymbol Z,c}(\boldsymbol u)
=
f(\boldsymbol w^0+\boldsymbol A\boldsymbol u).
\label{eq:seed_parameterized_objective}
\end{align}
Equation~\eqref{eq:seed_parameterized_objective} is exactly the original
federated objective restricted to the affine space
$\mathcal W_{\boldsymbol Z}$. In particular,
\begin{equation}
\label{eq:restricted_problem_equivalence}
\inf_{\boldsymbol u\in\mathbb R^K}
F_{\boldsymbol Z}(\boldsymbol u)
=
\inf_{\boldsymbol w\in\mathcal W_{\boldsymbol Z}}
f(\boldsymbol w).
\end{equation}
This fixed random-subspace interpretation is consistent with the finite-seed
representation used by FedKSeed~\cite{qin2024federated}.

By the chain rule,
\begin{align}
\nabla F_{\boldsymbol Z,c}(\boldsymbol u)
&=
\boldsymbol A^{\top}
\nabla f_c(\boldsymbol w^0+\boldsymbol A\boldsymbol u),
\notag\\
\nabla F_{\boldsymbol Z}(\boldsymbol u)
&=
\boldsymbol A^{\top}
\nabla f(\boldsymbol w^0+\boldsymbol A\boldsymbol u).
\label{eq:seed_parameterized_gradient}
\end{align}
If every $f_c$ is $L$-smooth in model space, then every
$F_{\boldsymbol Z,c}$ is $L_{\boldsymbol Z}$-smooth in parameter space,
where
\begin{equation}
\label{eq:seed_parameterized_smoothness}
L_{\boldsymbol Z}
=
L\|\boldsymbol A\|_{\mathrm{op}}^2
=
\frac{L}{K}\|\boldsymbol Z\|_{\mathrm{op}}^2.
\end{equation}

The gradient of the restricted objective has a direct full-space interpretation. At 
$\boldsymbol w=\boldsymbol w^0+\boldsymbol A\boldsymbol u$,
\begin{equation}
\label{eq:seed_gradient_energy}
\|\nabla F_{\boldsymbol Z}(\boldsymbol u)\|^2
=
\frac1K
\|\boldsymbol Z^{\top}\nabla f(\boldsymbol w)\|^2.
\end{equation}
Thus the convergence quantity below is simply the squared full gradient measured along the shared seed directions; no separate surrogate objective is introduced.

For $x\in\mathcal D_c$ and $k\in[K]$, define the fixed-seed central
difference
\begin{equation}
\label{eq:fixed_seed_central_difference}
\begin{aligned}
d_{\mu,k}(\boldsymbol u;x)
={}&\frac1{2\mu}\bigl[
\ell(\boldsymbol w^0+\boldsymbol A\boldsymbol u+\mu\boldsymbol z_k;x)
\\[-1mm]
&\qquad-
\ell(\boldsymbol w^0+\boldsymbol A\boldsymbol u-\mu\boldsymbol z_k;x)
\bigr].
\end{aligned}
\end{equation}
Because $\boldsymbol A\boldsymbol e_k=\boldsymbol z_k/\sqrt K$, a query
along $\boldsymbol z_k$ is exactly a coordinatewise central difference of the
parameterized original loss, with coordinate step $\mu\sqrt K$.

Let
\begin{align}
[\boldsymbol d_{\mu}(\boldsymbol u;x)]_k
&=
d_{\mu,k}(\boldsymbol u;x),
\notag\\
\boldsymbol d_{\mu,c}(\boldsymbol u)
&=
\frac1{n_c}
\sum_{x\in\mathcal D_c}
\boldsymbol d_{\mu}(\boldsymbol u;x).
\label{eq:fixed_seed_client_difference}
\end{align}
For FedZOO-NS, define the clipped counterparts by
\begin{align}
[\boldsymbol d_{\mu}^{C}(\boldsymbol u;x)]_k
&=
\operatorname{clip}_{C}
(d_{\mu,k}(\boldsymbol u;x)),
\notag\\
\boldsymbol d_{\mu,c}^{C}(\boldsymbol u)
&=
\frac1{n_c}
\sum_{x\in\mathcal D_c}
\boldsymbol d_{\mu}^{C}(\boldsymbol u;x).
\label{eq:fixed_seed_client_clipped_difference}
\end{align}

The accumulator update already gives the corresponding step in these coordinates. If
$\Delta\boldsymbol a_c^{r,t}=\boldsymbol a_c^{r,t}-\boldsymbol a_c^{r,t-1}$,
then
\begin{equation}
\boldsymbol u_c^{r,t}
=
\boldsymbol u_c^{r,t-1}
-
\eta\sqrt K\,\Delta\boldsymbol a_c^{r,t}.
\label{eq:parameter_space_local_recursions}
\end{equation}
For \textsc{FedZOO-S}, uniform mini-batch and seed sampling therefore gives
\begin{equation}
\mathbb E\!\left[
\sqrt K\,\Delta\boldsymbol a_c^{r,t}
\,\middle|\,\mathcal F_{r,t-1},c,\boldsymbol Z
\right]
=
\frac1{\sqrt K}
\boldsymbol d_{\mu,c}(\boldsymbol u_c^{r,t-1}),
\label{eq:scaled_clean_mean}
\end{equation}
where $\mathcal F_{r,t-1}$ contains the fixed seed pool, the current client set, and everything revealed before this local step. For \textsc{FedZOO-NS}, clipping changes the mean to
\begin{equation}
\mathbb E\!\left[
\sqrt K\,\Delta\boldsymbol a_c^{r,t}
\,\middle|\,\mathcal F_{r,t-1},c,\boldsymbol Z
\right]
=
\frac1{\sqrt K}
\boldsymbol d_{\mu,c}^{C}(\boldsymbol u_c^{r,t-1}).
\label{eq:scaled_clipped_mean}
\end{equation}
The added scalar Gaussian noises are independent and mean zero. After mapping them to seed coordinates, their contribution to the conditional second moment is $Kbq\sigma_\rho^2$; repeated use of the same seed coordinate does not create cross terms because the scalar noises are independent.

Define the finite-difference and clipping biases by
\begin{align}
\boldsymbol b_{\mu,c}(\boldsymbol u)
&=
\frac1{\sqrt K}
\boldsymbol d_{\mu,c}(\boldsymbol u)
-
\nabla F_{\boldsymbol Z,c}(\boldsymbol u),
\label{eq:scaled_fd_bias}
\\
\boldsymbol b_{C,c}(\boldsymbol u)
&=
\frac1{\sqrt K}
(\boldsymbol d_{\mu,c}^{C}(\boldsymbol u)
-
\boldsymbol d_{\mu,c}(\boldsymbol u)).
\label{eq:scaled_clipping_bias}
\end{align}
If every sample loss is $L$-smooth, then
\begin{equation}
\label{eq:scaled_fd_bias_bound}
\|\boldsymbol b_{\mu,c}(\boldsymbol u)\|^2
\le
\beta_{\mu,\boldsymbol Z}^2,
\qquad
\beta_{\mu,\boldsymbol Z}^2
=
\frac{L^2\mu^2}{4K}
\sum_{k=1}^{K}\|\boldsymbol z_k\|^4.
\end{equation}
Indeed, the error of the $k$th central difference is at most
$L\mu\|\boldsymbol z_k\|^2/2$.

With this fixed-seed parameterization in place, we next quantify the optimization progress of both learners on the restricted form of the original federated objective.

\subsection{Convergence in the Seed Subspace}
\label{sec:finite_seed_utility}

Because all model movement is confined to the span of the shared seeds, the quantity we can control directly is the full gradient projected onto this span. The main result below shows how this projected stationarity depends on the number of rounds, participating clients, local steps, finite-difference error, and client heterogeneity. Unless stated otherwise, probabilities and expectations in this subsection are conditional on $\boldsymbol Z$.

At one local step, the only sampling variability comes from choosing a mini-batch and then choosing seed coordinates for the selected examples. We can compute this variance directly, without introducing another estimator solely for the proof. If a \textsc{FedZOO-S} step starts from \(\boldsymbol u\), then
\begin{align}
&\mathbb E
\left\|
\sqrt K\,\Delta\boldsymbol a_c^{r,t}
-
\frac1{\sqrt K}\boldsymbol d_{\mu,c}(\boldsymbol u)
\right\|^2
\notag\\
&\quad=
\frac{K-q}{bqK}
\frac1{n_c}
\sum_{x\in\mathcal D_c}
\|\boldsymbol d_{\mu}(\boldsymbol u;x)\|^2
\notag\\
&\qquad+
\frac{n_c-b}{b(n_c-1)K}
\frac1{n_c}
\sum_{x\in\mathcal D_c}
\|\boldsymbol d_{\mu}(\boldsymbol u;x)
-
\boldsymbol d_{\mu,c}(\boldsymbol u)\|^2.
\label{eq:scaled_exact_sampling_variance}
\end{align}
The second term is zero when \(b=n_c\), with the ratio interpreted as zero when \(n_c=b=1\). For \textsc{FedZOO-NS}, the same calculation applied before the Gaussian perturbation gives
\begin{align}
&\mathbb E
\left\|
\sqrt K\,\Delta\boldsymbol a_{c,\mathrm{clip}}^{r,t}
-
\frac1{\sqrt K}\boldsymbol d_{\mu,c}^{C}(\boldsymbol u)
\right\|^2
\notag\\
&\quad\le
\frac{C^2}{b}
\left(
\frac{K-q}{q}
+
\frac{n_c-b}{n_c-1}
\right).
\label{eq:scaled_clipped_variance_bound}
\end{align}
Here \(\Delta\boldsymbol a_{c,\mathrm{clip}}^{r,t}\) denotes only the clipped finite-difference part of the accumulator increment at that step. Adding the independent Gaussian coefficients contributes another \(Kbq\sigma_\rho^2\) to the conditional variance. These calculations expose the only differences between the two learners that
matter for convergence: the mean of a local seed-coordinate update, its
variation around that mean, and its overall second moment. We state the
regularity conditions directly in terms of the update produced by the
algorithm.

\begin{assumption}[Utility regularity]
\label{ass:finite_seed_utility}
Conditionally on the realized seed pool \(\boldsymbol Z\), the following hold for the learner \(\mathsf A\in\{\mathsf S,\mathsf{NS}\}\).

\emph{(i)} Every sample loss is \(L\)-smooth, and the original federated objective is lower bounded:
\begin{equation}
\label{eq:full_objective_lower_bound}
f(\boldsymbol w)
\ge
f^\star> -\infty,
\qquad
\boldsymbol w\in\mathbb R^d.
\end{equation}

\emph{(ii)} Client gradients in the fixed seed coordinates have bounded heterogeneity:
\begin{equation}
\label{eq:coefficient_heterogeneity}
\frac1M
\sum_{c=1}^{M}
\|\nabla F_{\boldsymbol Z,c}(\boldsymbol u)
-
\nabla F_{\boldsymbol Z}(\boldsymbol u)\|^2
\le
\zeta_{\boldsymbol Z}^2,
\qquad
\boldsymbol u\in\mathbb R^K.
\end{equation}

\emph{(iii)} When one local step starts from \(\boldsymbol u\), let \(\Delta\boldsymbol a_{\mathsf A,c}(\boldsymbol u)\) be the random accumulator increment produced by that step. There exist finite constants \(\beta_{\mathsf A},\nu_{\mathsf A},G_{\mathsf A}\) such that, uniformly over clients and starting points,
\begin{align}
&\left\|
\mathbb E\!\left[\sqrt K\,\Delta\boldsymbol a_{\mathsf A,c}(\boldsymbol u)\mid\boldsymbol Z\right]
-
\nabla F_{\boldsymbol Z,c}(\boldsymbol u)
\right\|^2
\le
\beta_{\mathsf A}^2,
\notag\\
&\mathbb E\!\left[
\left\|\sqrt K\,\Delta\boldsymbol a_{\mathsf A,c}(\boldsymbol u)
-
\mathbb E[\sqrt K\,\Delta\boldsymbol a_{\mathsf A,c}(\boldsymbol u)\mid\boldsymbol Z]
\right\|^2
\middle|\boldsymbol Z
\right]
\le
\nu_{\mathsf A}^2,
\notag\\
&\mathbb E\!\left[
\|\sqrt K\,\Delta\boldsymbol a_{\mathsf A,c}(\boldsymbol u)\|^2
\middle|\boldsymbol Z
\right]
\le
G_{\mathsf A}^2.
\label{eq:uniform_conditional_estimator_moments}
\end{align}
\end{assumption}

The bounds above also give explicit choices for these constants. For \textsc{FedZOO-S}, the only systematic bias is the central-difference error, so
\begin{equation}
\label{eq:scaled_bias_bounds}
\beta_{\mathsf S}^2
\le
\beta_{\mu,\boldsymbol Z}^2.
\end{equation}
For \textsc{FedZOO-NS}, clipping adds another bias, and one may take
\begin{equation}
\beta_{\mathsf{NS}}^2
\le
2\beta_{\mu,\boldsymbol Z}^2
+
\frac{2}{K}
\sup_{c,\boldsymbol u}
\|\boldsymbol d_{\mu,c}^{C}(\boldsymbol u)
-
\boldsymbol d_{\mu,c}(\boldsymbol u)\|^2.
\label{eq:scaled_ns_bias_bound}
\end{equation}
Likewise, \(\nu_{\mathsf S}^2\) can be chosen as the supremum of the right-hand side of Eq.~\eqref{eq:scaled_exact_sampling_variance}, while for \textsc{FedZOO-NS} one may use
\begin{equation}
\label{eq:scaled_variance_ns}
\nu_{\mathsf{NS}}^2
\le
\sup_{c,\boldsymbol u}
\frac{C^2}{b}
\left(
\frac{K-q}{q}
+
\frac{n_c-b}{n_c-1}
\right)
+
Kbq\sigma_\rho^2.
\end{equation}
Because every clipped coefficient has magnitude at most \(C/(bq)\), the noisy learner also satisfies the simple second-moment bound
\begin{equation}
\label{eq:scaled_ns_second_moment}
G_{\mathsf{NS}}^2
\le
KC^2+Kbq\sigma_\rho^2.
\end{equation}

The heterogeneity assumption above is no stronger than a standard model-space condition up to the geometry of the seed matrix. In particular, if
\begin{equation}
\frac1M
\sum_{c=1}^{M}
\|\nabla f_c(\boldsymbol w)-\nabla f(\boldsymbol w)\|^2
\le
\zeta_f^2,
\end{equation}
then
\begin{equation}
\label{eq:model_to_parameterized_heterogeneity}
\zeta_{\boldsymbol Z}^2
\le
\|\boldsymbol A\|_{\mathrm{op}}^2\zeta_f^2
=
\frac{\|\boldsymbol Z\|_{\mathrm{op}}^2}{K}\zeta_f^2.
\end{equation}

A standard model-space gradient bound is sufficient for these moment conditions. If there is a finite \(G_f\) such that
\begin{equation}
\label{eq:uniform_model_gradient_moment}
\frac1{n_c}\sum_{x\in\mathcal D_c}
\|\nabla\ell(\boldsymbol w;x)\|^2
\le G_f^2,
\qquad c\in[M],\ \boldsymbol w\in\mathbb R^d,
\end{equation}
then sample-wise smoothness gives
\begin{equation}
\label{eq:uniform_directional_moment_bound}
\frac1{Kn_c}\sum_{x\in\mathcal D_c}
\|\boldsymbol d_\mu(\boldsymbol u;x)\|^2
\le
2\|\boldsymbol A\|_{\mathrm{op}}^2G_f^2
+
2\beta_{\mu,\boldsymbol Z}^2.
\end{equation}
Together with Eqs.~\eqref{eq:scaled_exact_sampling_variance}--\eqref{eq:scaled_ns_second_moment}, this is enough to make the constants in Assumption~\ref{ass:finite_seed_utility} finite for both learners. Equation~\eqref{eq:uniform_model_gradient_moment} is only a sufficient condition; it is not implied by smoothness and is not otherwise required by the theorem.

For uniform size-$m$ client sampling without replacement, define
\begin{equation}
\label{eq:client_finite_population_factor}
\chi_m
=
\begin{cases}
\displaystyle
\frac{M-m}{m(M-1)},
&M>1,\\[1mm]
0,
&M=1.
\end{cases}
\end{equation}

\begin{theorem}[Convergence in the fixed seed subspace]
\label{thm:finite_seed_stationarity}
Fix a realized pool $\boldsymbol Z$ and suppose
Assumption~\ref{ass:finite_seed_utility} holds for
$\mathsf A\in\{\mathsf S,\mathsf{NS}\}$. If
\begin{equation}
\label{eq:utility_stepsize_condition}
\eta>0,
\qquad
4\eta L_{\boldsymbol Z}E\le1,
\end{equation}
then the corresponding learner satisfies
\begin{align}
&\frac1R
\sum_{r=1}^{R}
\mathbb E\left[
\frac1K\|\boldsymbol Z^{\top}\nabla f(\boldsymbol w^{r-1})\|^2
\,\middle|\,\boldsymbol Z\right]
\notag\\
&\quad\le
\frac{2(f(\boldsymbol w^0)-f^\star)}{\eta ER}
+
6\beta_{\mathsf A}^2
\notag\\
&\qquad+
8L_{\boldsymbol Z}\eta
\left(
E\chi_m\zeta_{\boldsymbol Z}^2
+
\frac{\nu_{\mathsf A}^2}{m}
\right)
+
2L_{\boldsymbol Z}^2\eta^2E^2G_{\mathsf A}^2.
\label{eq:full_objective_seed_stationarity}
\end{align}
The expectation is over the learner's randomness after the pool is fixed.
Averaging over pools requires the conditions stated below.
\end{theorem}

The four terms reflect different parts of the learning process. More rounds and
local steps reduce the optimization term through $1/(\eta ER)$, whereas the
finite-difference and clipping bias does not disappear merely by training
longer. Averaging across participating clients reduces the centered local
variation through $\nu_{\mathsf A}^2/m$. Increasing the number of local steps
has a second effect, however: it also enlarges the contribution of client
heterogeneity and local drift through $E\chi_m\zeta_{\boldsymbol Z}^2$ and
$E^2G_{\mathsf A}^2$. In \textsc{FedZOO-NS}, the Gaussian noise required for
stability enters the last two quantities as well, so the same inequality makes
the stability--utility tradeoff explicit.

\begin{proof}[Proof of Theorem~\ref{thm:finite_seed_stationarity}]
Let $\mathcal H_{r-1}$ contain $\boldsymbol Z$ and the complete history
through round $r-1$, before sampling $\mathcal C^r$, and write
$\mathbb E_r[\cdot]=\mathbb E[\cdot\mid\mathcal H_{r-1}]$.
Fix $r$, put $\boldsymbol u=\boldsymbol u^{r-1}$ and
$\boldsymbol g=\nabla F_{\boldsymbol Z}(\boldsymbol u)$. Write
$\boldsymbol v_c^{r,t}=\sqrt K\,\Delta\boldsymbol a_c^{r,t}$
for the scaled accumulator increment at local step $t$. The bounds in
Assumption~\ref{ass:finite_seed_utility} apply to these actual updates. The round recursion is
\begin{equation}
\label{eq:parameterized_round_recursion}
\boldsymbol u^r
=
\boldsymbol u-
\eta\boldsymbol S_r,
\qquad
\boldsymbol S_r
=
\frac1m
\sum_{c\in\mathcal C^r}
\sum_{t=1}^{E}
\boldsymbol v_c^{r,t}.
\end{equation}
To use uniform client sampling without tying a client's local trajectory to
whether it is selected, imagine drawing the local randomness for every client
before $\mathcal C^r$ is sampled, independently across clients, and then
executing only the selected trajectories. This equivalent construction does
not change the law of the algorithm, and every such trajectory starts from
$\boldsymbol u$. The second-moment bound together with
$\boldsymbol u_c^{r,t-1}-\boldsymbol u=-\eta\sum_{s<t}\boldsymbol v_c^{r,s}$
therefore gives, for every client,
\begin{equation}
\label{eq:parameterized_local_drift}
\mathbb E_r
\|\boldsymbol u_c^{r,t-1}-\boldsymbol u\|^2
\le
\eta^2(t-1)^2G_{\mathsf A}^2.
\end{equation}

Let
$\overline{\boldsymbol v}^{r,t}
=m^{-1}\sum_{c\in\mathcal C^r}\boldsymbol v_c^{r,t}$.
Uniform client sampling, smoothness, Jensen's inequality, and
\eqref{eq:parameterized_local_drift} imply
\begin{equation}
\mathbb E_r[\overline{\boldsymbol v}^{r,t}]
=
\boldsymbol g+\boldsymbol e_{r,t},
\end{equation}
where
\begin{equation}
\label{eq:scaled_mean_update_error}
\|\boldsymbol e_{r,t}\|^2
\le
2L_{\boldsymbol Z}^2\eta^2(t-1)^2G_{\mathsf A}^2
+
2\beta_{\mathsf A}^2.
\end{equation}
Using
$\langle\boldsymbol a,\boldsymbol b\rangle
\ge-\|\boldsymbol a\|^2/4-\|\boldsymbol b\|^2$ for the error term,
\begin{align}
\mathbb E_r
\langle\boldsymbol g,
\overline{\boldsymbol v}^{r,t}\rangle
&\ge
\frac34\|\boldsymbol g\|^2
-
2L_{\boldsymbol Z}^2\eta^2(t-1)^2G_{\mathsf A}^2
-
2\beta_{\mathsf A}^2.
\label{eq:scaled_inner_product_bound}
\end{align}

At a fixed local step, separate each update into its local gradient, its
conditional bias, and the fresh randomness of that step. Define
\begin{equation}
\boldsymbol b_c(\boldsymbol x)
=
\mathbb E[\boldsymbol v_c^{r,t}\mid
\boldsymbol u_c^{r,t-1}=\boldsymbol x,\boldsymbol Z]
-\nabla F_{\boldsymbol Z,c}(\boldsymbol x),
\end{equation}
so Assumption~\ref{ass:finite_seed_utility} gives
$\|\boldsymbol b_c(\boldsymbol x)\|^2\le\beta_{\mathsf A}^2$.
With $\mathcal F_{r,t-1}$ denoting the history available immediately before
the sampling and Gaussian noise at local step $t$, define
\begin{align}
\boldsymbol\epsilon_c^{r,t}
&=\boldsymbol v_c^{r,t}
-\nabla F_{\boldsymbol Z,c}(\boldsymbol u_c^{r,t-1})
-\boldsymbol b_c(\boldsymbol u_c^{r,t-1}),
\notag\\
\mathbb E[\boldsymbol\epsilon_c^{r,t}\mid\mathcal F_{r,t-1}]&=\boldsymbol0.
\label{eq:utility_centered_local_error}
\end{align}
The difference between the full round update and $E\boldsymbol g$ can then be
written as four sources of error: client sampling, local drift, bias, and
fresh stochastic variation,
\begin{equation}
\boldsymbol S_r-E\boldsymbol g
=
\boldsymbol A_r+
\boldsymbol D_r+
\boldsymbol B_r+
\boldsymbol N_r,
\end{equation}
where
\begin{align}
\boldsymbol A_r
&=E\left(\frac1m\sum_{c\in\mathcal C^r}
\nabla F_{\boldsymbol Z,c}(\boldsymbol u)-\boldsymbol g\right),
\notag\\
\boldsymbol D_r
&=\frac1m\sum_{c\in\mathcal C^r}\sum_{t=1}^E
\left(\nabla F_{\boldsymbol Z,c}(\boldsymbol u_c^{r,t-1})
-\nabla F_{\boldsymbol Z,c}(\boldsymbol u)\right),
\notag\\
\boldsymbol B_r
&=\frac1m\sum_{c\in\mathcal C^r}\sum_{t=1}^E
\boldsymbol b_c(\boldsymbol u_c^{r,t-1}),
\notag\\
\boldsymbol N_r
&=\sum_{t=1}^E\frac1m\sum_{c\in\mathcal C^r}
\boldsymbol\epsilon_c^{r,t}.
\label{eq:utility_round_decomposition}
\end{align}
Each term is controlled by a different part of the assumptions. Uniform
sampling without replacement gives the finite-population variance in the first
line. Smoothness together with Eq.~\eqref{eq:parameterized_local_drift}
controls the change of each local gradient in the second. The third follows
from the bias bound. Finally, the centered errors are martingale differences
across local steps, while different clients use independent fresh sampling and
noise at a fixed step. Hence
\begin{align}
\mathbb E_r\|\boldsymbol A_r\|^2
&\le
E^2\chi_m\zeta_{\boldsymbol Z}^2,
\notag\\
\mathbb E_r\|\boldsymbol D_r\|^2
&\le
\frac{L_{\boldsymbol Z}^2\eta^2E^4}{3}
G_{\mathsf A}^2,
\notag\\
\mathbb E_r\|\boldsymbol B_r\|^2
&\le
E^2\beta_{\mathsf A}^2,
\notag\\
\mathbb E_r\|\boldsymbol N_r\|^2
&\le
\frac{E}{m}\nu_{\mathsf A}^2.
\label{eq:scaled_round_components}
\end{align}
For the drift term, Cauchy--Schwarz gives
$L_{\boldsymbol Z}^2E\sum_{t=1}^E\eta^2(t-1)^2G_{\mathsf A}^2$.
For the stochastic term, orthogonality of the martingale differences gives
\begin{equation}
\mathbb E_r\|\boldsymbol N_r\|^2
=\sum_{t=1}^E\mathbb E_r
\left\|\frac1m\sum_{c\in\mathcal C^r}
\boldsymbol\epsilon_c^{r,t}\right\|^2
\le\frac{E}{m}\nu_{\mathsf A}^2.
\label{eq:utility_round_noise_orthogonality}
\end{equation}
Only the new randomness generated at each local step is used here; the
local trajectories themselves need not be independent across their $E$ steps.
Combining the four contributions yields
\begin{align}
\mathbb E_r\|\boldsymbol S_r\|^2
&\le
2E^2\|\boldsymbol g\|^2
+
8E^2\chi_m\zeta_{\boldsymbol Z}^2
+
8E^2\beta_{\mathsf A}^2
\notag\\
&\quad+
\frac{8E}{m}\nu_{\mathsf A}^2
+
\frac{8L_{\boldsymbol Z}^2\eta^2E^4}{3}
G_{\mathsf A}^2.
\label{eq:scaled_round_second_moment}
\end{align}

We can now return to the objective. Smoothness of $F_{\boldsymbol Z}$
applied to the round update in Eq.~\eqref{eq:parameterized_round_recursion},
together with Eqs.~\eqref{eq:scaled_inner_product_bound} and
\eqref{eq:scaled_round_second_moment}, gives a descent term plus the three
error contributions above. Using
$\sum_{t=1}^{E}(t-1)^2\le E^3/3$ and the step-size condition
$4\eta L_{\boldsymbol Z}E\le1$ simplifies their coefficients, giving
\begin{align}
\mathbb E_r[F_{\boldsymbol Z}(\boldsymbol u^r)]
&\le
F_{\boldsymbol Z}(\boldsymbol u)
-
\frac{\eta E}{2}\|\boldsymbol g\|^2
+
3\eta E\beta_{\mathsf A}^2
\notag\\
&\quad+
4L_{\boldsymbol Z}\eta^2E^2
\chi_m\zeta_{\boldsymbol Z}^2
+
\frac{4L_{\boldsymbol Z}\eta^2E}{m}
\nu_{\mathsf A}^2
\notag\\
&\quad+
L_{\boldsymbol Z}^2\eta^3E^3G_{\mathsf A}^2.
\label{eq:full_objective_one_round_descent}
\end{align}
Since
$F_{\boldsymbol Z}(\boldsymbol u^r)=f(\boldsymbol w^r)\ge f^\star$,
summing over $r$ and using Eq.~\eqref{eq:seed_gradient_energy} proves
Eq.~\eqref{eq:full_objective_seed_stationarity}.
\end{proof}

The theorem controls an average over training checkpoints, rather than only the final model. Equivalently, if $J$ is drawn independently and uniformly from $\{0,\ldots,R-1\}$, then the projected gradient at $\boldsymbol w^J$ obeys the same bound in Eq.~\eqref{eq:full_objective_seed_stationarity}. This random-checkpoint view is only an interpretation of the average: both algorithms still return $\boldsymbol w^R$, and the theorem does not claim the same guarantee for that last iterate or for a checkpoint chosen adaptively from the run.

The result is conditional on the realized seed pool. For a Gaussian pool,
$\|\boldsymbol Z\|_{\mathrm{op}}$ has unbounded support, so no fixed
positive step size can satisfy the theorem's condition for every realization.
The next corollary therefore chooses a deterministic step size that is valid
on a high-probability spectral event. An unconditional expectation requires,
in addition, control of the complementary event.

\begin{corollary}[Fixed step size with a Gaussian seed pool]
\label{cor:fixed_stepsize_seed_utility}
Suppose the columns of $\boldsymbol Z$ are independent
$\mathcal N(\boldsymbol0,I_d)$ vectors. For $0<\delta<1$, set
\begin{align}
\overline L_\delta
&=
\frac{L}{K}
\left(\sqrt d+\sqrt K+\sqrt{2\log(2/\delta)}\right)^2,
\notag\\
\mathcal E_{\mathrm{sp}}(\delta)
&=
\left\{\|\boldsymbol Z\|_{\mathrm{op}}
\le\sqrt d+\sqrt K+\sqrt{2\log(2/\delta)}\right\}.
\label{eq:gaussian_seed_stepsize_event}
\end{align}
Assume utility regularity holds for every pool in
$\mathcal E_{\mathrm{sp}}(\delta)$ and choose the deterministic step size
\begin{equation}
\label{eq:deterministic_gaussian_stepsize}
0<\eta\le\frac1{4E\overline L_\delta}.
\end{equation}
Then, with probability at least $1-\delta$ over the seed pool,
Eq.~\eqref{eq:full_objective_seed_stationarity} holds, with its expectation
conditional on that pool.
\end{corollary}

\begin{proof}[Proof of Corollary~\ref{cor:fixed_stepsize_seed_utility}]
The Gaussian singular-value bound~\cite{vershynin2012introduction}
gives, for $t\ge0$,
\begin{equation}
\label{eq:gaussian_seed_spectral_tail}
\mathbb P_{\boldsymbol Z}\!\left(
\|\boldsymbol Z\|_{\mathrm{op}}>\sqrt d+\sqrt K+t\right)
\le2e^{-t^2/2}.
\end{equation}
Thus $\mathbb P_{\boldsymbol Z}(\mathcal E_{\mathrm{sp}}(\delta))
\ge1-\delta$. On this event, $L_{\boldsymbol Z}\le\overline L_\delta$, so
Eq.~\eqref{eq:deterministic_gaussian_stepsize} implies
Eq.~\eqref{eq:utility_stepsize_condition}. Apply
Theorem~\ref{thm:finite_seed_stationarity} separately for each such pool.
\end{proof}

Under the additional global moment condition
\eqref{eq:uniform_model_gradient_moment}, the excluded pools can also be
controlled. Using its explicit constants above, the joint expectation obeys
\begin{align}
&\frac1R\sum_{r=1}^R
\mathbb E_{\boldsymbol Z,\mathrm{alg}}
\left[\frac1K\|\boldsymbol Z^{\top}\nabla f(\boldsymbol w^{r-1})\|^2\right]
\notag\\
&\quad\le
\mathbb E_{\boldsymbol Z}\!\left[
\mathbf1_{\mathcal E_{\mathrm{sp}}(\delta)}
\left(
\frac{2(f(\boldsymbol w^0)-f^\star)}{\eta ER}
+6\beta_{\mathsf A}^2
+8L_{\boldsymbol Z}\eta\left(E\chi_m\zeta_{\boldsymbol Z}^2+\frac{\nu_{\mathsf A}^2}{m}\right)
+2L_{\boldsymbol Z}^2\eta^2E^2G_{\mathsf A}^2
\right)\right]
+G_f^2\sqrt{\delta\left(d^2+\frac{2d}{K}\right)}.
\label{eq:unconditional_seed_utility_with_tail}
\end{align}
Indeed, $\|\nabla f(\boldsymbol w)\|^2\le G_f^2$ by Jensen's inequality,
so
$K^{-1}\|\boldsymbol Z^{\top}\nabla f(\boldsymbol w)\|^2
\le G_f^2\|\boldsymbol A\|_{\mathrm{op}}^2$. On the complementary event,
Cauchy--Schwarz and
$\mathbb E\|\boldsymbol A\|_{\mathrm F}^4=d^2+2d/K$ give the second term.
The latter identity follows by expanding the square of the sum of $dK$
independent squared standard Gaussians, divided by $K^2$.
On $\mathcal E_{\mathrm{sp}}(\delta)$, integrate the conditional theorem.
All terms are integrable with the explicit constants, since they grow at
most polynomially in the Gaussian entries. Without the additional moment
condition, Corollary~\ref{cor:fixed_stepsize_seed_utility} remains a
high-probability statement about the pool, not an unconditional expectation
bound or a high-probability statement about the training noise.

The leading terms in Theorem~\ref{thm:finite_seed_stationarity} have the familiar stepsize tradeoff: a smaller stepsize slows the decrease of the initial objective gap, whereas a larger one magnifies stochastic variation and client heterogeneity. Balancing these two effects gives a more transparent rate.

\begin{corollary}[Balanced step size]
\label{cor:balanced_seed_stationarity}
Under the conditions of Theorem~\ref{thm:finite_seed_stationarity}, assume
\[
E\chi_m\zeta_{\boldsymbol Z}^2+\frac{\nu_{\mathsf A}^2}{m}>0,
\qquad
f(\boldsymbol w^0)>f^\star.
\]
If
\begin{equation}
\label{eq:balanced_seed_stepsize}
\eta
=
\sqrt{
\frac{f(\boldsymbol w^0)-f^\star}
{4L_{\boldsymbol Z}ER\left(E\chi_m\zeta_{\boldsymbol Z}^2+\nu_{\mathsf A}^2/m\right)}
}
\le
\frac1{4L_{\boldsymbol Z}E},
\end{equation}
then
\begin{align}
&\frac1R
\sum_{r=1}^{R}
\mathbb E
\left[\frac1K\|\boldsymbol Z^{\top}\nabla f(\boldsymbol w^{r-1})\|^2\right]
\notag\\
&\quad\le
8\sqrt{
\frac{
L_{\boldsymbol Z}\bigl(f(\boldsymbol w^0)-f^\star\bigr)
\left(E\chi_m\zeta_{\boldsymbol Z}^2+\nu_{\mathsf A}^2/m\right)
}{ER}
}
+6\beta_{\mathsf A}^2
\notag\\
&\qquad+
\frac{
L_{\boldsymbol Z}E G_{\mathsf A}^2\bigl(f(\boldsymbol w^0)-f^\star\bigr)
}{2R\left(E\chi_m\zeta_{\boldsymbol Z}^2+\nu_{\mathsf A}^2/m\right)}.
\label{eq:balanced_seed_stationarity}
\end{align}
\end{corollary}

\begin{proof}[Proof of Corollary~\ref{cor:balanced_seed_stationarity}]
Substitute Eq.~\eqref{eq:balanced_seed_stepsize} into
Eq.~\eqref{eq:full_objective_seed_stationarity}. The first term and the term linear in $\eta$ are then equal up to the displayed constants, while the $\eta^2$ term gives the final $R^{-1}$ contribution. The upper bound on $\eta$ is exactly the stepsize condition required by Theorem~\ref{thm:finite_seed_stationarity}.
\end{proof}

If the balancing choice in Eq.~\eqref{eq:balanced_seed_stepsize} is larger than $1/(4L_{\boldsymbol Z}E)$, one may instead use this boundary value and apply Theorem~\ref{thm:finite_seed_stationarity} directly. For a stream of deletion requests, the stepsize is inherited from the original learning run and is not retuned after each request.

The balanced bound also shows how client participation enters the rate. Since
\[
\chi_m=\frac{M-m}{m(M-1)}\le\frac1m,
\]
the quantity inside the square root is at most
\[
\frac{E\zeta_{\boldsymbol Z}^2+\nu_{\mathsf A}^2}{m}.
\]
Hence, when $E$, the seed pool, and the estimator moments are fixed, the leading term decreases as
\begin{equation}
\label{eq:linear_client_speedup_leading_term}
\mathcal O\!\left(\frac1{\sqrt{mER}}\right).
\end{equation}
For fixed $E$, this means that in the regime where the leading stochastic term dominates the bias and the higher-order local-drift term, multiplying the number of participating clients by a factor reduces the required number of communication rounds by the same factor. This is the usual linear client speedup. It is not a global asymptotic statement: as $m$ grows, the $R^{-1}$ local-drift term and any nonvanishing bias eventually become limiting. Full participation also removes the client-sampling heterogeneity term exactly because $\chi_M=0$.

For \textsc{FedZOO-S}, this averaging gain is not affected by any stability noise. For \textsc{FedZOO-NS}, increasing $m$ still reduces the sampling and client-heterogeneity terms, but the Gaussian variance required for stability may leave a residual that does not shrink with $mER$. In particular, when the horizon-dependent term in Eq.~\eqref{eq:fedzoons_sigma_calibration} determines $\sigma_\rho$,
\begin{equation}
\label{eq:scaled_calibrated_noise_variance}
\frac{Kbq\sigma_\rho^2}{mER}
=
\frac{4KbC^2}{M\underline n^2\rho^2}.
\end{equation}
Thus the client-averaging benefit eventually saturates at the stability-induced noise and bias floors. This dependence is important when interpreting the balanced rate: for \textsc{FedZOO-NS}, the variance and second-moment bounds inherit the horizon dependence of $\sigma_\rho$ and should not be treated as fixed asymptotic constants.

\subsection{From Seed-Subspace to Full-Space Guarantees}
\label{sec:full_space_utility}

Up to this point, the theory controls only the component of the gradient seen
through the fixed seed directions. Nothing stronger is possible for an
arbitrary objective: a low-dimensional seed span may be orthogonal to a
nonzero component of the full gradient. A full-space statement therefore needs
one additional property--the seed span must preserve the gradients that can
occur along the training trajectory.

\begin{assumption}[Gradient coverage by the seed span]
\label{ass:seed_gradient_coverage}
For the realized perturbation pool, there exist finite constants
$\kappa_{\boldsymbol Z}>0$ and $\varepsilon_{\mathrm{cov}}\ge0$, fixed
before the subsequent training randomness, such that every possible global
checkpoint $\boldsymbol w$ satisfies
\begin{equation}
\label{eq:seed_gradient_coverage}
\frac1K\|\boldsymbol Z^{\top}\nabla f(\boldsymbol w)\|^2
\ge
\kappa_{\boldsymbol Z}
\|\nabla f(\boldsymbol w)\|^2
-
\varepsilon_{\mathrm{cov}}^2.
\end{equation}
\end{assumption}

\begin{corollary}[Full-space stationarity]
\label{cor:full_space_stationarity}
Under the conditions of Theorem~\ref{thm:finite_seed_stationarity} and
Assumption~\ref{ass:seed_gradient_coverage},
\begin{align}
&\frac1R
\sum_{r=1}^{R}
\mathbb E\|\nabla f(\boldsymbol w^{r-1})\|^2
\notag\\
&\quad\le
\frac1{\kappa_{\boldsymbol Z}}
\left[
\frac{2(f(\boldsymbol w^0)-f^\star)}{\eta ER}
+
6\beta_{\mathsf A}^2
+
\varepsilon_{\mathrm{cov}}^2
\right.
\notag\\
&\qquad\left.
+
8L_{\boldsymbol Z}\eta
\left(
E\chi_m\zeta_{\boldsymbol Z}^2
+
\frac{\nu_{\mathsf A}^2}{m}
\right)
+
2L_{\boldsymbol Z}^2\eta^2E^2G_{\mathsf A}^2
\right].
\label{eq:full_space_stationarity_bound}
\end{align}
\end{corollary}

\begin{proof}[Proof of Corollary~\ref{cor:full_space_stationarity}]
Rearrange~\eqref{eq:seed_gradient_coverage}, average over the iterates, and
apply Theorem~\ref{thm:finite_seed_stationarity}.
\end{proof}

\begin{proposition}[Coverage under a fixed gradient subspace]
\label{prop:gaussian_seed_coverage}
Suppose $\boldsymbol z_1,\ldots,\boldsymbol z_K$ are independent
$\mathcal N(\boldsymbol0,I_d)$ vectors. Let
$\mathcal V\subseteq\mathbb R^d$ be a fixed $r$-dimensional subspace chosen
independently of the seed pool. Assume, for every pool under consideration,
that all possible global checkpoints satisfy
\begin{equation}
\label{eq:low_gradient_dimension}
\nabla f(\boldsymbol w)
\in
\mathcal V.
\end{equation}
For $0<\delta<1$, define
\begin{equation}
\label{eq:gaussian_embedding_alpha}
\alpha_{K,r,\delta}
=
\frac{
\sqrt r+\sqrt{2\log(2/\delta)}
}{\sqrt K}.
\end{equation}
If $\alpha_{K,r,\delta}<1$, then, with probability at least $1-\delta$ over
the seed pool,
\begin{equation}
\label{eq:gaussian_uniform_embedding}
(1-\alpha_{K,r,\delta})^2\|\boldsymbol g\|^2
\le
\frac1K\|\boldsymbol Z^{\top}\boldsymbol g\|^2
\le
(1+\alpha_{K,r,\delta})^2\|\boldsymbol g\|^2
\end{equation}
simultaneously for every $\boldsymbol g\in\mathcal V$. Consequently,
Assumption~\ref{ass:seed_gradient_coverage} holds with
\begin{equation}
\label{eq:gaussian_coverage_constants}
\kappa_{\boldsymbol Z}
=
(1-\alpha_{K,r,\delta})^2,
\qquad
\varepsilon_{\mathrm{cov}}=0.
\end{equation}
\end{proposition}

\begin{proof}[Proof of Proposition~\ref{prop:gaussian_seed_coverage}]
Let $\boldsymbol U\in\mathbb R^{d\times r}$ have orthonormal columns spanning
$\mathcal V$. Then $\boldsymbol Z^{\top}\boldsymbol U$ is a
$K\times r$ standard Gaussian matrix. The two-sided Gaussian
singular-value bound used in the proof of
Corollary~\ref{cor:fixed_stepsize_seed_utility} gives, with probability
at least $1-\delta$,
\begin{align}
\sigma_{\min}(\boldsymbol A^{\top}\boldsymbol U)
&\ge
1-\alpha_{K,r,\delta},
\notag\\
\sigma_{\max}(\boldsymbol A^{\top}\boldsymbol U)
&\le
1+\alpha_{K,r,\delta}.
\end{align}
Applying these inequalities to
$\boldsymbol g=\boldsymbol U\boldsymbol y$ proves
\eqref{eq:gaussian_uniform_embedding}. The event is uniform over the entire
subspace, so it remains valid although the optimization iterates depend on
the same seed pool.
\end{proof}

\begin{corollary}[High-probability full-space guarantee]
\label{cor:gaussian_full_space_stationarity}
Suppose the Gaussian-pool and fixed-subspace conditions of
Proposition~\ref{prop:gaussian_seed_coverage} hold. Let $0<\delta<1$,
$\alpha_{K,r,\delta/2}<1$, and choose a deterministic step size
\begin{equation}
\label{eq:joint_gaussian_stepsize}
0<\eta\le\frac1{4E\overline L_{\delta/2}}.
\end{equation}
Assume utility regularity holds on
$\mathcal E_{\mathrm{sp}}(\delta/2)$. With probability at least $1-\delta$
over the seed pool,
\begin{equation}
\label{eq:gaussian_full_space_bound}
\frac1R\sum_{r=1}^{R}
\mathbb E\bigl[\|\nabla f(\boldsymbol w^{r-1})\|^2
\mid\boldsymbol Z\bigr]
\le
\frac{1}{(1-\alpha_{K,r,\delta/2})^2}
\left[
\frac{2(f(\boldsymbol w^0)-f^\star)}{\eta ER}
+6\beta_{\mathsf A}^2
+8L_{\boldsymbol Z}\eta\left(E\chi_m\zeta_{\boldsymbol Z}^2+\frac{\nu_{\mathsf A}^2}{m}\right)
+2L_{\boldsymbol Z}^2\eta^2E^2G_{\mathsf A}^2
\right].
\end{equation}
\end{corollary}

\begin{proof}[Proof of Corollary~\ref{cor:gaussian_full_space_stationarity}]
Let $\mathcal E_{\mathrm{cov}}(\delta/2)$ be the uniform embedding event
in Proposition~\ref{prop:gaussian_seed_coverage}, and set
\begin{equation}
\label{eq:joint_gaussian_utility_event}
\mathcal E_\delta
=
\mathcal E_{\mathrm{sp}}(\delta/2)
\cap\mathcal E_{\mathrm{cov}}(\delta/2).
\end{equation}
A union bound gives $\mathbb P_{\boldsymbol Z}(\mathcal E_\delta)
\ge1-\delta$; independence of the two events is not needed. On this event,
the fixed step size is valid and the coverage constant is at least
$(1-\alpha_{K,r,\delta/2})^2$. Apply
Corollary~\ref{cor:full_space_stationarity} conditionally on each such pool.
\end{proof}

If Eq.~\eqref{eq:uniform_model_gradient_moment} also holds, the same event
split yields the unconditional bound
\begin{align}
&\frac1R\sum_{r=1}^{R}
\mathbb E_{\boldsymbol Z,\mathrm{alg}}
\|\nabla f(\boldsymbol w^{r-1})\|^2
\notag\\
&\quad\le
\frac{\mathbb E_{\boldsymbol Z}\!\left[\mathbf1_{\mathcal E_\delta}
\left(
\frac{2(f(\boldsymbol w^0)-f^\star)}{\eta ER}
+6\beta_{\mathsf A}^2
+8L_{\boldsymbol Z}\eta\left(E\chi_m\zeta_{\boldsymbol Z}^2+\frac{\nu_{\mathsf A}^2}{m}\right)
+2L_{\boldsymbol Z}^2\eta^2E^2G_{\mathsf A}^2
\right)\right]}
{(1-\alpha_{K,r,\delta/2})^2}
+\delta G_f^2,
\label{eq:unconditional_full_space_with_tail}
\end{align}
because $\|\nabla f\|^2\le G_f^2$ on the excluded pools. The fixed-subspace
condition concerns the gradients themselves at all possible global
checkpoints; it is not implied merely by a low-rank Hessian or the existence of
a low-dimensional set of good solutions.

The coverage condition cannot in general be removed. If $\operatorname{rank}(\boldsymbol Z)<d$, choose
\begin{equation}
f(\boldsymbol w)=\frac12\|\boldsymbol w-\boldsymbol w^\star\|^2
\end{equation}
with $\boldsymbol w^\star-\boldsymbol w^0\notin\operatorname{range}(\boldsymbol Z)$. The minimizer over $\mathcal W_{\boldsymbol Z}$ is
\begin{equation}
\boldsymbol w_{\boldsymbol Z}^{\star}
=
\boldsymbol w^0+
\Pi_{\boldsymbol Z}(\boldsymbol w^\star-\boldsymbol w^0),
\end{equation}
where $\Pi_{\boldsymbol Z}$ is the orthogonal projector onto $\operatorname{range}(\boldsymbol Z)$. Its gradient is orthogonal to that range, so
\begin{equation}
\frac1K\|\boldsymbol Z^{\top}\nabla f(\boldsymbol w_{\boldsymbol Z}^{\star})\|^2=0,
\qquad
\|\nabla f(\boldsymbol w_{\boldsymbol Z}^{\star})\|>0.
\end{equation}
For an independent Gaussian pool with $K<d$, the same construction can be fixed before the pool is sampled: any nonzero $\boldsymbol w^\star-\boldsymbol w^0$ lies outside $\operatorname{range}(\boldsymbol Z)$ almost surely. Thus a method confined to a rank-deficient seed span cannot guarantee full-space stationarity for arbitrary smooth objectives. A condition relating the seed span to the gradients encountered during training is therefore genuinely necessary, rather than an artifact of the proof.

The exactness result now closes the loop between unlearning and utility. After
request $j$ is processed, the returned state has exactly the law of a fresh run
on the retained data. The same convergence theorem can therefore be applied to
that run without introducing a separate unlearning-error term. Only the
objective, and hence its data-dependent constants, changes after deletion.

\subsection{Utility After Unlearning}
\label{sec:post_unlearning_utility}

After request $j$, denote the retained local datasets by
$\mathcal D_c^{(j+1)}$ and write
\begin{align}
f_c^{(j)}(\boldsymbol w)
&=
\frac1{|\mathcal D_c^{(j+1)}|}
\sum_{x\in\mathcal D_c^{(j+1)}}
\ell(\boldsymbol w;x),
\notag\\
f^{(j)}(\boldsymbol w)
&=
\frac1M\sum_{c=1}^{M}f_c^{(j)}(\boldsymbol w).
\label{eq:retained_full_objective}
\end{align}
The constants in Assumption~\ref{ass:finite_seed_utility} are evaluated for
this retained objective and are denoted by
$\beta_{\mathsf A,j}$, $\nu_{\mathsf A,j}$,
$G_{\mathsf A,j}$, and $\zeta_{\boldsymbol Z,j}$; also let
$f_j^\star=\inf_{\boldsymbol w}f^{(j)}(\boldsymbol w)>-\infty$.
The seed pool and algorithmic hyperparameters are inherited from the original
run. The same sample-wise smoothness constant $L$ therefore gives the same
$L_{\boldsymbol Z}$. The other constants may change after a deletion and must
be valid along the corresponding retained-data trajectory.

\begin{corollary}[Utility after unlearning]
\label{cor:retained_finite_seed_utility}
Consider the non-adaptive request stream in
Eq.~\eqref{eq:nonadaptive_request_stream}. After request $j$, continue the
state returned by the unlearning procedure on $\mathcal D^{(j+1)}$ to the
fixed horizon $R$ with no further requests, and denote the resulting global
trajectory by $\{\boldsymbol w_{\mathsf A}^{r,(j)}\}_{r=0}^{R}$. Suppose
Assumption~\ref{ass:finite_seed_utility} holds for this retained dataset.
Condition on a realized seed pool $\boldsymbol Z$, keep the original step size,
and assume $4\eta L_{\boldsymbol Z}E\le1$. Then
\begin{equation}
\label{eq:retained_finite_seed_utility}
\frac1R
\sum_{r=1}^{R}
\mathbb E
\left[\frac1K\|\boldsymbol Z^{\top}\nabla f^{(j)}
(\boldsymbol w_{\mathsf A}^{r-1,(j)})\|^2\right]
\le
\frac{2(f^{(j)}(\boldsymbol w^0)-f_j^\star)}{\eta ER}
+6\beta_{\mathsf A,j}^2
+8L_{\boldsymbol Z}\eta
\left(E\chi_m\zeta_{\boldsymbol Z,j}^2+\frac{\nu_{\mathsf A,j}^2}{m}\right)
+2L_{\boldsymbol Z}^2\eta^2E^2G_{\mathsf A,j}^2.
\end{equation}
The expectation is over the learning and unlearning randomness, conditional on
the fixed seed pool and request sequence. Earlier deletion requests do not add
a separate error term; their effect enters only through the retained objective
and its constants.
\end{corollary}

\begin{proof}[Proof of Corollary~\ref{cor:retained_finite_seed_utility}]
Condition on the seed pool and the fixed request sequence. Exact unlearning for
one request, together with the sequential-request corollaries, says that after
request $j$ the returned state has exactly the law of the corresponding
learner on $\mathcal D^{(j+1)}$ at the same logical point. Continuing both
states with the ordinary retained-data learning transitions preserves this
equality in law through round $R$. Theorem~\ref{thm:finite_seed_stationarity}
therefore applies directly to the resulting retained-data trajectory and gives
Eq.~\eqref{eq:retained_finite_seed_utility}.

If the seed pool is random, the same argument can be averaged over any event on
which the retained-data assumptions and step-size condition hold. In
particular, Corollary~\ref{cor:fixed_stepsize_seed_utility} applies on its
spectral event whenever the retained objectives satisfy the required
regularity there. An unconditional expectation additionally requires control
of the complementary seed-pool event, exactly as in
Eq.~\eqref{eq:unconditional_seed_utility_with_tail}.
\end{proof}

The preceding corollary controls the gradient component visible to the fixed
seed span. If that span also captures the retained full gradient, the same
argument gives a full-space statement.

\begin{corollary}[Full-space utility after unlearning]
\label{cor:retained_full_space_utility}
Under the conditions of Corollary~\ref{cor:retained_finite_seed_utility},
suppose that, at every global checkpoint that can arise on the retained-data
trajectory,
\begin{equation}
\frac1K\|\boldsymbol Z^{\top}\nabla f^{(j)}(\boldsymbol w)\|^2
\ge
\kappa_{\boldsymbol Z,j}
\|\nabla f^{(j)}(\boldsymbol w)\|^2
-
\varepsilon_{\mathrm{cov},j}^2
\end{equation}
for some $\kappa_{\boldsymbol Z,j}>0$ and
$\varepsilon_{\mathrm{cov},j}\ge0$. Then, conditionally on the seed pool,
\begin{equation}
\label{eq:retained_full_space_utility}
\frac1R
\sum_{r=1}^{R}
\mathbb E
\|\nabla f^{(j)}
(\boldsymbol w_{\mathsf A}^{r-1,(j)})\|^2
\le
\frac{1}{\kappa_{\boldsymbol Z,j}}
\left[
\frac{2(f^{(j)}(\boldsymbol w^0)-f_j^\star)}{\eta ER}
+6\beta_{\mathsf A,j}^2
+8L_{\boldsymbol Z}\eta
\left(E\chi_m\zeta_{\boldsymbol Z,j}^2+\frac{\nu_{\mathsf A,j}^2}{m}\right)
+2L_{\boldsymbol Z}^2\eta^2E^2G_{\mathsf A,j}^2
+\varepsilon_{\mathrm{cov},j}^2
\right].
\end{equation}
\end{corollary}

\begin{proof}[Proof of Corollary~\ref{cor:retained_full_space_utility}]
Apply the displayed gradient-coverage inequality at each retained-data
checkpoint and average it over the trajectory in
Corollary~\ref{cor:retained_finite_seed_utility}. This gives
Eq.~\eqref{eq:retained_full_space_utility} directly.

For a random Gaussian pool, the averaging event must satisfy both the step-size
condition and the retained-data coverage property. If the gradients of a
retained objective lie in a fixed subspace, the Gaussian embedding result above
can be applied to that subspace with the inherited step size. A guarantee that
holds simultaneously for several retained objectives requires a common
embedding event, or an explicit union bound over the relevant subspaces; the
coverage statement for the original objective alone is not enough.
\end{proof}

\end{document}